\documentclass[11pt]{article}
\usepackage{dsfont}
\usepackage{environ}
\usepackage[margin=1.15in]{geometry}

\usepackage[utf8]{inputenc} 
\usepackage{cmap}
\usepackage[T1]{fontenc}    
\usepackage[pagebackref]{hyperref}
\hypersetup{
	colorlinks = true,
	urlcolor = royalBlue,
	linkcolor = royalBlue,
	citecolor = royalBlue
}
\usepackage{url}            

\usepackage{natbib}

\usepackage[usenames,dvipsnames]{xcolor}
\definecolor{nicePink}{RGB}{247,83,148}
\definecolor{royalBlue}{HTML}{057DCD}

\usepackage{bm}
\usepackage{amsmath}
\usepackage{amsfonts}
\usepackage{amssymb}
\usepackage{amsbsy}
\usepackage{amsthm}
\usepackage{booktabs}
\usepackage{graphicx, ucs}
\usepackage{subcaption}
\usepackage{rotating}
\usepackage{float}
\usepackage{tikz}
\usepackage{algorithm, caption}
\usepackage[noend]{algpseudocode}
\usepackage{listings}
\usepackage{enumitem}
\usepackage{tablefootnote}

\usepackage{multirow}
\usepackage{array}
\usepackage{chngcntr}
\counterwithin*{equation}{section}
\usepackage{soul}
\usepackage{dsfont}
\usepackage{nicefrac}
\usepackage{fancyhdr}
\fancyheadoffset{0pt}
\usepackage[framemethod=tikz]{mdframed}

\usepackage[nameinlink]{cleveref}

\newtheorem{theorem}{Theorem} 
\numberwithin{theorem}{section}
\newtheorem{lemma}[theorem]{Lemma}
\newtheorem{claim}[theorem]{Claim}
\newtheorem{proposition}[theorem]{Proposition}

\newtheorem{challenge}{Challenge}

\newtheorem{assumption}{Assumption}

\newtheorem{inftheorem}{Informal Theorem}
\newtheorem{infassumption}{Informal Assumptions}

\theoremstyle{definition}
\newtheorem{definition}[theorem]{Definition}
\newtheorem{remark}[theorem]{Remark}

\newtheorem*{theorem*}{Theorem}
\newtheorem*{lemma*}{Lemma}

\renewcommand{\Pr}{\mathop{\bf Pr\/}}
\newcommand{\E}{\mathop{\bf E\/}}

\newcommand{\Var}{\mathop{\bf Var\/}}
\newcommand{\Cov}{\mathop{\bf Cov\/}}

\newcommand{\poly}{\textnormal{poly}}

\newcommand{\reals}{\mathbb R}

\newcommand{\nats}{\mathbb N}
\newcommand{\natszero}{\mathbb N_0}

\newcommand{\ints}{\mathbb Z}

\newcommand{\eps}{\epsilon}

\newcommand{\calA}{\mathcal{A}}
\newcommand{\calB}{\mathcal{B}}
\newcommand{\calC}{\mathcal{C}}
\newcommand{\calD}{\mathcal{D}}
\newcommand{\calE}{\mathcal{E}}

\newcommand{\calG}{\mathcal{G}}
\newcommand{\calH}{\mathcal{H}}
\newcommand{\calI}{\mathcal{I}}
\newcommand{\calJ}{\mathcal{J}}

\newcommand{\calL}{\mathcal{L}}
\newcommand{\calM}{\mathcal{M}}

\newcommand{\calR}{\mathcal{R}}

\newcommand{\calT}{\mathcal{T}}

\newcommand{\calW}{\mathcal{W}}
\newcommand{\calX}{\mathcal{X}}
\newcommand{\calY}{\mathcal{Y}}
\newcommand{\calZ}{\mathcal{Z}}

\def\l{\ell}
\def\<{\langle}
\def\>{\rangle}

\newcommand{\Be}{\mathrm{Be}}
\newcommand{\Bin}{\mathrm{Bin}}
\newcommand{\NBin}{\mathrm{NBin}}

\newcommand{\Z}{\mathrm{Z}}
\newcommand{\TG}{\calT\calG}
\newcommand{\Geo}{\mathrm{Geo}}
\newcommand{\Poi}{\mathrm{Poi}}
\newcommand{\Dgauss}{\calZ}
\newcommand{\plow}{p_{\mathrm{low}}}

\newcommand{\tv}{d_{TV}}
\newcommand{\kl}{D_{KL}}

\def\wt{\widetilde}
\def\wh{\widehat}

\def\vec{\bm}

\newcommand{\ball}{\mathbb{B}}

\newcommand{\dgaussvar}{Z}
\newcommand{\ncrit}{n_{\mathrm{crit}}}

\def\Dgauss{\mathcal{Z}}

\def\rcrit{r_{\textnormal{crit}}}
\def\descriptivityparam{\theta}
\def\calAvarrho{\calA_{\varrho}}
\def\modes{\calM}
\def\dnew{d_{\mathrm{ST}}}
\def\calX{\mathcal{X}}

\def\ball{\mathbb{B}}

\def\siirvs{$\textnormal{SIIRVs}$}
\def\siiervs{$\textnormal{SIIERVs}$}
\def\setoperator{\mathtt{Op}}
\def\closure{\mathtt{closure}}
\def\ccalA{\overline{\calA}}
\def\chull{\mathtt{Conv}}
\def\conehull{\mathtt{Cone}}
\def\-{{\text -}}
\def\ball{\mathbb{B}}
\def\shade{\mathtt{Shade}}

\def\TT{{\vec T}}
\def\ZZ{{\Z}_{\vec T}}
\def\Deltae{\Delta {e}_{\vec a, \vec b}} 
\def\vol{\textnormal{Vol}}

\makeatletter
\renewenvironment{abstract}{%
	\if@twocolumn
	\section*{\abstractname}%
	\else 
	\begin{center}%
		{\bfseries \large\abstractname\vspace{\z@}}
	\end{center}%
	\quotation
	\fi}
{\if@twocolumn\else\endquotation\fi}
\makeatother

\begin{document}
	
	\title{Learning and Covering Sums of Independent Random Variables with Unbounded Support}
	\author{
 		{Alkis Kalavasis\thanks{\url{kalavasisalkis@mail.ntua.gr}. Supported by the Hellenic Foundation for Research and Innovation (H.F.R.I.) under the ``First Call for
H.F.R.I. Research Projects to support Faculty members and Researchers and the procurement
of high-cost research equipment grant", project BALSAM, HFRI-FM17-1424.}} \\
 		 National Technical University of Athens \\
 		\and
 		{Konstantinos Stavropoulos\thanks{\url{kstavrop@utexas.edu}. Suppported by the NSF AI Institute for Foundations of Machine Learning (IFML) and by a scholarship from Bodossaki Foundation. }} \\
 		 The University of Texas at Austin \\
 		\and
 		{Manolis Zampetakis\thanks{\url{mzampet@berkeley.edu}. Supported by the Army Research Office (ARO) under contract W911NF-17-1-0304 as
part of the collaboration between US DOD, UK MOD and UK Engineering and Physical
Research Council (EPSRC) under the Multidisciplinary University Research Initiative (MURI).}} \\
 		 University of California, Berkeley \\
	}
	\maketitle
	\thispagestyle{empty}
	
	\begin{abstract}
	\small
      	  We study the problem of covering and learning sums $X = X_1 + \cdots + X_n$ of
	  independent integer-valued random variables $X_i$ (SIIRVs) with \textit{unbounded}, or even \textit{infinite}, support.
	  \cite{de2018learning} at FOCS 2018, showed that the maximum value of the
	  collective support of $X_i$'s necessarily appears in the sample complexity of
	  learning $X$. 
	  In this work, we address two questions: (i) Are there general families of SIIRVs with unbounded support that can be learned with sample complexity independent of both $n$ and the maximal element of the support? (ii) Are there general families of SIIRVs with unbounded support that admit proper sparse covers in total variation distance?
	  As for question (i), we provide a set of simple conditions that allow the unbounded SIIRV to be learned with complexity $\poly(1/\eps)$ bypassing the aforementioned lower bound. We further address question (ii) in the general setting where each
	  variable $X_i$ has unimodal probability mass function and is a different
	  member of some, possibly multi-parameter, \textit{exponential family} $\calE$ that satisfies some structural properties. 
	   These properties allow
	  $\calE$ to contain heavy tailed and non log-concave distributions. Moreover, we show that for every $\eps > 0$, and every $k$-parameter family $\calE$ that satisfies some structural assumptions, there exists an algorithm
	    with $\wt{O}(k) \cdot \poly(1/\eps)$
	    samples that learns a sum of $n$ arbitrary members of $\calE$ within $\eps$
	    in TV distance. The output of the learning algorithm is also a sum of random variables whose distribution lies in the family $\calE$. En route, we prove that any discrete unimodal exponential family with bounded constant-degree central moments can be approximated by the family corresponding to a bounded subset of the initial (unbounded) parameter space.

    \end{abstract}
	
\newpage

\thispagestyle{empty}

\thispagestyle{empty}
\newpage
\setcounter{page}{1}

\section{Introduction}
In this paper, we revisit the problem of learning distributions of the form $X = X_1 + \ldots + X_n$, where $n \in \nats$ and the terms $X_i$ are independent integer random variables. We focus on the cases where each $X_i$ has unbounded, even infinite support. Our work follows the literature of learning distributions from independent samples (see e.g.,~\citep{dasgupta1999learning, rabani2014learning, acharya2015optimal, canonne2015big, canonne2020survey, diakonikolas2019robust, moitra2010settling}), that has been introduced in \cite{kearns1994learnability}. In this problem, we observe independent samples from a random variable $X$, which is a priori known to belong to a class of distributions $\calC$, and the goal is to compute another random variable $X'$ such that $\tv(X, X') \le \eps$. The main question to ask follows: Given $\calC$, how many samples from $X$ do we need to compute the estimate $X'$? If the output $X'$ belongs to $\calC$, we say that we have \textit{properly learned} $X$.

The problem of learning distributions is closely related to the problem of \textit{sparsely covering} a class of distributions. Given a class $\calC$ of distributions, an \textit{$\eps$-cover} for this class is a set $\calC_{\eps}$ of distributions such that, for every $D \in \calC$, there exists a $D' \in \calC_{\eps}$ such that $\tv(D, D') \le \eps$. If $\calC_{\eps} \subseteq \calC$, then $\calC_{\eps}$ is called a \textit{proper cover}. Clearly, the existence of a (small) cover for a class $\calC$ is interesting by its own. Furthermore, once we have designed a cover $\calC_{\eps}$, then there exist generic algorithms, e.g., the tournament procedure \citep{daskalakis2014faster}, that uses $\calC_{\eps}$ to produce a learning algorithm with sample complexity $O(\log(|\calC_{\eps}|)/\eps^2)$ and running time $\wt{O}(|\calC_\eps|/\eps^2)$.

A fundamental problem in distribution learning arises when the elements of the class $\calC$ can be expressed as a sum $X = X_1 + \cdots + X_n$ of $n \in \nats$ independent but not identical random variables (SIIRVs). This problem has been extensively studied in the Theoretical Computer Science literature in the last decade. The seminal work of \cite{daskalakis2015sparse, daskalakis2015learningPBD, diakonikolas2016optimal} settled the fundamental problem of covering and learning sums $X$ of independent but not identical Bernoulli random variables, where they prove the surprising result that the number of samples needed for learning the random variable $X$ in this case is independent of $n$ and almost the same as the number of samples needed to learn a single Bernoulli random variable. Subsequently, \cite{daskalakis2013learningSIIRV, diakonikolas2016fourier} solved the problem of learning sums of integer random variables with support from $0$ to $m - 1$ and otherwise follow arbitrary distribution using $\wt{O}(m/\eps^2)$ samples (again independent of $n$). Follow-up works have also considered multidimensional distributions again with bounded support size \citep{daskalakis2015structure, daskalakis2016size, diakonikolas2016properly}. This line of work has found applications in Game Theory for computing equilibrium in anonymous games \citep{daskalakis2016size, diakonikolas2016properly, goldberg2017query, cheng2017playing}, in Mechanism Design for designing auctions \citep{goldberg2015auction}, and in Stochastic Optimization \citep{de2018boolean}. Crucially, such applications make use of the delicate structure of such sums (reflected in proper sparse covers) and are not necessarily implied by the learning results. In fact, learning SIIRVs could also be seen as a fundamental (and not trivial) application of the corresponding covering results.

All the previous work in this literature considered learning sums of random variables whose support is bounded in size and the maximum elements in the support are also bounded. In particular, \cite{daskalakis2013learningSIIRV} observed that if the support of the terms is unbounded, then the sample complexity of learning the distribution of the sum will depend on the number of terms in the worst case, even under the assumption that the terms have bounded moments. Moreover, the recent work of \cite{de2018learning} showed that, even when the size of the support is $4$, there should be a dependence of sample complexity on the maximum value of their support in general. In many settings though, both in Game Theory and in Stochastic Optimization, it is natural to encounter random variables with large or even infinite support. In these cases, any algorithm whose sample or time complexity depends on the support size or the maximum value of the random variable can be very inefficient or even useless. The above discussion gives rise to our first challenge:
\begin{challenge}
[Infinite Support]
\label{challenge:1}
When is it possible to learn SIIRVs with infinite support with a number of samples independent of $n$ and the maximum element of the support?
\end{challenge}

Note that, in our setting, the bounds of \cite{de2018learning} are only interesting from a qualitative perspective, since they focus on the (very weak) dependence of the sample complexity on the maximum value of the collective support (for which they provide tight bounds), but they enable doubly exponential dependence on the size of the collective support in their upper bounds.

Proper sparse covering (and hence proper learning due to the covering method), is a quite delicate requirement: To the best of our knowledge, the only known results for properly covering sums of independent univariate random variables, apply to the class of Poisson Binomial distributions \cite{daskalakis2015sparse,diakonikolas2016properly} and the class of $m$-SIIRVs (each summand is supported on $0$ to $m-1$) \cite{diakonikolas2016optimal}. This is the second challenge: 
\begin{challenge}
[Proper Covers]
\label{challenge:2}
Are there general families of SIIRVs with infinite support that admit \emph{proper} sparse covers in total variation distance?
\end{challenge}

\subsection{Our Contribution}\label{sec:contribution}
We initiate the study of SIIRVs with unbounded and even infinite support (SIIURVs). 
There are two aspects of our work. First, we overcome the aforementioned lower bounds with an appropriate set of simple assumptions. Under our assumptions, we prove that the sample complexity of the learning problem is independent from the number of terms, giving an answer to Challenge \ref{challenge:1} (see Section \ref{section:siiurvs}). Our result is important from a theoretical perspective, since in the distribution learning setting there is a lack of a tight combinatorial characterization of the sample complexity, unlike, for example, the case of binary classification. The standard upper bound, metric entropy, includes, in our case, a dependence on the number of terms.
Second, we give an answer to Challenge \ref{challenge:2} (and then to Challenge \ref{challenge:1}) by properly covering (and then learning) SIIRVs with structured distributions. In particular, we focus on SIIRVs where each term is a member of a given exponential family of distributions $\calE$, which we call $\calE$-SIIRVs or SIIERVs. The exponential family paradigm is a multi-parametric, extremely expressive framework that captures many interesting families of distributions. 
Our results identify delicate structural properties for a quite broader class of random variables than what has been previously known and, importantly, demonstrate that the size of the support is not an utter impediment in acquiring such delicate results.
We present our result on SIIERVs in Section \ref{section:proper}.

\paragraph{Results for SIIERVs.} 
Our main results concern the family $\calE$-SIIRV of $X = X_1 + ... + X_n$. Each $X_i$ is a member of an exponential family of distributions $\calE$, that is the probability mass function of $X_i$ at the point $x \in \ints$ is proportional to the quantity $\exp( - \vec{a}_i \cdot \vec{T}(x))$, where $\vec{T} : \ints \to \reals^k$ is the vector of \textit{sufficient statistics} of $\calE$ and $\vec{a}_i$ is the vector of parameters of $X_i$ that belongs to the \textit{parameter space} $\calA \subseteq \reals^k$ of $\calE$. So in our setting, for every $X_i$ the sufficient statistics $\vec{T}$ are the same but the parameter vector $\vec{a}_i$ is different for every $i$. The sum $X$ will be called an $\calE_{\vec T}(\calA)$-SIIRV of order $n$.

\begin{infassumption}[Assumption \ref{assumption:proper}] \label{asp:informal}
Assume that
there exist constants $L,B,\gamma,\Lambda>0$ so that
the exponential family $\calE = \calE_{\vec T}(\calA)$
is well-defined and:
\begin{enumerate}
    \item (Geometry) $\calA$ is closed, path-connected and its conical hull is a polyhedral cone.
    \item (Modes) Every distribution in $\calE$ is unimodal and the modes lie in $[-L,L]$.
    \item (Bounded Moment) Every distribution in $\calE$ has fourth central moment at most $B$.
    \item (Variance) The variance of each distribution in $\calE$ is lower bounded by $\gamma$.
    \item (Covariance) For any $\vec a$ in the convex hull of $\calA$, it holds $\Cov_{\vec a}(\vec T(W)) \preceq \Lambda \cdot I_k$.
\end{enumerate}
\end{infassumption}

\noindent For a discussion on the minimality of our assumptions, we refer to Sections \ref{section:siiurvs} and \ref{section:proper}. In the next results, the set $\calA'$ is a superset of $\calA \subseteq \reals^k$ (see the discussion after the statements).  

\begin{inftheorem}[Weakly-Proper Covering Theorem \ref{theorem:covering-siiervs}]
  Under Assumption \ref{asp:informal}, for any $\eps > 0$, there exists a set of distributions $\calC_{\eps}$ that $\eps$-covers the family of $\calE_{\vec T}(\calA)$-SIIRVs of order $n$ in total variation distance. The set $\calC$ has size $(n/\eps)^{O(k)} + 2^{k \cdot \poly(1/\eps)}$ and each element of $\calC_{\eps}$ is an $\calE_{\vec T}(\calA')$-SIIRV of order $\Theta(n)$.
\end{inftheorem}

\begin{inftheorem} [Learning Theorem \ref{theorem:learning-SIIERVs-main}]
  Under Assumption \ref{asp:informal}, given $m = k \cdot \wt{O}(1/\eps^2)$ samples from an unknown $\calE_{\vec T}(\calA)$-SIIRV $X$ of order $n$, there exists an algorithm that outputs $\wh{X}$ so that $\tv(X,\wh{X}) \leq \eps$ with high probability. Moreover, $\wh{X}$ is an $\calE_{\vec T}(\calA')$-SIIRV of order $\Theta(n)$.
\end{inftheorem}

\noindent\textbf{Weakly-Proper Covering.} We say that a cover is a weakly-proper cover for the family of $\calE_{\vec T}(\calA)$-SIIRVs of order $n$ if its elements belong to the family of $\calE_{\vec T}(\calA')$-SIIRVs with parameters in a slightly larger set $\calA' \subseteq \reals^k$ and with possibly more than $n$ terms. In the rest of the paper, we mostly stick with the term proper for brevity (see also Appendix \ref{subsubsection:proper-meaning}). We think of $\calA$ as input to the problem, but we 
focus on the various challenges arising by the nature of this problem instead of possible adversarial selections of $\calA$.

The covering result gives an answer to Challenge \ref{challenge:2} and the learning one provides an algorithm with sample complexity independent of $n$ and the maximum element of the support (Challenge \ref{challenge:1}).
In the above informal theorem, we have treated the relevant parameters of the exponential
family $\calE$ (e.g., $B$) as constants. If we also consider the accuracy $\eps$ to be constant, the learner runs in time $n^{O(k)}$. 
The assumptions about the central moments as well as the covariance matrix are standard. The assumption regarding the geometry of $\calA$ is a mild, technical assumption, that has, however, important technical implications. The variance lower bound is a substitute of particularly subtle -- and most probably not omnipotent -- technical tools that can be used to discard low variance terms in special cases.
Finally, the assumption about the modes provides the structure needed to apply the most powerful tool we possess to confront Challenge \ref{challenge:1}: quantitative versions of the Central Limit Theorem. Moreover, Challenge \ref{challenge:2} restricts the flexibility we have in applying such a tool, which we believe to indicate that our assumption about modes is, in some sense, essential for our purposes.

Our Assumptions \ref{asp:informal} do not exclude any reasonable exponential family and our methods capture (among others) Geometric, Bernoulli, Poisson, Zeta, Gamma, Gaussian, Laplacian distributions and interpolations thereof (see Appendix \ref{appendix:examples}). For instance, our results apply to sums with both Gaussian and Laplacian terms. In particular, the naturally occuring Zeta distributions (Zipf's law \citep{Chao1949HumanBA}) are not log-concave and no non-trivial learning results are known on sums thereof, even without requiring proper learning.

\paragraph{Technical Contributions.} 
For a more detailed discussion about our novel technical features, we refer to Sections \ref{section:technical-1} and \ref{section:technical-2}.
First, we provide a fundamental structural result about exponential families that satisfy our assumptions, by reducing the problem of properly sparsely covering an exponential family as such to the standard problem of covering a bounded subset of $\reals^k$ (for some $k\in\nats$) in Euclidean distance. The main challenge, which familiar results about exponential families do not resolve, is that the parameter space of the exponential family may be unbounded. We show that a bounded subset of the parameter space approximately generates any distribution in the family (Theorem \ref{theorem:projection}), which we believe to be of independent interest.
To this end, we identify an analogy between the geometry of exponential families and polyhedral theory; we essentially reduce probabilistic properties of exponential family distributions to geometric properties of polyhedral cones, for which we settle a novel result (Theorem \ref{theorem:geometry-up}). We also prove that for any distribution with parameter vector that has a sufficiently large (yet bounded) norm, the number of important points of the support is bounded, which leads to the resolution of the main technical challenge. Secondly, we provide a continuity argument which implies that for any $\calE$-SIIRV $X$ of order $n$, there exists some $\calE$-SIIRV $Y$ which is the sum of i.i.d. random variables in the family $\calE$ such that the distance between the expectation of $X$ and the expectation of $Y$, as well as the distance between the variance of $X$ and the variance of $Y$ are bounded. This is important in order to prevent our learning algorithm from running in time exponential to the number of terms $n$. 

\paragraph{Roadmap.} In Section \ref{section:siiurvs}, we establish a simple set of assumptions that is sufficient to address Challenge \ref{challenge:1}. At the same time, we revisit the main ideas used in the literature of SIIRVs and describe how we implement them on our setting. In Section \ref{section:proper}, we continue with the presentation of our main contribution, namely, covering and learning SIIERVs (addressing Challenges \ref{challenge:2} and then \ref{challenge:1}). Additional notation, missing proofs, technical tools that we use as well as some applications of our results can be found in the Appendix.

\section{Warm-Up: Structure and Learning of SIIURVs}
\label{section:siiurvs}

We show, as a warm-up, that there exists a set of assumptions under which learning in total variation distance the distribution of an unknown sum of (at most) $n$ independent random variables with possibly unbounded support can be done using a number of independent samples that does not increase with $n$. In particular, let $\calD$ be a family of distributions over $\ints$. We consider sums of independent integer-valued random variables of order $n$ of the form $X = \sum_{i\in[n']} X_i$, where $n'\le n$, $X_i \sim D_i \in \calD$. We call $X$ a $\calD$-SIIRV (or SIIURV, i.e., sum of independent integer random variables with unbounded support) of order $n$. 

\begin{assumption} 
\label{assumption:SIIURV}
We make the following assumptions for the family of distributions $\calD$.
\begin{enumerate}
    \item\label{cond:unimodal-SIIURV} Every distribution in $\calD$ is \textnormal{(1a)} unimodal 
    and \textnormal{(1b)} the mode is assigned probability at most equal to $1-\gamma$, for some constant $\gamma\in(0,1)$ (common for all distributions in $\calD$).\footnote{A unimodal distribution could have many consequent modes, each assigned equal amount of mass.}
    \item\label{cond:bounded-modes-SIIURV} Every mode of any distribution in $\calD$ lies within a (common) interval of constant length $L>0$.
    \item \label{cond:bounded-centered-SIIURV} The fourth central moment $\E \left[|W-\E [W]|^4 \right]$ of each distribution in $\calD$ is upper bounded by a constant $B>0$ (uniformly for all distributions in $\calD$).
\end{enumerate}
\end{assumption}

\paragraph{Minimality of Assumption \ref{assumption:SIIURV}.} 
Removing condition \ref{cond:unimodal-SIIURV} (unimodality and a bound on the mass assigned on the mode) would activate a lower bound on the sample complexity (Observation 1.3 from \cite{daskalakis2013learningSIIRV}) that involves some dependence on the number of terms $n$; we aim for sample complexity independent from $n$. The terms considered in the lower bound all have zero as a mode (essentially satisfying condition \ref{cond:bounded-modes-SIIURV} in the case of multimodal distributions) and all of their moments are upper bounded by a sequence of values (stronger than condition \ref{cond:bounded-centered-SIIURV}). They are, however, not unimodal and they assign almost all of their mass to zero (condition \ref{cond:unimodal-SIIURV} does not hold). Waiving condition \ref{cond:bounded-modes-SIIURV} enables one to form a family which does not have a sparse cover, even when the sums only have a single term. In particular, we can consider a sequence of arbitrarily shifted Bernoulli distributions with parameter $1/2$, each of which has a distance at least equal to $1/2$ from any other distribution of the sequence. Moreover, the aforementioned sequence does not violate conditions \ref{cond:unimodal-SIIURV} or \ref{cond:bounded-centered-SIIURV}. Finally, condition \ref{cond:bounded-centered-SIIURV} is important, since it rules out, for example, the case of $\calD$ containing all Geometric distributions (enabling probability of success arbitrarily close to $0$). In this case, there is some constant $\eps>0$ such that we may consider an infinite sequence of geometric distributions with diminishing success probabilities $p_n=2^{-n}$ with pairwise statistical distance at least $\eps$. The degree of the moment we assume to be bounded is $4$, which is useful in the dense case, to establish the rate of convergence of the sum to a discretized Gaussian distribution; importantly, the degree is constant.
We get the following result.

\begin{theorem}
[Learning]
\label{theorem:learning-SIIURVs-main}
Set $n \in \nats$ and $\calD$ some family of distributions satisfying Assumption \ref{assumption:SIIURV}.
Let $\eps, \delta \in (0,1)$ and $X$ be an unknown $\calD$-\textnormal{SIIRV} of order $n$. 
There exists an algorithm (Figure \ref{algorithm:learning-SIIURVs}) with the following properties: Given $n, \eps, \delta, L, B, \gamma$ and sample access to $X$, the algorithm uses $m = O (\frac{1}{\eps^2}\cdot \log(1/\delta) ) + O(\poly(B,1/\gamma,1/\eps) \cdot \log(L) )$ samples from $X$ and, in time $
\poly (m, L^{\poly(B,1/\gamma, 1/\eps)}  )\,,$
outputs 
a (succint description of a) distribution 
$\wt{X}$ 
with
$\tv(X, \wt{X}) \leq \eps$, with probability 
$1-\delta$.
\end{theorem}

Theorem \ref{theorem:learning-SIIURVs-main} is based on a common technique used in problems related to SIIRVs, which uses quantitative versions of the Central Limit Theorem (like Lemma \ref{lemma:discr-gaussian-approx} of \cite{chen2010normal}) to reduce the learning problem into two sub-problems; covering $\calD$ in total variation distance and estimating the variance and expectation of the unknown SIIRV.

\begin{theorem}
[Structure of SIIURVs]
\label{theorem:structural-SIIURV}
Set $n \in \nats$ and $\calD$ some family of distributions satisfying Assumption \ref{assumption:SIIURV}. For any $\eps > 0$, and any $\calD$-\textnormal{SIIRV} $X$ of order $n$, there exists some $Y$ such that $\tv(X,Y) \leq \eps$ and either (i) $Y$ is a random variable among $L^{\poly(B,1/\gamma,1/\eps)}$ candidates that are independent from the particular $X$ (sparse form) or (ii) $Y$ is a discretized Gaussian random variable with $\E[X] = \E[Y]$ and $\Var(X)=\Var(Y)$ (dense form).
\end{theorem}

Hence, the learner first runs two different learning procedures, corresponding to the sparse and dense forms of Theorem \ref{theorem:structural-SIIURV}. For the sparse case, it runs a tournament over the possible candidates and in the dense one, it computes the parameters of the (potentially) nearby discretized Gaussian. From the two procedures, two hypotheses are obtained and finally hypothesis testing is performed in order to select the correct one. Our focus on the Gaussian approximation is the reason why we assumed that condition \eqref{cond:unimodal-SIIURV} holds.
In principle, there might be ways to relax Assumption \ref{assumption:SIIURV} and learn SIIURVs (independently from $n$) with different techniques, but we are particularly interested in using the Gaussian approximation in the dense case (compare, e.g., with the approach of \cite{daskalakis2013learningSIIRV}), since it will be pivotal to our main technical and conceptual contribution outlined in the following section.

\section{Structure and Proper Learning of SIIERVs}
\label{section:proper}
In the seminal work of \cite{daskalakis2015sparse}, it was shown that the class of Poisson Binomial Distributions (i.e., sums of independent indicator random variables) has a structure with similar properties as the one presented in Theorem \ref{theorem:structural-SIIURV}. Crucially, however, their results had an additional property. In both sparse and dense cases, the candidate distributions (i.e., the representatives of the class) were Poisson Binomial Distributions themselves (namely, the class admitted \textit{proper sparse covers}). The result unlocked the possibility of \textit{proper} learning for PBDs (see \cite{daskalakis2015learningPBD,diakonikolas2016properly}), with sample complexity independent from the number of terms. Besides the result of \cite{daskalakis2015sparse}, to the best of our knowledge, there are \emph{no further known results for properly covering sums of integer-valued (univariate) random variables, with terms in some structured-parametric family of distributions}.\footnote{\cite{diakonikolas2016optimal} provide proper sparse covers for the class of $m$-SIIRVs; nevertheless, this family does not have the ``structure'' we focus on this work since it is nonparametric.}

We provide significantly general results for the structure of sums of integer (and unbounded) random variables, under the condition that the terms belong in any fixed exponential family that satisfies a set of assumptions. Our results imply proper learning with sample complexity independent from the number of terms. We consider exponential families supported on the whole (unbounded) set of integer numbers, although our results could be extended to exponential families supported on some subset of $\ints$, like $\natszero$.

\subsection{Preliminaries and Definitions}
 
\noindent\textbf{Exponential Families.} For $k\in\natszero$, $\calA\subseteq\reals^k$ and $\vec T:\ints\to\reals^k$, we denote with $\calE_{\vec T}(\calA)$ the \textbf{exponential family} with \textbf{sufficient statistics} $\vec T$ and \textbf{parameter space} $\calA$. If $W\sim\calE_{\vec T}(\vec a)$ for some $\vec a\in\calA$, then $
    \Pr[W=x] \propto \exp(-\vec a\cdot \vec T(x))$ for any $x\in\ints$ . We will use $\Pr_{\vec a}[W=x]$ (similarly $\E_{\vec a}[W]$ and $\Var_{\vec a}(W)$ for expectation and variance correspondingly) to refer to the probability that $W=x$ given that $W\sim\calE_{\vec T}(\vec a)$, whenever it is clear that the distribution of $W$ belongs in $\calE_{\vec T}(\calA)$. \footnote{
In general, an exponential family $\calE$ over $\ints$ is also defined in terms of some carrier measure  $h : \ints \mapsto \reals_+$ so that if $W \sim \calE$, then $\Pr[W=x]\propto h(x)\cdot\exp(-\vec a\cdot \vec T(x))$ and denote $W\sim\calE_{\vec T,h}(\vec a)$. We can reduce this setting to $h \equiv 1$ by considering $\calA'=\calA\times\{1\}$ and $\vec T' = (\vec T, -\log_e(h(x)))$.}
\smallskip

\noindent\textbf{SIIERVs.} We will consider distributions of sums of the form $X=\sum_{i\in[n']}X_i$, where $n'\le n$, $(X_i)_i$ independent and $X_i\sim\calE_{\vec T}(\vec a_i)$ with $\vec a_i\in\calA$. We call this class of distributions as \textbf{$\calE_{\vec T}(\calA)$-\siirvs\ of order $n$} or simply \textbf{\siiervs}\ when $n,\calA$ and $\vec T$ are clear by the context.
\smallskip

\noindent\textbf{Set Operators.} We say that $\setoperator$ is an extensive set operator on $\reals^k$ if for any set $\calA\subseteq\reals^k$, we have that $\calA\subseteq \setoperator\calA \subseteq \reals^k$. We use the following extensive set operators:
For the definitions of the 
\textbf{convex hull} $\chull$ and
the \textbf{conical hull} $\conehull$ operators, we refer to \eqref{eq:convex} and \eqref{eq:cone}. The \textbf{$\varrho$-conical hull} operator $  
        \varrho\-\conehull \calA= \calA \cup (\conehull\calA \setminus \ball_\varrho(\vec 0))$, i.e., the  $\varrho\-\conehull$ operator inserts in $\calA$ all points of the conical hull of $\calA$ with norm at least $\varrho$. For additional notation, we refer to the Appendix \ref{appendix:notation}.

\subsection{Main results}

Our results for $\calE_{\vec T}(\calA)$-SIIRVs hold under the following set of assumptions about $\calE_{\vec T}(\calA)$.

\begin{assumption}
\label{assumption:proper}
Let $k \in \nats$, $\vec T:\ints\to\reals^k$ and $\calA\subseteq\reals^k$. 
Denote $\calA_{\varrho} = \varrho\-\conehull\calA$ and $\overline{\calA}_{\varrho} = \chull \calA_{\varrho}$ for $\varrho > 0$.
We assume that there exists some constant $\varrho>0$ so that
the exponential family $\calE_{\vec T}(\calA_{\varrho})$
is well-defined 
and the following hold:
\begin{enumerate}
    \item\label{assumption:geometry} The parameter space $\calA$ is closed and its conical hull $\conehull\calA$ is a polyhedral cone.

    \item\label{assumption:unimodal} Every distribution in $\calE_{\vec T}(\calAvarrho)$ is unimodal.
    
    \item\label{assumption:bounded-modes} There exists some constant $L>0$ such that every mode of every distribution in $\calE_{\vec T}(\calAvarrho)$ lies within the interval $[-L,L]$.
    
    \item\label{assumption:bounded-moments} There exists some constant $B>0$ such that every distribution in $\calE_{\vec T}(\calA_{\varrho})$ has fourth central moment that is upper bounded by $B$, i.e.,
    $\E_{\vec a} \left [ \left|W-\E_{\vec a}[W]\right|^4 \right] \le B, \text{ for any }\vec a\in\calA_{\varrho}\,.$
    
    \item\label{assumption:bounded-max-eval} There exists some constant $\Lambda>0$ such that $\Cov_{\vec a}(\vec T(W)) \preceq \Lambda \cdot I_k, \text{ for any }\vec a\in \overline{\calA}_{\varrho}\,.$
    
    \item\label{assumption:connectivity} The parameter space $\calA$ is path-connected. 

    \item\label{assumption:variance-lower-bound} There exists some constant $\gamma>0$ such that $\Var_{\vec a}(W) \ge \gamma, \text{ for any }\vec a\in\calA\,.$
    
\end{enumerate}
\end{assumption}

Assumptions \eqref{assumption:unimodal}-\eqref{assumption:bounded-moments} correspond to conditions in Assumption \ref{assumption:SIIURV}, but we make some additional ones. In particular, variants of assumption \eqref{assumption:bounded-max-eval} have been used in the past (see \cite{diakonikolas21IsingRobust}), e.g., for parameter estimation in exponential families and essentially ensure that parameter vectors close in Euclidean distance correspond to distributions close in statistical distance. Assumptions \eqref{assumption:connectivity} and \eqref{assumption:variance-lower-bound} are important only in the case that the number of terms $n$ in the sum is large. Assumption \eqref{assumption:variance-lower-bound} ensures that as the number of terms increases, the distribution approaches its limit (i.e., the discretized Gaussian distribution), with some constant rate. In fact, the variance lower bound is a substitute of particularly subtle but specialized technical tools that can be used to discard low variance terms in some specific cases (e.g., see the so called massage step of \cite{daskalakis2015sparse} which refers to Poisson Binomial Distributions and our own Appendix \ref{appendix:pnbds}, which refers to sums of independent geometric RVs, namely, Poisson Negative Binomial Distributions).  Finally, assumption \eqref{assumption:connectivity} implies some kind of continuity with respect to the parameter vector which is important for proper learning so that the behavior of a sum of a large number of terms can be described by a constant number of parameter vectors (in our case exactly one). For a discussion on the verification of the assumptions, see Appendix \ref{appendix:verification}.

We stress that assumptions \eqref{assumption:unimodal}-\eqref{assumption:bounded-max-eval} are imposed on a slightly expanded exponential family (by extending the parameter space $\calA$ to $\calAvarrho$). In fact, our analysis involves the study of the influence that changes in the parameter vector have on the corresponding distribution, given some properties generated by our assumptions. Extending the space on which such properties hold enables the study of a wider range of changes of the parameter vector. On the contrary, assumption \eqref{assumption:variance-lower-bound} is imposed on $\calA$ (and not its extended version). It essentially excludes some parameter vectors with large norms, but, in general, it \textit{does not imply} that $\calA$ is bounded. In general, $\calA$ should be though of as the space of parameters for the input and $\calAvarrho (\supseteq \calA)$ the space of parameters for the output.

\begin{theorem}
[Structure of SIIERVs] 
\label{theorem:covering-siiervs}
Set $n\in\nats$ and $\calE_{\vec T}(\calA)$ some exponential family satisfying Assumption \ref{assumption:proper}. Let $\calAvarrho = \varrho\-\conehull\calA$. There exists some value $\theta = \theta(\calA, \vec T)>0$ 
such that, for any $\eps \in (0,1)$ and any $\calE_{\vec T}(\calA)$-\textnormal{SIIRV} $X$ of order $n$, there exists some random variable $Y$ such that $\tv(X,Y)\le \eps$ and either (i) Y is an $\calE_{\vec T}(\calAvarrho)$-\textnormal{SIIRV} among $ ( \frac{\varrho \cdot \sqrt{\Lambda}}{\descriptivityparam}  )^{k\cdot \wt{O} (\frac{1}{\eps^2} )\cdot \poly (B, L, \frac{1}{\gamma}  )}$ candidates (sparse form) or (ii) Y is a sum of i.i.d. $\calE_{\vec T}(\calAvarrho)$-random variables among $ ( n^2\cdot \poly (B, \frac{1}{\gamma}  ) \cdot O( \frac{\varrho\cdot\sqrt{\Lambda}}{\descriptivityparam\cdot \eps}  )  )^k$ candidates (dense form).
\end{theorem}

\noindent The role of the quantity $\descriptivityparam$ is thoroughly discussed in Section \ref{section:technical-1}. Roughly, the parameter space $\calA$ and the trajectory of the sufficient statistics $\vec T(x)$, $x\in\ints$ are associated with a finite number of polyhedral cones which are important parts of the structure of the family $\calE_{\vec T}(\calA)$. The value of $\descriptivityparam$ depends on the geometry of the specified polyhedral cones.

The structural result implies a proper learning algorithm (Figure \ref{algorithm:learning-SIIERVs}) which roughly applies the tournament method and hypothesis selection routines over the cover both in the sparse and (after an additional important step) in the dense case (see Proposition \ref{proposition:hypothesis-de}). For the learning result, we assume access to some sample and evaluation oracles (see Appendix \ref{subsection:learning-SIIERVs}). Such access is needed in order to apply hypothesis testing over our covers in a formal sense. Let $D$ be a distribution over $\ints$. Consider the sample oracle $\mathrm{EX}(D)$ that, when invoked, returns a sample with law $D$ and the approximate evaluation oracle $\mathrm{EVAL}_D(\beta)$ that, when invoked with query $x \in \ints$, returns a value $q$ that satisfies $D(x)/(1+\beta) \leq q \leq (1+\beta)D(x)$ for some $\beta > 0$ (this oracle is used in \cite{de2014learning}). Below, the relation between $\beta$ and the desired learning accuracy $\eps$ is provided by \cite{de2014learning}. 
Finally, assume that the cover of Theorem \ref{theorem:covering-siiervs} (the set of candidate distributions) of radius $\eps$ can be constructed in time $T_{\mathrm{c}} = T_{\mathrm{c}}(\calA, n, \eps, \varrho, L, B, \Lambda, \gamma, \theta, \vec T)$ (see Remark \ref{remark:runtime} for a discussion on the runtime). 

\begin{theorem}
[Learning SIIERVs]
\label{theorem:learning-SIIERVs-main}
Set $n\in\nats$ and $\calE_{\vec T}(\calA)$ some exponential family satisfying Assumption \ref{assumption:proper}. Let $\calAvarrho = \varrho\-\conehull\calA$. There exists 
$\theta = \theta(\calA, \vec T) > 0$
such that for any $\eps, \delta \in (0,1)$ there exists an algorithm (Figure \ref{algorithm:learning-SIIERVs}) with the following properties: Given $n, \eps, \delta, B,L,\Lambda,\gamma,\theta$ and (i) sample access to an unknown $\calE_{\vec T}(\calA)$-\textnormal{SIIRV} $X$ of order $n$, (ii) $\mathrm{EX}(\dgaussvar(\mu, \sigma^2))$ for any $\mu, \sigma^2$ and (iii) $\mathrm{EX}(D)$ and evaluation oracle access to $\mathrm{EVAL}_D(\beta)$ for any $D \in \calE_{\vec T}(\calA_\varrho)$ for some $\beta \geq 0$ with $(1+\beta)^2 \leq 1 + \eps/8$, the algorithm uses $    
m = O  (\frac{1}{\eps^2} \log(1/\delta)) + k \cdot \wt{O}(\frac{1}{\eps^2})\cdot \poly(B, \frac{1}{\gamma} ) \cdot \log ( \frac{\varrho\cdot \Lambda}{\descriptivityparam} )$ samples from $X$ and, in time $
\poly (m, 2^{k \cdot \wt{O}(1/\eps^2)\cdot \poly(B, L, 1/\gamma)}, n^k, (\varrho \Lambda/\theta)^k, T_{\mathrm{c}} )\,,$
outputs a (succint description of a) distribution $\wt{X}$ which satisfies
$\tv(X, \wt{X}) \leq \eps$, with probability $1-\delta$. Moreover, $\wt{X}$ is an $\calE_{\vec T}(\calA_\varrho)$-SIIRV of order $(\sqrt{B} / \gamma)\cdot n$. 
\end{theorem}

In particular, for the dense case hypothesis, the learner runs in two steps. First, similarly to the learner of Theorem \ref{theorem:learning-SIIURVs-main}, it estimates the expectation and the variance of $X$ by $\wt{O}(1/\eps^2)$ samples, thereby specifying a discretized Gaussian that is close to $X$. However, as a second step, it runs the tournament hypothesis testing procedure between \emph{the estimated Gaussian} and the candidate distributions of the dense form. Importantly, the tournament selection does not need to use real samples from $X$, but, instead, it generates draws from the Gaussian. 

In the following sections, we analyze the structural result of Theorem \ref{theorem:covering-siiervs}, which is our main technical contribution. In Appendix \ref{appendix:pnbds}, we provide an example corollary of our methods for the case of Poisson Negative Binomial Distributions (Theorem \ref{theorem:main-cover-pnbds}).

\subsection{Sparsifying the Parameter Space of an Exponential Family}
\label{section:technical-1}
We first solve the proper covering problem in the simplest possible case of $n=1$, i.e., we provide sparse covers for any exponential family $\calE_{\vec T}(\calA)$ satisfying (some of the assumptions in) Assumption \ref{assumption:proper}. The following result constitutes the main building block of our analysis in the case of general $n$ (see Section \ref{section:technical-2}). The proof of this Theorem can be found in Appendix \ref{proof:discretization}.

\begin{theorem}
[Sparsifying the Parameter Space]
\label{theorem:single-term}
    Under assumptions \eqref{assumption:geometry}, \eqref{assumption:unimodal}, \eqref{assumption:bounded-modes}, \eqref{assumption:bounded-moments} and \eqref{assumption:bounded-max-eval}, there exists $\descriptivityparam = \descriptivityparam(\calA, \vec T)>0$, 
    so that for any $\eps\in(0,1)$, there exists $\calB \subseteq\varrho\-\conehull\calA$ with $|\calB|\le (\widetilde{O}(\frac{\sqrt{\Lambda}\cdot \varrho}{\eps} + \frac{\sqrt{\Lambda}}{\eps\cdot \descriptivityparam}  + \frac{\sqrt{\Lambda}}{\eps\cdot \descriptivityparam}\cdot \log(B) ) )^k$ so that, for any $\vec a\in\calA$,
    $
        \tv(\calE_{\vec T}(\vec a),\calE_{\vec T}(\vec b))\le \eps\,, \text{ for some }\vec b\in\calB\,.
    $
\end{theorem}

\noindent The idea behind Theorem \ref{theorem:single-term} has two parts. First, we show that, although our assumptions do not exclude the possibility that $\calA$ is unbounded, there exists some bounded set $\calA'\subseteq\calAvarrho:=\varrho\-\conehull\calA$ so that $\calE_{\vec T}(\calA')$ covers $\calE_{\vec T}(\calA)$ in TV distance. Second, due to assumption \eqref{assumption:bounded-max-eval}, Lemma \ref{lemma:exp-fami-tv-kl} and Pinsker's inequality, we may discretize $\calA'$ (with standard sparse covers in Euclidean distance) to get a sparse cover for $\calE_{\vec T}(\calA')$ (which will also be a proper sparse cover for $\calE_{\vec T}(\calA)$). In the rest of the current section, we will discuss about the main technical challenge, namely the proof of the first part of our idea, which is formally stated in the following theorem (for the proof, see Appendix \ref{proof:projection-step}).

\begin{theorem}
[Bounding the Parameter Space]
\label{theorem:projection}
Under assumptions \eqref{assumption:geometry}, \eqref{assumption:unimodal}, \eqref{assumption:bounded-modes} and \eqref{assumption:bounded-moments}, there exists $\descriptivityparam = \descriptivityparam(\calA, \vec T)>0$, 
such that for any $\eps\in(0,1)$ and any $\vec a\in \calA$, there exists $\vec b\in \varrho\-\conehull \calA$ with $\|\vec b\| \le (\varrho+\frac{1}{\descriptivityparam}) \cdot \ln(1/\eps) + \frac{\ln(B)}{2\descriptivityparam} + O(\varrho + \frac{1}{\descriptivityparam})$ so that $
    \tv(\calE_{\vec T}(\vec a),\calE_{\vec T}(\vec b)) \le \eps\,.$
\end{theorem}

In order to bound the parameter space, one has to analyze the behavior of a distribution $\calE_{\vec T}(\vec a) \in \calE_{\vec T}(\calA)$ as the norm of the parameter vector $\vec a$ increases (and its direction is fixed). Any point $x\in\ints$ is assigned by $\calE_{\vec T}(\vec a)$ mass proportional to $\exp(-\vec a\cdot \vec T(x))$. Let $\modes_{\vec a}$ be the set of modes of $\calE_{\vec T}(\vec a)$, i.e., the set of (global) maximizers of its probability mass function. Note that rescaling $\vec a$ does not alter the positions of the modes, since the order of the quantities $\vec a\cdot \vec T(x)$ (for $x\in\ints$) is preserved. Moreover, as the norm of the parameter vector increases, the distribution $\calE_{\vec T}(\vec a)$ tends to become the uniform distribution over $\modes_{\vec a}$. 

It turns out, however, that the direction of $\vec a$ affects the rate with which $\calE_{\vec T}(\vec a)$ tends to its limit. Consider, for example, the case $\vec T(x) = (x,x^2)$ and $\vec a_\delta = r(1,1+\delta),$ for some $\delta>0$ and $r>0$. The distribution $\calE_{\vec T}(\vec a_\delta)$ has a unique mode on $x=0$, but as $\delta$ tends to $0$, for any fixed $r$, the point $x=-1$ tends to become a mode of $\calE_{\vec T}(\vec a_\delta)$. Therefore, it is possible that there exists some integer $x$ such that for any fixed parameter vector norm $r\ge \varrho$, there are directions $\wh{\vec a}$ in $\calA$ so that $\calE_{\vec T}(r\wh{\vec a})$ assigns to $x$ mass arbitrarily close to the mass it assigns to a mode. Bounding the parameter space is, hence, not straightforward, since there is no uniform (over the directions in $\calA$) threshold for the parameter vector's norm upon which every distribution is close to its limit. In other words, \textit{orthogonally} projecting the parameter vectors with large norms on any fixed radius sphere does not work. We, therefore, have to develop some further geometric intuition.

We claim that, under assumptions \eqref{assumption:geometry}, \eqref{assumption:unimodal}, \eqref{assumption:bounded-modes} and \eqref{assumption:bounded-moments}, there is a way to project (not necessarily orthogonally) any given parameter vector $\vec a\in \calA$ with large norm on an origin-centered sphere with bounded radius so that the resulting distribution is close to $\calE_{\vec T}(\vec a)$ in TV distance. We prove our claim in two steps. \textit{First,} we prove our claim with respect to a new notion of distance between distributions (instead of TV distance), which we call \textit{structural distance}. For this step, we establish an interesting connection between exponential families and polyhedral cones. 
\textit{Second,} we prove that by picking the radius of the sphere on which we project to be large enough, the bounds in structural distance from the previous step imply bounds in total variation distance.

We proceed with a formal definition of the structural distance, which is a distance metric (Lemma \ref{lemma:new_distance}). In a nutshell, structural distance is the minimum possible relative threshold, such that two distributions agree (in relative terms) on every point of their domain with mass higher than the threshold (i.e., any significant point).

\begin{definition}
[Structural Distance]
\label{definition:new-distance} 
Consider two discrete distributions $D_1, D_2$ over $\mathbb{X}$ and let $p_i = \max_{x \in \mathbb{X}} D_i(x)$, $i = \{1,2\}$. The structural distance $\dnew(D_1,D_2)$ between $D_1$ and $D_2$ is the
minimum $\eps \in [0,1]$ such that for any $x \in \mathbb{X}$, at least one of the following holds:
\[
(i)~~ D_1(x) \le \eps \cdot p_1 ~~\&~~ D_2(x) \le \eps \cdot p_2\,,~~~
\textnormal{or}~~~(ii)~~ D_1(x)/p_1 = D_2(x)/p_2 \,.
\]
\end{definition}

\noindent For any $\eps\in[0,1]$, and any discrete distribution $D$ over $\mathbb{X}$, let $\wh{D}^{(\eps)}$ denote the truncation of $D$ on the points $x$ such that $D(x)\ge \eps\cdot \max_{y\in\mathbb{X}}D(y)$. Then, the structural distance between two distributions $D_1$ and $D_2$ can be described as the minimum threshold $\eps\in[0,1]$ so that the distributions $\wh{D}_1^{(\eps)}$ and $\wh{D}_2^{(\eps)}$ are identical. In that sense, structural distance measures the degree in which two distributions have the same structure (on significant points). Two distributions with different modes have structural distance $1$ (maximum possible). Structural distance is meaningful when the two distributions (or one of them) can be contrived to have the same structure on significant points, which we prove to be possible in the context of exponential families. We prove the following lemma.

\begin{lemma}
[Structural Distance and Bounding Norms]
\label{lemma:descriptivity} 
Under assumptions \eqref{assumption:geometry}, \eqref{assumption:unimodal} and \eqref{assumption:bounded-modes}, there exists some constant $\descriptivityparam>0$ such that for any $r\ge \varrho$ and any $\vec a \in \calA$ with $\|\vec a\|\ge r$, there exists some $\vec b\in \calA_{\varrho}$ and $\|\vec b\| = r$ so that
    $\dnew(\calE_{\vec T}(\vec a), \calE_{\vec T}(\vec b)) \leq e^{- \descriptivityparam \cdot r}$.
\end{lemma}

\noindent The proof of Lemma \ref{lemma:descriptivity} can be found in Appendix \ref{proof:sparsification-st}. However, we will now outline the main ingredients of the proof. Let us translate Definition \ref{definition:new-distance} in terms of the parameter vectors of exponential families. In particular, let $\vec a\in\calA$ and $\vec b\in\calAvarrho$ and $M_{\vec a}$ (resp. $M_{\vec b}$) be any mode of $\calE_{\vec T}(\vec a)$ (resp. $\calE_{\vec T}(\vec b)$). Then $\dnew(\calE_{\vec T}(\vec a),\calE_{\vec T}(\vec b)) = \eps$ means that for any $x\in\ints$ either $\vec a\cdot (\vec T(x)-\vec T(M_{\vec a})) \ge \eps$ and $\vec b\cdot (\vec T(x)-\vec T(M_{\vec b})) \ge \eps$ ($x$ is not significant) or $\vec a\cdot (\vec T(x)-\vec T(M_{\vec a})) = \vec b\cdot(\vec T(x)-\vec T(M_{\vec b}))$ ($x$ is significant). Therefore, reducing the norm of $\vec a$ without moving significantly in structural distance corresponds to a geometric problem regarding the parameter and the sufficient statistics vectors. In particular, given some $\vec a\in\calA$ with large norm, $\vec b$ should be chosen so that the quantities $\vec b\cdot \vec v_x$ for $x\in\ints$ and some sequence $(\vec v_x)_x$ of vectors that depends on $\vec a$ are constrained to be (i) equal to $\vec a\cdot \vec v_x$ when $\vec a\cdot \vec v_x$ is small and (ii) large when $\vec a\cdot \vec v_x$ is large. The number of constraints for $\vec b$ is infinite, since $x\in\ints$. However, due to unimodality, one can show that only a finite number of them is crucial (the others can be trivially satisfied);
we then make use of ($\calAvarrho=\varrho\-\conehull\calA$ and) the following (to the best of our knowledge) novel theorem we prove about the geometry of polyhedral cones.

\begin{theorem}\label{theorem:geometry-up}
    Consider any polyhedral cone $\calC\subseteq\reals^k$, $k\in\nats$, where $\calC = \{\vec u: H^T\vec u\ge \vec 0\}$ for some matrix $H\in\reals^{k\times t}$, $t\in\nats$ is a description of $\calC$ as an intersection of halfspaces. Then there exists some $\theta>0$ such that for any $\vec u\in \calC$ with $\|\vec u\|\ge 1$, there exists $\vec u'\in\calC$ with $\|\vec u'\| = 1$ so that for any column $\vec h$ of $H$ at least one of the following is true:
    \[
(i)~~ \textnormal{Either}~~ \vec h\cdot \vec u \ge \theta \textnormal{~~and~~} \vec h\cdot \vec u' \ge \theta\,,~~~
\textnormal{or}~~~(ii)~~ \vec h\cdot \vec u = \vec h\cdot \vec u' \,.
\]
\end{theorem}

The idea behind Theorem \ref{theorem:geometry-up} 
is to subtract from $\vec u$ a vector within the nullspace of the matrix $H_{\calI}^T$, where $\calI$ is the set of indices of columns $\vec h$ of $H$ such that $\vec h\cdot \vec u$ is small. In order to pick the correct vector, we use, additionally, a pivot vector $\vec w$ which satisfies $\vec h\cdot \vec w\ge \theta$ for any interesting column $\vec h$ of $H$. The vector which we subtract from $\vec u$ depends on the projections of $\vec u$ and $\vec w$ on the nullspace of $H_{\calI}^T$.  We believe that Theorem \ref{theorem:geometry-up} is of independent interest.
\begin{proof}[Proof of Theorem \ref{theorem:geometry-up}]
    Due to Minkowski-Weyl theorem (Proposition \ref{proposition:minkowski-weyl}), there exists some $s\in\ints$ and some matrix $Z\in\reals^{k\times s}$ so that $\calC = \{\vec u\in\reals^k: \vec u = Z\vec x, \vec x\ge \vec 0\}$. Let $(\vec h_i)_{i\in [t]}$ be the columns of $H$ and $(\vec z_i)_{i\in [s]}$ the columns of $Z$. Suppose without loss of generality that $\|\vec z_j\|=1$ for any $j\in[s]$.
    
    \paragraph{Step 1.} In this step, we will find some bound for $\theta$ which implies the existence of a vector $\vec w$ which has some useful properties for the next step. We will use $\vec w$ as a point of reference for moving $\vec u$ to get $\vec u'$. We consider the quantity
    \[
        \theta_1 := \min_{i,j}\{\vec h_i\cdot \vec z_j \,|\, i\in[t],j\in[s]:\vec h_i\cdot \vec z_j > 0\}\,.
    \]
    Note that since $s,t<\infty$, the quantity $\theta_1$ (when it is well defined) is positive (and only depends on the columns of $H$ and the geometry of $\calC$).
    
    In case $\vec h_i\cdot \vec z_j = 0$ for any $i,j$, we have that for any $\vec u\in\calC$, $\vec h_i\cdot \vec u = 0$ for any $i\in[t]$ (since $\calC$ is generated by the columns of $Z$) and therefore we may consider $\vec u' = \vec u/\|\vec u\|$, which satisfies all the required properties.
    
    Next, we observe that there must exist some vector $\vec x\in\reals^s$ with $\vec x\ge \vec 1$ (component-wise) and $Z\vec x= \sum_{j\in[s]}x_j \vec z_j \neq 0$. Otherwise, $Z$ would have zero rank (its null space would be of full dimension) and $\calC=\{\vec 0\}$ (our statement would then be trivially satisfied). Hence, we have $Z\vec x \neq 0$. Consider the value $N = 2\cdot\max_{\calJ\subseteq[s]}\left\|\sum_{j\in\calJ}x_j\vec z_j\right\| (>0)$ and the vector $\vec w$ with 
    \[
        \vec w = Z\vec x / N\,.
    \]
    We have that for any $i\in[t]$ and $j\in[s]$ with $i,j$ such that $\vec h_i\cdot \vec z_j >0$ it holds
    \[
        \frac{x_j}{N}\cdot \vec h_i \cdot \vec z_j \ge \frac{x_j}{N}\cdot \theta_1\,,
    \]
    and therefore, for any $i\in[t]$ such that there exists $j\in[s]$ with $\vec h_i\cdot \vec z_j>0$, $\vec h_i\cdot \vec w\ge  \theta_1/N$ (since $x_j\ge 1$). Consider the following quantity
    \[
        \theta_2 := \theta_1/N\,.
    \]
    Note that again, $\theta_2>0$ and only depends on the geometry of $\calC$. We demand that $\theta\le \theta_2$\,. Then, $\vec w\in\calC$, $\|\vec w\|\le 1/2$ (and the norm bound also holds for any part of $\vec w$ of the form $\sum_{j\in\calJ}w_j \vec z_j$, for $\calJ\subseteq[s]$) and $\vec h_i\cdot \vec w\ge \theta$ for any interesting $\vec h_i$ (i.e., for any $\vec h_i$ that is not orthogonal to every point in $\calC$). 
    
    \def\calJ{\mathcal{J}}
    
    \paragraph{Step 2.} Consider, now, any $\vec u\in\calC$. We will find $\vec u'\in\calC$ with the desired properties. We have that
    \[
        \vec u = \sum_{j\in[s]} u_j \vec z_j\,, \text{ where }u_j\ge 0\,.
    \]
    Consider the set $\calI\subseteq[t]$ as follows
    \begin{align*}
        \calI\subseteq[t]: &\ \vec h_i\cdot \vec u < \theta\, \text{ for }i\in\calI \, \& \\
            &\ \vec h_i\cdot \vec u \ge \theta\, \text{ for }i\not\in\calI\,.
    \end{align*}
    The set $\calI$ includes the columns on $H$ that correspond to halfspaces with boundaries to which $\vec u$ is close. Moreover, we define for any $\calI\subseteq[t]$ the set $\calJ_\calI$ as follows
    \[
        \calJ_\calI = \{j\in[s] : \vec h_i\cdot \vec z_j = 0 \text{ for all }i\in\calI\}\,.
    \]
    The set $\calJ_\calI$ corresponds to the generating vectors that lie within the nullspace of the set of columns of $\calH$ corresponding to $\calI$. The vectors corresponding to $\calJ_\calI$ control the part of $\vec u$ that is parallel to the boundaries to which $\vec u$ is close.
    
    We use the notation $\vec u_\calI$ (resp. $\vec w_\calI$) to refer to the vector $\vec u_\calI = \sum_{j\in\calJ_\calI} u_j\vec z_j$, i.e., the part of $\vec u$ that corresponds to nearby boundaries. We let
    \[
        \vec u' = \vec u - c(\vec u_\calI - \vec w_\calI)\,, 
    \]
    for some $c\in[0,1]$ to be disclosed. We show that $\vec u'$ (for appropriate $c$ and possibly some additional bounds on $\theta$) has all the desired properties.
    \begin{enumerate}
        \item Consider any $i\in\calI$. We have that $\vec h_i\cdot \vec u_\calI = 0 = \vec h_i\cdot \vec w_\calI$, due to the definition of $\calJ_\calI$. Therefore $\vec h_i\cdot \vec u = \vec h_i\cdot \vec u'$.
        
        \item For $i\in[t]\setminus\calI$, we have that:
        $
            \vec h_i\cdot \vec u' = \vec h_i\cdot \vec u - c \vec h_i\cdot \vec u_\calI + c\vec h_i\cdot \vec w_\calI \,,
        $
        where if $\vec h_i\cdot \vec z_j = 0$ for every $j\in\calJ_\calI$, we have $\vec h_i\cdot \vec u' = \vec h_i\cdot \vec u \ge \theta$, since $i\not\in\calI$. Otherwise, $\vec h_i\cdot \vec w_\calI \ge \theta$ (due to \textbf{Step 1}) and we get
        \[
            \vec h_i\cdot \vec u'\ge (1-c) \vec h_i\cdot \vec u + c \vec h_i\cdot \vec w_\calI \ge (1-c)\theta + c\theta = \theta\,.
        \]
        
        \item We have that $\vec u'= \sum_{j\in[s]}u_j'\vec z_j $, where $u_j'\ge 0$ since $c\le 1$. Therefore $\vec u'\in\calC$. It remains to show that for some selection of $c\in[0,1]$, we achieve $\|\vec u'\| = 1$. Consider, temporarily the case $c=1$. If we show that for this value of $c$, $\vec u'$ is within the unitary ball, then, since $\vec u$ is outside (and corresponds to $c=0$), there must exist some $c\in[0,1]$ (i.e., a point on the line segment connecting $\vec u$ and $\vec u - \vec u_\calI + \vec w_\calI$) for which $\|\vec u'\|=1$.
        
        For $c=1$ we have that
        \[
            \|\vec u'\| = \|\vec u-\vec u_\calI + \vec w_\calI\| \le \|\vec u-\vec u_\calI\| + \|\vec w_\calI\|\,.
        \]
        Observe, now that, $\|\vec w_\calI\|\le 1/2$ (due to the definition of $N$ in \textbf{Step 1}) and that as $\theta$ decreases, $\|\vec u-\vec u_\calI\|$ tends to become zero for fixed $\calI$, independently from the selection of $\vec u$ given that $\vec u$ corresponds to $\calI$ (we may minimize over finitely many possible $\calI$). Therefore, for small enough $\theta>0$, we have that $\|\vec u'\|\le 1$ (for $c=1$, which implies that $\|\vec u'\|=1$ for some appropriate selection of $c\in[0,1]$). 
        
        More specifically, we have that 
        \[
            \|\vec u-\vec u_\calI\|^2 = \sum_{j,j'\not\in\calJ_\calI}u_j u_{j'} (\vec z_j\cdot \vec z_{j'}) \le \left( \sum_{j\not\in\calJ_\calI} u_j \right)^2\,.
        \]
        For each $j\not\in\calJ_\calI$, there must exist some $i_j\in\calI$ so that $\vec h_{i_j}\cdot \vec z_j>0$ (otherwise, $j\in\calJ_\calI$). Recall that due to the definition of $\theta_1$, we have $\vec h_{i_j}\cdot \vec z_j\ge \theta_1$. Moreover, we have
        \[
            \vec h_{i_j}\cdot \vec u < \theta\,,
        \]
        since $i_j\in\calI$. Hence, since $\vec u=\sum_{j'\in[s]}u_{j'} \vec z_{j'}$ and $\vec h_{i_j}\cdot\vec z_{j'}\ge 0$ by the fact that $\vec z_{j'}\in\calC$, we get that
        \[
            u_j (\vec h_{i_j}\cdot \vec z_j) < \theta\,, \text{ or }u_j < \theta/\theta_1\,, \text{ for any } j\not\in\calJ_\calI\,.
        \]
        Therefore $(\sum_{j\not\in\calJ_\calI} u_j) \le \theta \cdot s/ \theta_1$\,. We demand that 
        \[
            \theta \le \frac{\theta_1}{2\cdot s}\, \left(\text{and } \theta\le \theta_2 \text{ from \textbf{Step 1}}\right)\,,
        \]
        which concludes our proof.\qedhere
    \end{enumerate}
\end{proof}

\paragraph{Structural Distance \& TV Distance.} The last step towards Theorem \ref{theorem:projection} is to show that Lemma \ref{lemma:descriptivity} can be transformed in terms of TV distance (see Appendix \ref{proof:projection-step}). We stress that structural distance is a local metric, since it is defined in the terms of a property that every point of the domain satisfies independently, while TV distance is a global metric since it expresses the total difference between two distributions over the whole domain. Therefore, the main technical complication here is that the support is infinite and it is not clear whether structural distance implies bounds for the TV distance. The complication is resolved by the following lemma, whose general form (Lemma \ref{lemma:vanishing-deviation}) is in fact also useful in other parts of the proof of Theorem \ref{theorem:covering-siiervs} and states that when the parameter vector's norm is large enough, then almost all the mass lies within an bounded length interval around the mode. Its proof is based on the bounded central moments assumption \eqref{assumption:bounded-moments}. 
\begin{lemma}
[Corollary of Lemma \ref{lemma:vanishing-deviation}]
\label{lemma:vanishing-expectation-informal}
    Under assumptions \eqref{assumption:unimodal} and \eqref{assumption:bounded-moments}, there exists some natural number $\ell=O(\sqrt{B})$ such that for any $\vec a\in\varrho\-\conehull\calA$ and any mode $M_{\vec a}$ of $\calE_{\vec T}(\vec a)$ we have
    $
        \Pr_{\vec a} \left[ |W- M_{\vec a}|> \ell \right] 
        \le \exp(-\|\vec a\|/\varrho) \cdot O(1)\,.
    $
\end{lemma}

In fact, Lemma \ref{lemma:vanishing-expectation-informal} is a Corollary of a more general one which we now state and prove.

\begin{lemma}\label{lemma:vanishing-deviation}
    Under assumptions \eqref{assumption:unimodal} and  \eqref{assumption:bounded-moments}, for any $\kappa > 0$, any $\eta>0$ and any $s\in\{0,1,2\}$ there exists some $\ell=e^{\kappa/(3-\eta-s)}\cdot O(B^{\frac{5}{4\cdot(3-\eta-s)}})$ such that for any $\vec a\in\varrho\-\conehull\calA$ and any mode $M_{\vec a}$ of the corresponding distribution we have
    \begin{enumerate}
        \item $\Pr_{\vec a}[W=x] \le \frac{e^{-\kappa\cdot \max \left\{1,\frac{\|\vec a\|}{\varrho} \right\} }}{|x-M_{\vec a}|^{1 + \eta + s}} \cdot\Pr_{\vec a}[W=M_{\vec a}]$, for any $x\in\ints$ with $|x-M_{\vec a}|\ge \ell$.
        \item If $Q_\ell = \vec 1\{|W-M_{\vec a}|\le \ell\}$ then
    \[
        \E_{\vec a}[|W-M_{\vec a}|^s] \le \E_{\vec a}[|W-M_{\vec a}|^s\cdot Q_\ell\bigr] +  e^{-\kappa\cdot \max \left\{1,\frac{\|\vec a\|}{\varrho} \right\} } \cdot O(1/\eta) \,.
    \]
    \end{enumerate}
    In particular, for $s=0$, we use the convention $\E[W^0] = \Pr[W\neq 0]$, for any random variable $W$.
\end{lemma}

\begin{proof}
    From assumption \eqref{assumption:bounded-moments}, we get the following inequality for any $\vec a\in \varrho\-\conehull\calA$
    \[
        \Var_{\vec a}(W)\le O(\sqrt{B})\,.
    \]
    We also know that $\E_{\vec a}[|W-M_{\vec a}|^2]\le 4\Var_{\vec a}(W)$, due to unimodality of the random variable $W$ (which implies that $|\E_{\vec a}[W]-M_{\vec a}|\le \sqrt{3\Var_{\vec a}(W)}$ as shown by \cite{johnson1951themomentproblem}). Therefore $\E_{\vec a}[|W-M_{\vec a}|^2]\le O(\sqrt{B})$ and similarly $\E_{\vec a}[|W-M_{\vec a}|^4]\le O(B)$. 
    
    \paragraph{Proof of Part (1).} For any parameter vector $\vec a\in\varrho\-\conehull\calA$ (and some fixed corresponding mode $M_{\vec a}$), we have that $\E_{\vec a}[|W-M_{\vec a}|^4] = O(B)$ (since $\E_{\vec a}[|W-\E_{\vec a}[W]|^4] = O(B)$ and $\Var_{\vec a}(W) = O(\sqrt{B})$). 
    
    Let $\vec b\in\varrho\-\conehull\calA$. Suppose that for some $x\in\ints$ with $x\neq M_{\vec b}$, we have that
    \[
        \vec b\cdot (\vec T(x) - \vec T(M_{\vec b})) < \left(1+\eta+s\right)\cdot\ln(|x-M_{\vec b}|) + \kappa\,.
    \]
    Then, we have that 
    \[
        \E_{\vec b}[|W-M_{\vec b}|^4]\cdot e^{\vec b\cdot \vec T(M_{\vec b})}\cdot\Z_{\vec T}(\vec b) > |x-M_{\vec b}|^4\cdot e^{-(1+\eta+s) \ln(|x-M_{\vec b}|) - \kappa} = e^{-\kappa} \cdot |x-M_{\vec b}|^{3-\eta-s}\,.
    \]
    Since, additionally, $\E_{\vec b}[|W-M_{\vec b}|^4] = O(B)$, and using Lemma \ref{lemma:partition-function-upper-bound} (which provides an upper bound for the partition function), it must be that $|x-M_{\vec b}| < \ell$, for some $\ell$ with $\ell\le e^{\frac{\kappa}{3-\eta-s}}\cdot O \left (B^{\frac{5}{4\cdot(3-\eta-s)}} \right)$.
    
    For $x\in\ints$ with $|x-M_{\vec b}|\ge \ell$ we therefore get that
    \begin{equation}\label{equation:vanishing-deviation:inner-product-bound}
        \vec b\cdot \vec T(x) \ge \ \vec b\cdot \vec T(M_{\vec b}) + \left(1+\eta+s\right)\ln(|x-M_{\vec b}|) + \kappa\,.
    \end{equation}
    
    Consider now any $\vec a\in\varrho\-\conehull\calA$ with $\|\vec a\|\ge \varrho$. Then, there exists some $\vec b\in\varrho\-\conehull\calA$ with $\|\vec b\|=\varrho$ so that $\vec b = \varrho \vec a / \|\vec a\|$. We multiply both sides of Equation \eqref{equation:vanishing-deviation:inner-product-bound} with $\|\vec a\|/\varrho \geq 1$ and use the fact that rescaling the parameter vector does not change the set of modes to get
    \begin{align*}
        \vec a\cdot \vec T(x) &\ge \vec a\cdot \vec T(M_{\vec a}) + \left(1+\eta+s\right)\ln(|x-M_{\vec b}|)  + \kappa\|\vec a\|/\varrho, \;\;\text{ or } \\
        \Pr_{\vec a}[W=x] &\le \exp(-\kappa\|\vec a\|/\varrho)\cdot \frac{1}{|x-M_{\vec b}|^{\left(1+\eta+s\right)}} \cdot\Pr_{\vec a}[W=M_{\vec a}]\,,
    \end{align*}
    for any $x\in\ints$ with $|x-M_{\vec a}|\ge \ell$. Note that in the above we did not multiply $\left(1+\eta+s\right)\ln(|x-M_{\vec b}|)$ with $\| \vec a \|/\varrho \geq 1$ since it only helps the inequality. Since $M_{\vec a} = M_{\vec b}$ due to the choice of $\vec b$, we accomplish our goal for the case $\|\vec a\| \geq \varrho$.
    
    On the other side, if $\|\vec a\|<\varrho$, then $\kappa\ge \kappa\cdot\|\vec a\|/\varrho$ and we may pick $\vec b = \vec a$, concluding the proof of the first part of the Lemma.

    \paragraph{Proof of Part (2).} Let us set $s \in \{0,1,2\}$. We have that
    \begin{align*}
        \E_{\vec a}[|W-M_{\vec a}|^s] = &
        \frac{
                \sum_{x\in\ints} |x-M_{\vec a}|^s\cdot \exp( -\vec a\cdot (\vec T(x)-\vec T(M_{\vec a}) ) )
            }{
                \exp( \vec a\cdot \vec T(M_{\vec a}) ) \cdot \Z_{\vec T}(\vec a)
            } \\
        = & \E_{\vec a}[|W-M_{\vec a}|^s \cdot Q_{\ell}] + \sum_{x:|x-M_{\vec a}|>\ell} \frac{
                |x-M_{\vec a}|^s\cdot \exp( -\vec a\cdot (\vec T(x)-\vec T(M_{\vec a}) ) )
            }{
                \exp( \vec a\cdot \vec T(M_{\vec a}) ) \cdot \Z_{\vec T}(\vec a)
            } \\
        \le & \E_{\vec a}[|W-M_{\vec a}|^s \cdot Q_{\ell}] + \sum_{x:|x-M_{\vec a}|>\ell} |x-M_{\vec a}|^s\cdot \exp( -\vec a\cdot (\vec T(x)-\vec T(M_{\vec a}) ) ) \\
        \le & \E_{\vec a}[|W-M_{\vec a}|^2 \cdot Q_{\ell}] + e^{-\kappa \cdot \max\{ \|\vec a\|/\varrho,1 \}} \cdot 2\zeta(1+\eta) \\ 
        \le & \E_{\vec a}[|W-M_{\vec a}|^2 \cdot Q_{\ell}] + e^{-\kappa \cdot \max\{\|\vec a\|/\varrho,1\}}\cdot O({1}/{\eta})\,,
    \end{align*}
where the first inequality follows from the fact that $\exp( \vec a\cdot \vec T(M_{\vec a}) ) \cdot \Z_{\vec T}(\vec a) = \sum_{x\in\ints} \exp(-\vec a\cdot (\vec T(x) - \vec T(M_{\vec a})) ) \ge \exp(-\vec a\cdot (\vec T(M_{\vec a}) - \vec T(M_{\vec a})) ) = 1$. The second inequality follows by applying \textbf{Part (1)} and noting that $\sum_{x : |x - M_{\vec a}| > \l} \frac{1}{|x-M_{\vec a}|^{1+\eta}} \leq 2\zeta(1+\eta),$ where $\zeta(\cdot)$ denotes the Riemann zeta function and $\zeta(1+\eta) = \Theta(1/\eta)$ as $\eta\to 0$.
\end{proof}

\def\ncrit{n'_{\textnormal{crit}}}
\subsection{Sparsifying SIIERVs}
\label{section:technical-2}

We now consider distributions of the form $X=\sum_{i\in[n']}X_i$, where $n'\le n$, $(X_i)_i$ independent and $X_i\sim\calE_{\vec T}(\vec a_i)$ for $\vec a_i\in\calA$. When $n'$ is small, then the distribution can be approximated term by term, by setting $\eps$ of Theorem \ref{theorem:single-term} equal to $\eps/n'$, since the total error of approximation is at most equal to the sum of the errors for each term. When $n'$ is large, the distribution of $X$ resembles the distribution of a discretized Gaussian random variable, due to the Berry-Esseen type bound of Lemma \ref{lemma:discr-gaussian-approx}. Assumption \eqref{assumption:variance-lower-bound} ensures that the variance of $X$ will be large enough, but also that the shift distance (i.e., $\tv(X,X+1)$) will be small enough. In particular, for unimodal distributions, the shift distance of a single term equals the mass assigned to the mode. Using Lemma \ref{lemma:vanishing-deviation}, we show that the variance lower bound implies an upper bound for the mass on the mode.

We only need to account for the case of large $n'$ and, in fact, find some critical value $\ncrit \in\nats$ with respect to which sparse and dense cases are split. In the dense case, the distribution of $X$ can be represented by its mean and variance alone. The goal is, therefore, to find some SIIERV $Y$ that is also close to a discretized Gaussian and $\E[X]\approx\E[Y]$ and $\Var(X)\approx\Var(Y)$. We prove in Lemma \ref{lemma:continuity} that (under assumptions \eqref{assumption:unimodal} and \eqref{assumption:bounded-moments}) $\Var_{\vec a}(W)$ and $\E_{\vec a}[W]$ are continuous functions of $\vec a$ on $\calA$ (and by assumption \eqref{assumption:connectivity}, $\calA$ is path connected). Therefore, one might hope that, although $\E[X]$ is a sum of many different quantities of the form $\E_{\vec a_i}[W]$ with $\vec a_i\in\calA$ (and similarly $\Var(X)$, due to independence), continuity implies the existence of a single parameter vector $\vec b$ in $\calA$ which expresses the total behavior of the quantities $\Var(X)$ and $\E[X]$. It turns out that the correct way to define $\vec b$ is as the parameter vector for which $\Var(X)/\E[X] = \Var_{\vec b}(W)/\E_{\vec b}[W]$. The reason ratios are used is to eliminate the influence of the number of terms $n'$ (which the parameter vector has no influence on). The idea is to use some kind of intermediate value theorem to prove the existence of $\vec b$, motivated by the fact that in the case $\E[X_i]>0$ for any $i\in[n']$, then $\Var(X)/\E[X] \in [\min_{i\in[n']}\Var(X_i)/\E[X_i] , \max_{i\in[n']}\Var(X_i)/\E[X_i]]$. However, there needs to be a careful handling for the cases that $\E[X_i]= 0$ or $\E[X_i]<0$ for some $i\in[n']$. The random variable $Y$ is selected to be a SIIERV of order $m = \lceil \E[X]/\E_{\vec b}[W]\rceil$ consisting of i.i.d. terms $(Y_i)_{i\in[m]}$ each following the distribution $\calE_{\vec T}(\vec b)$.

Finally, we observe that $\vec b\in\calA$ and Theorem \ref{theorem:single-term} can be used once again with accuracy $\eps/m$ to discretize the parameter space and provide a sparse cover for the dense case. We proceed with the formal proof of Theorem \ref{theorem:covering-siiervs}.

\begin{proof}[Proof of Theorem \ref{theorem:covering-siiervs}]
    We consider some exponential family $\calE_{\vec T}(\calA)$, the class $\calE_{\vec T}(\calA)$-\siirvs\ of order $n$ and any random variable $X$ with distribution within this class. That is
    \[
        X = \sum_{i\in[n']} X_i\,, \text{ where }n'\le n, X_i\sim\calE_{\vec T}(\vec a_i), \vec a_i\in\calA \text{ and }(X_i)_i\text{ independent.}
    \]
    We will show that, under our assumptions (see Assumption \ref{assumption:proper}) we can approximate the distribution of any such $X$ by some distribution lying within a small subset of $\calE_{\vec T}(\calA')$-\siirvs\ of order $m$, where $m\le \varrho\cdot n$ and $\calA'=\setoperator\calA$, for some $\varrho\ge 1$ and some extensive set operator $\setoperator$.
    
    \paragraph{Sparse Case.} The proof we will give has two main ingredients. The first one is Theorem \ref{theorem:single-term}, which states that, under our assumptions, we may sparsify the parameter space into a small set $\calB\subseteq\varrho\-\conehull\calA$. The first ingredient directly implies that if $n'$ is small, then the distribution of $X$ must be close to the distribution of some random variable $Y = \sum_{i\in[n']}Y_i$, where $Y_i\sim\calE_{\vec T}(\vec b_i)$, $\vec b_i\in\calB$ and  $(Y_i)_i$ are independent. In particular, considering some $\ncrit\in\nats$ which will be specified later, whenever $n'\le \ncrit$, Theorem \ref{theorem:single-term}, applied for $\eps\gets \eps/\ncrit$ gives some set $\calB\subseteq\varrho\-\conehull\calA$ with $|\calB|^{1/k} \le \widetilde{O}(\frac{\ncrit\cdot\varrho\cdot\sqrt{\Lambda}}{\eps\cdot \descriptivityparam} \cdot \log(B))$ so that each $\calE_{\vec T}(\calA)$-SIIRV of order $\ncrit$ can be $\eps$-approximated by some $\calE_{\vec T}(\calB)$-SIIRV of order $\ncrit$. Observe that there are only 
    \[
        |\calB|^{\ncrit} \le \left(\widetilde{O}\left(\frac{\ncrit\cdot\varrho\cdot\sqrt{\Lambda}}{\eps\cdot \descriptivityparam} \cdot \log(B) \right) \right)^{k\cdot \ncrit}
    \]
    different $\calE_{\vec T}(\calB)$-\siirvs\ of order $\ncrit$. The value $\ncrit$ will be obtained by the upcoming dense case.
    
    \paragraph{Dense Case.} The second ingredient is Lemma \ref{lemma:discr-gaussian-approx}, which indicates that when we have many terms in the sum, then the distribution of $X$ can be accurately represented by its mean and variance alone, since $X$ will be close to a Discretized Gaussian distribution. Therefore, when $n'\ge \ncrit$ (where $\ncrit$ sufficiently large for our purposes), we might try to find some SIIERV $Y$ that also has sufficiently many terms (and hence is accurately represented by its mean and variance alone) whose expectation and variance are close to the expectation and variance of $X$ respectively.
    
    In particular, our proof consists of the following parts.
    \begin{enumerate}
        \item We prove that a random variable $\dgaussvar_{X}$ that follows the discretized Gaussian distribution with mean $\E[X]$ and variance $\Var(X)$ is $O(\eps)$-close in total variation distance to $X$, given that $n'(\ge\ncrit)$ is large enough.
        
        \item We then find some random variables $Y'$ where $Y'=\sum_{i\in[m]}Y_i'$ and $Y_1', \ldots, Y_m'$ are i.i.d., each following the distribution $\calE_{\vec T}(\vec b')$ for some $\vec b'\in\calA$ so that $\E[Y']\approx \E[X]$ and $\Var(Y')\approx \Var(X)$.
        
        \item Next, we show that $Y'$ is $O(\eps)$-close to a discretized Gaussian random variable $\dgaussvar_{Y'}$ with  mean $\E[Y']$ and variance $\Var(Y')$, when $n'$ is large enough.
        
        \item Afterwards, we show that $\dgaussvar_{X}$ and $\dgaussvar_{Y'}$ are $O(\eps)$-close in total variation distance.
        
        \item Finally, we use Theorem \ref{theorem:single-term} for the distribution $\calE_{\vec T}(\vec b')$ in order to discretize the parameter vector.
    \end{enumerate}
    
    In fact, our goal here is to acquire a set of inequalities of the form $n'\ge n_i$, for $n_i>0$ so that if $n'$ satisfies all of them, then each one of the steps presented above corresponds to some $O(\eps)$ deviation in total variation distance.
    
    \paragraph{Step 1: Gaussian Approximation.} Let $\mu=\E[X]=\sum_{i\in[n']}\mu_i,$ where $\mu_i=\E[X_i]$ and $\sigma^2=\Var(X)=\sum_{i\in[n']}\sigma_i^2$, where $\sigma_i^2=\Var(X_i)$ and consider some random variable $\dgaussvar_{X}$ with $\dgaussvar_{X}\sim\Dgauss(\mu,\sigma^2)$. Let $\beta = \sum_{i\in[n']}\beta_i$ where $\beta_i = \sum_{i\in[n']}\E[|X_i-\E[X_i]|^3]$ and let $\delta\in[0,1]$ such that
    \[
        \delta= \max_{i\in[n']}\tv(X-X_i,X-X_i+1)\,.
    \]
    From Lemma \ref{lemma:discr-gaussian-approx}, we get that
    \[
        \tv(X,\dgaussvar_{X}) \le O(1/\sigma) + O(\delta) + O(\beta/\sigma^3) + O(\delta\cdot\beta / \sigma^2)\,.
    \]
    First, we upper bound the quantity $\beta/\sigma^2$. 
        Using Lemma \ref{lemma:ratio-ineq}, we have that
        \[
            \frac{\beta}{\sigma^2}\le \max_{i\in[n']}\frac{\beta_i}{\sigma_i^2} \le B/\gamma\,,
        \]
        due to assumptions \eqref{assumption:bounded-moments} and \eqref{assumption:variance-lower-bound}. Next, we provide a lower bound for $\sigma^2$. In particular, we get that
        \[
            \sigma^2 \ge n'\cdot \gamma \,,
        \]
        due to assumption \eqref{assumption:variance-lower-bound}. We now demand that $n'$ is large enough so that
        \[
            \left( 1+\frac{\beta}{\sigma^2} \right) \cdot \frac{1}{\sigma} \le O(\eps)\,,
        \]
        thereby concluding to the following demand for the number of summands $n'$:
        \[
            n' \ge n_1\,, \text{ where }n_1 = O\left( \frac{B^2}{\eps^2 \cdot \gamma^3} \right)\,.
        \]
        Finally, we calculate $\delta$ and provide another demand of the form $n'\ge n_i$ for some $n_i>0$. For any $\vec a\in\calA$ and any $W\sim\calE_{\vec T}(\vec a)$, we have that
        \[
            2 \tv(W, W+1) = \frac{1}{\Z_{\vec T}(\vec a)}\sum_{x \in \ints} \left| \exp(-\vec a \cdot \vec T(x)) - \exp(-\vec a \cdot \vec T(x+1)) \right|\,.
        \]
        By using the unimodatily assumption \eqref{assumption:unimodal}, we get that the summation in the right hand side of the above equation is telescopic on both sides around $M_{\vec a}$ and therefore
        \[
            \tv(W,W+1) = \Pr[W=M_{\vec a}]\,.
        \]
        We now bound $\Pr[W=M_{\vec a}]$ for any $\vec a\in\calA$ by using Lemma \ref{lemma:vanishing-deviation} with $\eta=1/2$, $s=2$ and $\kappa>0$ to be decided. Note that by picking smaller $\eta$ we can shrink the order of $\ell$ but cannot make it smaller than $O(B^{5/4})$. We get some $\ell\le e^{2\kappa}\cdot O(B^{2.5})$ so that 
        \[
            \E[|W-M_{\vec a}|^2]\le \ell^2 \cdot \Pr[W\neq M_{\vec a}] + e^{-\kappa}\cdot O(1)\,.
        \]
        We pick $\kappa = \ln(O(1/\gamma))$ to get
        \[
            \E[|W-M_{\vec a}|^2]\le \ell^2 \cdot \Pr[W\neq M_{\vec a}] + \gamma/2\,.
        \]
        We also know that $\E[|W-M_{\vec a}|^2] \ge \Var_{\vec a}(W) \ge \gamma$. Hence, we have
        \[
            1 - \Pr[W = M_{\vec a}] \ge \Omega(B^5/\gamma^5)\,.
        \]
        Moreover, by using Lemma \ref{lemma:dshift-sum}, we get that
        \[
            \tv(X,X-X_i+1) \le \frac{\sqrt{2/\pi}}{\sqrt{\frac{1}{4}+(n'-1)\inf_{\vec a\in\calA}(1-\tv(W_{\vec a},1+W_{\vec a}))}}\,,
        \]
        where $W_{\vec a}\sim\calE_{\vec T}(\vec a)$.
        
        We conlcude to the following demand for the number of summands $n'$ (so that we have $(1+\beta/\sigma^2)\cdot\delta\le O(\eps)$):
        \[
            n' \ge n_2\,, \text{ where }n_2 = O\left( \frac{B^7}{\eps^2 \cdot \gamma^7} \right)\,.
        \]
        
    \paragraph{Step 2: Matching Variances and Expectations.} In this step, we will find some random variable $Y'$ that is the sum of a number of i.i.d. random variables within $\calE_{\vec T}(\calA)$ such that the expectation (resp. variance) of $Y'$ is close to the expectation (resp. variance) of $X$. We will split cases according to the sign of the expectation $\E[X]$.
    \begin{itemize}
        \item If $\E[X] = 0$, then we have $\sum_{i\in[n']}\E[X_i] = 0$, which implies that either $\E[X_j] = 0$ for some $j\in[n']$ (in which case we may consider $\vec b' = \vec a_j$), or that $\E[X_i]\cdot \E[X_j] < 0$, for some $i,j\in[n']$, which, since $\E_{\vec a}[W]$ is a continuous function when $\vec a\in\calA$ (see Lemma \ref{lemma:continuity}), gives by intermediate value theorem (and the fact that $\calA$ is connected by assumption \eqref{assumption:connectivity}) some $\vec b'\in\calA$ with $\E_{\vec b'}[W] = 0$.
        
        We now pick
        \[
            m = \left\lceil \frac{\Var(X)}{\Var_{\vec b'}(W)} \right\rceil\,.
        \]
        We have that $\Var(Y') = m\cdot \Var_{\vec b'}(W) \in [\Var(X), \Var(X)+\Var_{\vec b'}(W)]$ and hence we get that $\Var(Y') \in [\Var(X), \Var(X)+\sqrt{B}]$, due to assumption \eqref{assumption:bounded-moments} and the fact that $\vec b'\in\calA$. Moreover, we have that $\E[X] = \E[Y'] = 0$.
       
        \item If $\E[X] > 0$, then we split $X = \sum_{i\in[n']}X_i$ into three summations $X=X^++X^-+X^0$, according to the sign of $\E[X_i]$ (for example $X^+ = \sum_{i\in I^+} X_i$, where $I^+$ is the set of $i\in[n']$ so that $\E[X_i]>0\}$). We then have that $\E[X^+] > |\E[X^-]|$ (since $\E[X]>0$) and
        \[
            \frac{\Var(X)}{\E[X]} = \frac{\Var(X^+)+\Var(X^-)+\Var(X^0)}{\E[X^+] - |\E[X^-]|} \ge \frac{\Var(X^+)}{\E[X^+]}\,.
        \]
        Moreover, we have that 
        \[
            \frac{\Var(X^+)}{\E[X^+]} = \frac{\sum_{i\in I^+} \Var(X_i)}{\sum_{i\in I^+} \E[X_i]} \ge \min_{i\in[n']}\frac{\Var(X_i)}{\E[X_i]}\,,
        \]
        since $\Var(X_i),\E[X_i]>0$, for any $i\in I^+$. Recall that the distribution of $X_i$ is $\calE_{\vec T}(\vec a_i)$ for some $\vec a_i\in\calA$. 
        
        Suppose, first that there exists some $j\in[n']$ so that $\E[X_j] \le 0$ (i.e., $I^0\cup I^-\neq\emptyset$). Then, there exists some $\vec a_j\in\calA$ such that $\E_{\vec a_j}[W] \le 0$. Since $\calA$ is connected, there exists some path connecting $\vec a_i$ and $\vec a_j$. Let $\vec a'$ be the first point in the path between $\vec a_i$ and $\vec a_j$ (beginning from $\vec a_i$) so that $\E_{\vec a'}[W] = 0$. We know that there exists such a point and that when $\vec a$ goes from $\vec a_i$ to $\vec a'$ through the path we described, $\E_{\vec a}[W]$ always remains positive (since it is a continuous function by Lemma \ref{lemma:continuity}). Moreover, as $\vec a$ approaches $\vec a'$, the expectation $\E_{\vec a}[W]$ becomes arbitrarily small, while the variance $\Var_{\vec a}(W)$ remains lower bounded by $\gamma$ (due to assumption \eqref{assumption:variance-lower-bound}). Let $P$ denote the path from $\vec a_i$ to $\vec a'$, excluding $\vec a'$. Then the quantity $\Var_{\vec a}(W)/\E_{\vec a}[W]$ is a continuous function of $\vec a$ when $\vec a\in P$, due to Lemma \ref{lemma:continuity} and the fact that $\E_{\vec a}[W]>0$ for any $\vec a\in P$. Also, we have that as $\vec a\to\vec a'$ (through the path $P$), $\Var_{\vec a}(W)/\E_{\vec a}[W]\to \infty$ and therefore, due to the intermediate value theorem, there exists some $\vec b'\in P\subseteq\calA$ so that 
        \[
            \frac{\Var_{\vec b'}(W)}{\E_{\vec b'}[W]} = \frac{\Var(X)}{\E[X]} \in \left[\frac{\Var_{\vec a_i}(W)}{\E_{\vec a_i}[W]} , \infty\right) \,.
        \]
        
        When $\E[X_i]>0$ for any $i\in[n']$, we have that
        \[
            \frac{\Var(X)}{\E[X]} = \frac{\Var(X^+)}{\E[X^+]} \in \left[ \min_{i\in[n']}\frac{\Var(X_i)}{\E[X_i]}, \max_{i\in[n']}\frac{\Var(X_i)}{\E[X_i]} \right]\,,
        \]
        since all terms are positive. By a similar continuity argument we get once again that there exists some $\vec b'\in\calA$ so that 
        \[
            \frac{\Var_{\vec b'}(W)}{\E_{\vec b'}[W]} = \frac{\Var(X)}{\E[X]}\,.
        \]
        We may pick
        \[
            m = \left\lceil \frac{\Var(X)}{\Var_{\vec b'}(W)} \right\rceil\,.
        \]
        We have that $\Var(Y') = m\cdot \Var_{\vec b'}(W) \in [\Var(X), \Var(X)+\Var_{\vec b'}(W)]$ and hence we get that $\Var(Y') \in [\Var(X), \Var(X)+\sqrt{B}]$, due to assumption \eqref{assumption:bounded-moments} and the fact that $\vec b'\in\calA$. Moreover, due to the selection of $\vec b'$ we have that
        \[
            m = \left\lceil \frac{\E(X)}{\E_{\vec b'}[W]} \right\rceil\,.
        \]
        Hence, we get the following bound for the expectation of $Y'$ with respect to the expectation of $X$
        \begin{align*}
            \E[Y'] = m\cdot \E_{\vec b'}[W] \in \left[\E[X], \E[X] + \E_{\vec b'}[W]\right]
        \end{align*}
        
        \item If $\E[X]<0$, then we may use an analogous reasoning as for the case that $\E[X]>0$ to prove the existence of some $\vec b'\in\calA$ so that $\E[Y']\in[\E[X], \E[X]+\E_{\vec b'}[W]]$ and also $\Var(Y')\in [\Var(X), \Var(X)+\sqrt{B}]$.
    \end{itemize}
    
    Therefore, in any case, we have proven that for some $\vec b'\in\calA$, the random variable $Y'=\sum_{i\in[m]}Y'_i$ where $Y'_i$ are i.i.d. random variables following the distribution $\calE_{\vec T}(\vec b')$ has
    \[
        \E[X] \le \E[Y'] \le \E[X] + \E_{\vec b'}[W] \text{ and } \Var(X)\le \Var(Y') \le \Var(X) + \sqrt{B}\,.
    \]
    One merit of the result presented above is that the difference between the variances (resp. expectations) of $X$ and $Y'$ does not depend on the number of terms $n'$ of $X$. This is crucial in order to be able to apply Lemma \ref{lemma:tv-gaussians} to show that whenever $n'$ is large enough, $X$ is close to some $Y'$ as described above.

    \paragraph{Step 3: $Y'$ is similar to a Gaussian.} In this step, we use the same arguments as in \textbf{Step 1} to find a sufficient condition for $n'$ so that $Y'$ is $O(\eps)$-close in total variation distance to some $\dgaussvar_{Y'}\sim\Dgauss(\E[Y'],\Var(Y'))$. In particular, the quantities of interest are three. First, the ratio of the sum of the third centralized moments of $Y_i'$ to the variance of $Y'$, for which the upper bound we provided in \textbf{Step 1} continues to hold. Second, the lower bound for the variance of $Y'$, which is $m\cdot\gamma$. Third, the shift distance $\delta_{Y'}$ in which, $m$ will appear in the denominator in the position of $n'$. 
    
    We have that $m\ge \frac{\Var(X)}{\Var_{\vec b'}(W)} \ge n'\cdot\gamma/\sqrt{B}$. Therefore, applying the similar demands for the shift distance as in \textbf{Step 1}, we get the following sufficient demand for $n'$
        \[
            n'\ge n_3, \text{ where }n_3 = O\left( \frac{B^{7.5}}{\eps^2\cdot \gamma^8} \right)\,.
        \]
        
    \paragraph{Step 4: The Gaussian approximations are close.} In this step, we make use of Lemma \ref{lemma:tv-gaussians} in order to find sufficient conditions for $n'$ so that 
    $X$ and $Y'$ are $O(\eps)$ close in total variation distance. 
    We have that $|\E[X]-\E[Y']|\le |\E_{\vec b'}[W]| \le \E_{\vec b'}[|W-M_{\vec b'}|] + |M_{\vec b'}| \le B^{1/4}+L$ (due to assumptions \eqref{assumption:bounded-moments} and \eqref{assumption:bounded-modes}) and $|\Var(X)-\Var(Y')|\le \sqrt{B}$ and also $\Var(X)\le \Var(Y')$. Hence, by Lemma \ref{lemma:tv-gaussians}, which bounds the total variation distance between two discretized Gaussians using the differences of their parameters, we get that it is sufficient that
    \[
        \Var(X)\ge (B^{1/4}+L)^2/\eps^2 \text{ and }\Var(X) \ge \sqrt{B}/\eps\,.
    \]
    We know that $\Var(X) \ge n'\cdot \gamma$ (by assumption \eqref{assumption:variance-lower-bound}). We arrive to the following sufficient condition for $n'$.
    \[
        n' \ge n_4, \text{ where } n_4 = O\left( \frac{L^2+\sqrt{B}}{\eps^2 \cdot \gamma^2} \right)\,.
    \]
    
    Gathering all of the conditions for $n'$, we get that $\ncrit =\frac{1}{\eps^2}\cdot \poly(B,L,1/\gamma)$.
        
    \paragraph{Step 5: Discretization.} Finally, we make use of Theorem \ref{theorem:single-term} in order to discretize the space of possible parameter vectors. In particular, we find a sparse set (subset of $\calA_{\varrho}$) that contains (for any input distribution of $X$) some $\vec b$ so that if $Y=\sum_{i\in[m]}Y_i$ with $(Y_i)_i$ i.i.d. with distribution $\calE_{\vec T}(\vec b)$, then $Y,Y'$ are $O(\eps/m)$ close in total variation distance. We apply Theorem \ref{theorem:single-term} with error margin $\eps/m$, using the fact that $m\le n'\cdot\sqrt{B}/\gamma$ to quantify our results. 
    
    We get that $X$ is $O(\eps)$ close to $Y$ in total variation distance.
\end{proof}

\paragraph{Acknowledgment.} We thank Dimitris Fotakis for useful discussions on various stages of this work. We would also like to thank the anonymous reviewers of NeurIPS 2022 for their feedback.

\bibliographystyle{alpha}
\begingroup
    \setlength{\bibsep}{12pt}
    \bibliography{references.bib}
\endgroup

\appendix

\newpage

\section{Additional Notation, Preliminaries and General Tools} 
\label{appendix:appendix}

In this section, we provide some notation that will be useful in the proofs. Moreover, we report a collection of existing results that we will apply in our proofs.

\subsection{Additional Notation}
\label{appendix:notation}
\paragraph{Standard Notation.}
We denote with $\nats$ the set of natural numbers $\{1,2,3,\dots\}$ and $\natszero=\{0\}\cup\nats$. We denote with $\ints$ the set of integer numbers $\{\dots,-2,-1,0,1,2,\dots\}$ and with $\reals$ the set of real numbers. 
We denote vectors using bold letters e.g., $\vec a, \vec T$. 
We let $\|\cdot\|_p$ denote the $L_p$ norm with $p \in \{1,2,\infty\}$. In general, we let $\| \cdot \| = \| \cdot \|_2$ be the standard Euclidean norm. The inner product between two vectors $\vec a, \vec b$ is denoted by $\vec a \cdot \vec b$.
We use $\Gamma$ for the Gamma function and $\zeta$ for the Riemann zeta function.
For any $r>0$, $k\in\nats$, $\vec a\in \reals^k$, we denote with $\ball_{r}(\vec a)$ (resp. $\ball_{r}[\vec a]$) the open (resp. closed) ball centered at $\vec a$ with radius $r$.

\paragraph{Probability.} For a random variable $X$, we may denote its distribution by $\calL(X)$. We denote by $\Pr, \E, \Var, \Cov$ the probability, the expectation, the variance and the covariance operators. We denote $\Be, \Bin, \NBin, \Geo, \TG, \Dgauss, \Poi, \mathrm{Uni};$
the Bernoulli, Binomial, Negative Binomial, Geometric, Truncated Geometric, discretized Gaussian, Poisson and Uniform probability distribution respectively.

\paragraph{Notions of Distance for Distributions.}
Let $P, Q$ be two probability measures in the discrete probability space $(\Omega, \mathcal{F})$. The \emph{total variation distance} or \emph{statistical distance} between $P$ and $Q$, denoted $\tv(P,Q)$, is defined as $\tv(P,Q) = \frac{1}{2}\sum_{x \in \Omega}|P(x) - Q(x)| = \max_{A \in \mathcal{F}}|P(A)-Q(A)|$. The \emph{Kullback–Leibler divergence} (or simply, KL divergence), denoted $D_{KL}(P \parallel Q)$, is defined as $D_{KL}(P \parallel Q) = \E_{x \sim P}\!\left[\log \left(\frac{P(x)}{Q(x)}\right)\right] = \sum_{x \in \Omega}P(x) \log \left(\frac{P(x)}{Q(x)} \right)$. 

\paragraph{Exponential Families} For $k\in\natszero$, $\calA\subseteq\reals^k$ and $\vec T:\ints\to\reals^k$, we denote with $\calE_{\vec T}(\calA)$ the exponential family with sufficient statistics $\vec T$ and parameter space $\calA$. In particular, if $W\sim\calE_{\vec T}(\vec a)$ for some $\vec a\in\calA$, then for any $x\in\ints$
\[
    \Pr[W=x] \propto \exp(-\vec a\cdot \vec T(x))\,.
\]
We will use the notation $\Pr_{\vec a}[W=x]$ (similarly $\E_{\vec a}[W]$ and $\Var_{\vec a}(W)$ for expectation and variance correspondingly) to refer to the probability that $W=x$ given that $W\sim\calE_{\vec T}(\vec a)$, whenever it is clear by the context that the distribution of $W$ belongs in $\calE_{\vec T}(\calA)$. Note that we will use the following notation $\E_{W}[|W|^0] = \E_{W}[\vec 1\{ W\neq 0 \}] = \Pr_{W}[W\neq 0]$, i.e., we interpret $0^0$ as $0$. Note that in the general case, an exponential family $\calE$ supported on $\ints$ is defined in terms of some additional function  $h:\ints\to(0,+\infty)$ so that if $W\sim \calE$, then $\Pr[W=x]\propto h(x)\cdot\exp(-\vec a\cdot \vec T(x))$. In this case, we use the notation $W\sim\calE_{\vec T,h}(\vec a)$. However, we can reduce this setting to $h\equiv 1$ by considering $\calA'=\calA\times\{1\}$ and $\vec T' = (\vec T, -\log_e(h(x)))$. We also define the logarithmic partition function $\Lambda_{\vec T, h} : \reals^k \rightarrow \reals_{+}$ as $\Lambda_{\vec T, h}(\vec a) = \log \Big( \sum_{x=0}^{\infty} h(x) \exp( -\vec a \cdot \vec T(x) ) \Big)$.

\paragraph{Modes of Distributions.} For any distribution $\calD$ over the lattice of integers $\ints$, we consider the set of modes of $\calD$ to be $\modes = \arg\max_{x\in\ints}\calD(x)$ (where $\calD(x) = \Pr_{W\sim\calD}[W=x]$). We say that $\calD$ is unimodal if there exists some (mode) $M$ in $\ints$ such that if $W\sim \calD$, then
\begin{align*}
    \Pr[W=x]&\ge \Pr[W=x+1]\,, \text{ for any }x\ge M \text{ and} \\
    \Pr[W=x]&\le \Pr[W=x+1]\,, \text{ for any }x< M\,.
\end{align*}
Note that it might be the case that $\Pr[W=x]=\Pr[W=x+1]$, and therefore $\calD$ could have many neighboring modes (although we call it unimodal). 

For $\vec a\in\calA$, consider the distribution $\calE_{\vec T}(\vec a)$ which lies in the exponential family $\calE_{\vec T}(\calA)$. We denote with $\modes_{\calE_{\vec T}(\vec a)}$ or simply $\modes_{\vec a}$ (whenever $\vec T$ is clear by the context) the set of modes of $\calE_{\vec T}(\vec a)$. We also use the notation $M_{\vec a}$ to refer to any specific mode of $\calE_{\vec T}(\vec a)$, like in $\Pr_{\vec a}[W=M_{\vec a}]$ (the referenced point could equivalently be any mode of $\calE_{\vec T}(\vec a)$ since all modes are assigned the same probability mass).

We say that the exponential family $\calE_{\vec T}(\calA)$ is unimodal if $\calE_{\vec T}(\vec a)$ is unimodal for each $\vec a\in\calA$. We denote with $\modes_{\calA}$ the union of the sets of modes of the distributions in $\calE_{\vec T}(\calA)$, i.e., $\modes_{\calA} = \cup_{\vec a\in\calA} \modes_{\vec a}$.

\paragraph{Sets and Set Operators.}
A polyhedron is the intersection of finitely many affine halfspaces.
A cone is a subset $K \subseteq \reals^k$ with $\vec 0 \in K$ and $\alpha \vec y \in K$ for all $\vec y \in K$ and $\alpha \in \reals_+$. A polyhedral cone is a polyhedron that is a cone. We say that $\setoperator$ is an extensive set operator on $\reals^k$ if for any set $\calA\subseteq\reals^k$, $\setoperator\calA$ is a subset of $\reals^k$ and $\calA\subseteq \setoperator\calA$. We will make use of some particular extensive set operators:
The closure operator $\closure$: $\closure \calA = \calA \cup \partial \calA$. We also use the notation $\ccalA := \closure\calA$.
The convex hull operator $\chull$:
    \begin{equation}
        \label{eq:convex}
     \chull \calA = \left\{\sum_{i\in[n]}t_i\vec a_i: \vec a_i\in\calA, t_i\in[0,1], \sum_{i\in[n]}t_i=1, n\in\nats\right\},
    \end{equation}
    for any $\calA\subseteq\reals^k$. The conical hull operator $\conehull$:
    \begin{equation}
        \label{eq:cone}
        \conehull \calA = \left\{\sum_{i\in[n]}t_i\vec a_i: \vec a_i\in\calA, t_i\ge 0, n\in\nats\right\},
    \end{equation}
    for any $\calA\subseteq\reals^k$
    The $\varrho$-conical hull operator $\varrho\-\conehull$, where $\varrho>0$:
    \[  
        \varrho\-\conehull \calA= \calA \cup (\conehull\calA \setminus \ball_\varrho(\vec 0))\,.
    \]
    In other words, $\varrho\-\conehull$ operator inserts in $\calA$ all points of the conical hull of $\calA$ with norm at least $\varrho$.
    The shade operator:
    \[
        \shade \calA = \{t\vec a: \vec a\in\calA, t\ge 1\}\,.
    \]
We will use these set operators to rule out the possibility that $\calA$ is a very contrived set that makes the proper covering problem unreasonably difficult or complicated, since we focus on providing the first general approach for the proper covering problem. In particular, we will relax our demand that the cover is proper, by enabling sums of random variables that belong in a slightly wider exponential family by enlarging $\calA$ appropriately.

\subsubsection{What do Proper Covering and Learning mean for SIIERVs?}
\label{subsubsection:proper-meaning}
\noindent\textbf{(Weakly) Proper Covers.} Below, we give a slightly relaxed definition for the notion of proper covers for our setting. We will use these set operators to rule out the possibility that $\calA$ is a very contrived set that makes the proper covering problem unreasonably difficult or complicated, since we focus on providing the first general approach for the proper covering problem. In particular, we will relax our demand that the cover is proper, by enabling sums of random variables that belong in a slightly wider exponential family by enlarging $\calA$ appropriately.
\begin{definition}[Proper Covers of \siiervs]\label{definition:proper}
    Let $\calC$ be a cover of $\calE_{\vec T}(\calA)$-\siirvs\ of order $n$, where $\calA\subseteq\reals^k$. Consider $R \ge 1$ and $\setoperator$ is an extensive set operator on $\reals^k$. If any element of $\calC$ is the distribution of an $\calE_{\vec T}(\setoperator \calA)$-SIIRV of order $R\cdot n$, then we say that $\calC$ is $(R, \setoperator)$-proper.
\end{definition}

\noindent\textbf{(Weakly) Proper Learning.} For any $n,k\in\nats$, any $\calA\subseteq\reals^k$ and any $\vec T:\ints\to\reals^k$ (so that $\calE_{\vec T}(\calA)$ is well defined), we say that the the class of $\calE_{\vec T}(\calA)$-\siirvs\ of order $n$ can be $(R,\setoperator)$-properly learned if there exist some $R\ge 1$ and some extensive set operator $\setoperator$ such that for any $\eps,\delta\in(0,1)$ there exists some algorithm $A$ and some polynomial $m$ on $1/\eps,1/\delta$ (and other relevant parameters) such that $A$, given $m$ independent samples from some unknown $\calE_{\vec T}(\calA)$-SIIRV $X$ of order $n$, outputs some $\calE_{\vec T}(\setoperator\calA)$-SIIRV $Y$ of order $R\cdot n$ such that $\tv(X,Y)\le \eps$, with probability at least $1-\delta$.

\subsection{General Tools}

We show that our novel notion of distance, the structural distance is actually a distance metric.
\begin{lemma}\label{lemma:new_distance}
    The structural distance (Definition \ref{definition:new-distance}) is a distance metric between discrete distributions.
\end{lemma}
\begin{proof}
Note that $\dnew$ is non-negative. It is sufficient to prove the three following properties:\\
\noindent\texttt{Identity of indiscernibles.} $\dnew(\calD_1,\calD_2) = 0$ if and only if $\calD_1$ and $\calD_2$ are equivalent. The first direction of the property is satisfied, since when $\eps=0$, any point outside the support of $\calD_1$ must be assigned zero mass by $\calD_2$ (hence they have common support) and any point in the support must have proportionally equivalent mass (which implies exactly equal mass, since the probabilities must sum to $1$). For the other direction, observe that if $\calD_1$ and $\calD_2$ are equivalent, then the structural distance is zero.\\
\noindent\texttt{Symmetry.} Observe that Definition \ref{definition:new-distance} is, in fact, symmetric and so $\dnew(\calD_1,\calD_2) = \dnew(\calD_2,\calD_1)$.\\
\noindent\texttt{Triangle Inequality.} We have to show that $\dnew(\calD_1,\calD_3) \le \dnew(\calD_1,\calD_2) + \dnew(\calD_2,\calD_3)$. Let $\dnew(\calD_1,\calD_2) = \eps_1$ and $\dnew(\calD_2,\calD_3) = \eps_2$.
        
Consider, first, any $x\in\calX$ such that $\calD_1(x) > (\eps_1+\eps_2)\cdot p_1$. Then $\calD_1(x) > \eps_1\cdot p_1$ and hence, for the point $x$ and the pair $(\calD_1,\calD_2)$, it must hold $\calD_2(x)/p_2 = \calD_1(x)/p_1 > \eps_1+\eps_2$. Similarly, we have $\calD_2(x)/p_2 > \eps_1 + \eps_2$ and so $\calD_3(x)/p_3 = \calD_2(x)/p_2 > \eps_1 + \eps_2$. This implies that $\calD_3(x)/p_3 = \calD_1(x)/p_1$.

For the rest points $x\in\calX$ with $\calD_1(x) \le (\eps_1+\eps_2)\cdot p_1$, we have to take cases for $\calD_1 (x)$ (either $\le \eps_1 p_1$ or $> \eps_1 p_1$) and after applying what we know about $\dnew(\calD_1,\calD_2)$, take cases for $\calD_2(x)$ (and apply knowledge about $\dnew(\calD_2,\calD_3)$). In any case, we get that $\calD_3(x) \le (\eps_1+\eps_2)\cdot p_3$. Therefore $\dnew(\calD_1,\calD_3) \le \eps_1+\eps_2$.
\end{proof}

We continue with a collection of general tools that we are going to need for our proofs. The Minkowski-Weyl Theorem shows that any polyhedron can be either represented in a constrained or a finitely generated form.
\begin{proposition}
[Minkowski-Weyl Theorem {[\cite{minkowski1986},\cite{weyl1935}]}]
\label{proposition:minkowski-weyl}
    Let $\calC\subseteq\reals^k$ for $k\in\ints$. Then the following are equivalent:
    \begin{enumerate}
        \item There exists $t\in\nats$ and some matrix $H\in\reals^{k\times t}$ such that $\calC = \{\vec u\in\reals^k : H^{T} \vec u\ge \vec 0\}$.
        \item There exists $s\in\nats$ and some matrix $Z\in\reals^{k\times s}$ such that $\calC = \{\vec u\in \reals^k : \vec u = Z\vec x, \vec x\ge \vec 0\}$.
    \end{enumerate}
\end{proposition}

The next standard inequality shows that in order to control the statistical distance for two distributions, it suffices to control the KL divergence.
\begin{proposition}
[Pinsker's Inequality]
\label{proposition:pinsker}
For any probability distributions $P$ and $Q$ on a common measurable space, the following inequality holds.
\[
    \tv(P,Q)\le \sqrt{\frac{1}{2}\kl(P \parallel Q)}\,.
\]
\end{proposition}

The following lemma is a useful tool in various parts of our proofs. It essentially controls ratios of sums of positive quantities.
\begin{lemma}
[Ratio of Sums Inequality]
\label{lemma:ratio-ineq}
Let $a_1,\ldots, a_n$ and $b_1, \ldots, b_n$ be \emph{positive} numbers. Then, it holds that
\[
\min_{i \in [n]} \frac{a_i}{b_i} \leq \frac{\sum_{i \in [n]} a_i}{\sum_{i \in [n]} b_i} \leq \max_{i \in [n]} \frac{a_i}{b_i}\,.
\]
\end{lemma}
\begin{proof}
We have that
\[
\sum_{i \in [n]} a_i = 
\sum_{i \in [n]} b_i \left ( \frac{a_i}{b_i} \right )
\leq \sum_{i \in [n]} b_i \max_{j \in [n]} \frac{a_j}{b_j}\,.
\]
Hence, we get that
\[
\frac{\sum_{i \in [n]} a_i}{\sum_{i \in [n]} b_i}
\leq \max_{j \in [n]} \frac{a_j}{b_j}\,.
\]
The other direction follows similarly.
\end{proof}

The next tool characterizes the expectation and the variance of the sufficient statistics vector
in terms of the log-partition function. Moreover, it provides a way to control the KL divergence of two exponential family distributions.
\begin{lemma}
[Moments and KL Divergence for Exponential Families]
\label{lemma:exp-fami-tv-kl}
Let $\calE_{\vec T,h}$ be an exponential family with parameters $\vec a \in \reals^k$, log-partition function $\Lambda(\cdot) = \Lambda_{\vec T, h}(\cdot)$ and range of natural parameters $\calA = \calA_{\vec T, h}.$ The following hold for the distribution $\calE(\vec a) \in \calE_{\vec T, h}.$
\begin{itemize}
    \item[$(i)$] For all $\vec a \in \calA$, it holds that
    \[
    \E_{\vec x \sim \calE(a)}[\vec T(\vec x)] = \nabla \Lambda(\vec a)\,.
    \]
    \item[$(ii)$] For all $\vec a \in \calA$, it holds that
    \[
    \Cov_{\vec x \sim \calE(a)}[\vec T(\vec x)] = \nabla^2 \Lambda(\vec a)\,.
    \]
    \item[$(iii)$] For all $\vec a, \vec a' \in \calA$ and for some $\vec \xi \in L(\vec a, \vec a')$\footnote{We denote $L(\vec x, \vec y) = \{ \vec z \in \reals^k : \vec z = \lambda \vec x + (1-\lambda)\vec y, \lambda \in [0,1]\}$.}, it holds that
    \begin{equation}
    \label{equation:kl-div-exp-fam}
    \kl(\calE(\vec a) \parallel \calE(\vec a')) = (\vec a' - \vec a)^T \nabla^2 \Lambda(\vec \xi) (\vec a' - \vec a)\,.
    \end{equation}
\end{itemize}
\end{lemma}

The upcoming lemma control the statistical distance for sums of independent random variables.
\begin{lemma}
[TV-Subadditivity for Sums of Random Variables]
\label{lemma:tv-subadd} 
Let $(X_i)_{i \in [n]}$ be a sequence of $n$ independent random variables and $(Y_i)_{i \in [n]}$ be also a sequence of $n$ independent random variables. Then, we have that
\[
    \tv\left(\sum_{i \in [n]}  X_i,\sum_{i \in [n]} Y_i\right) \le \sum_{i \in [n]} \tv(X_i, Y_i)\,.
\]
\end{lemma}

This lemma controls the statistical distance between two Poisson distributions.
\begin{lemma}
[Statistical Distance of Poisson RVs]
\label{lemma:tv-poissons} 
Let $\lambda_1,\lambda_2 > 0$. Then, it holds that:
\[
    d_{TV}(\mathrm{Poi}(\lambda_1), \Poi(\lambda_2)) \leq  \frac{e^{|\lambda_1 - \lambda_2|} - e^{-|\lambda_1 - \lambda_2|}}{2}\,.
\]
\end{lemma}

This lemma controls the statistical distance between two discretized Gaussian random variables.
\begin{lemma}
[Statistical Distance of Discretized Gaussian RVs]
\label{lemma:tv-gaussians}
Let $\mu_1,\mu_2\in \reals $ and $0<\sigma_1\le \sigma_2$. Then, it holds that: 
\[
    \tv(\Dgauss(\mu_1,\sigma_1^2),\Dgauss(\mu_2,\sigma_2^2)) \le \frac{1}{2}\left(\frac{|\mu_1-\mu_2|}{\sigma_1}+\frac{\sigma_2^2-\sigma_1^2}{\sigma_1^2}\right) \,.
\]
\end{lemma}

The next lemma provides a bound on the shift distance for SIIRVs.
\begin{lemma}
[Statistical Distance of Shifted SIIRV]
\label{lemma:dshift-sum}
Let $X = \sum_i X_i$ be the sum of independent integer-valued random variables. Then, it holds that
\[
    \tv(X, X+1) \le \frac{\sqrt{2/\pi}}{\sqrt{\frac{1}{4}+\sum_i(1-\tv(X_i,X_i+1))}}\,.
\]
\end{lemma}

The next lemma is a standard tool that essentially says that an SIIRV that has bounded third moment and TV shift invariance is close to a discretized Gaussian in statistical distance. The TV shift invariance bound is crucial; note that if we drop this property, standard CLT theorems imply bounds for the weaker Kolmogorov distance.
\begin{lemma}
[Discretized Gaussian Approximation (\cite{chen2010normal})]
\label{lemma:discr-gaussian-approx}
Let $X_1, ..., X_n$ be a finite sequence of independent integer-valued random variables and let $X =\sum_{i \in [n]} X_i$. If $\mu_i = \E[X_i], \sigma_i^2 = \Var(X_i), \beta_i = \E[|X_i-\mu_i|^3], \mu = \sum_{ i \in [n]} \mu_i, \sigma^2 = \sum_{i \in [n]} \sigma_i^2, \beta = \sum_{i \in [n]} \beta_i$ and
\[
    \sup_{i \in [n]} \tv(X-X_i, X-X_i+1) \le \delta\,,
\]
then, if $Z$ is distributed according to the discretized Gaussian distribution $\Dgauss(\mu, \sigma^2)$, we have that
\[
    \tv(X,Z)\le O(1/\sigma)+O(\delta)+O(\beta/\sigma^3)+O(\delta\beta/\sigma^2)\,.
\]
In particular, we have that
\[
    \tv(X,Z) \le \delta\left(1+\frac{3}{2}\frac{\beta}{\sigma^2}\right)+\frac{1}{\sigma}\left(\frac{1}{2\sqrt{2\pi}}+(5+3\sqrt{\pi/8})\frac{\beta}{\sigma^2}\right)\,.
\]
\end{lemma}


\section{Learning SIIURVs and SIIERVs}
\label{section:general-learning}
\label{section:learning}
We now formally state our main learning results. To do this, we begin by formally setting up our learning framework for SIIRVs.

\paragraph{Learning Framework.} In the problem of learning an SIIRV $X = \sum_{i \in [n]} X_i$, the learner is given the value $n$ (of the number of summands) and has sample access to independent draws from an unknown target $X$. The goal of the learning algorithm is to output a hypothesis distribution $\wt{X}$ that is $\eps$-close to $X$ in total variation distance, i.e., $\tv(X, \wt{X}) \leq \eps$, with probability at least $1-\delta$. The accuracy $\eps$ and the confidence $\delta$ are both provided to the learner as input.

In Appendix \ref{subsection:learning-SIIURVs}, the target $X$ will be an SIIURV of order $n$ for some given $n \in \nats$ (i.e., a sum with $n' \leq n$ random terms) that belongs to the family of distributions that contain all the sums of at most $n$ random variables that satisfy Assumption \ref{assumption:SIIURV}. We are going to provide a learning algorithm for this class of distributions.

In Appendix \ref{subsection:learning-SIIERVs}, the target $X$ will be an $\calE_{\vec T}(\calA)$-SIIRV of order $n$, i.e., a sum of at most $n$ random variables each one belonging in the exponential family $\calE_{\vec T}(\calA)$ satisfying Assumption \ref{assumption:proper}. We will give a (weakly) proper learning algorithm in the sense that the output will be an $\calE_{\vec T}(\calA')$-SIIERV of order $m$ where $m$ and $\calA'$ will be slightly different that $n$ and $\calA$ respectively.

\paragraph{Common Technical Tool.} Our learning algorithms (Figure \ref{algorithm:learning-SIIURVs} and Figure \ref{algorithm:learning-SIIERVs}) use hypothesis testing as a distinct tool for the learning procedures.
Hypothesis testing will appear in various points of the algorithms; in the sparse learning phase (Claim \ref{claim:improper-learning-sparse-SIIURVs} and Claim \ref{claim:proper-learning-sparse-SIIERVs}), in the proper dense one (Claim \ref{claim:proper-learning-dense-SIIERVs}) and in the hypothesis selection of the second stage (Proposition \ref{proposition:learning-hypothesis-selection}). Intuitively, for some desired accuracy $\eps > 0$, given a collection of $M$ candidate hypothesis distributions, one of which is $\eps$-close in total variation distance to the target distribution of $X$, a hypothesis testing algorithm draws $\wt{O}(\log(M)/\eps^2)$ samples from $X$, runs in time polynomial in $M$ and $1/\eps$ and outputs a hypothesis (among the $M$ candidates) that is $O(\eps)$-close to the true $X$, with high probability. Such testing procedures have been studied by a range of authors (see e.g., \cite{yatracos1985rates, daskalakis2014faster, acharya2014near, daskalakis2013learningSIIRV, daskalakis2015learningPBD, de2014learning, de2018learning}). For concreteness we are going to recall some standard results later.

\subsection{Learning SIIURVs}
\label{subsection:learning-SIIURVs}

Assume that the target $X$ is an SIIURV satisfying Assumption \ref{assumption:SIIURV}. We first provide a discussion on some previous results and the oracles we require.

\subsubsection{Hypothesis Testing and Oracle Access}
The following proposition about hypothesis selection can be found in \cite{daskalakis2015learningPBD}.
\begin{proposition}
[Hypothesis selection (Lemma 8 at \cite{daskalakis2015learningPBD})]
\label{proposition:learning-hypothesis-selection}
There exists an algorithm
$\textsc{SelectHypothesis}^X(H_1, H_2, \eps, \delta)$, which is given sample access to 
a distribution $X$, two hypothesis
distributions $H_1, H_2$ for $X$,
an accuracy parameter $\eps$ and a confidence parameter $\delta > 0$, draws
\[
O(\log(1/\delta)/\eps^2)
\]
samples from $X$ and,
in time polynomial
in the number of samples,
returns some
$H \in \{H_1, H_2\}$ with the following guarantee: If $\tv(H_i, X) \leq \eps$ for some $i \in \{1,2\}$, then the distribution
$H$ that $\textsc{SelectHypothesis}^X$ returns satisfies $\tv(H,X) \leq 6\eps$.
\end{proposition}
The routine $\textsc{SelectHypothesis}^X(H_1, H_2, \eps, \delta)$ runs a competition between $H_1$ and $H_2$ as follows (for more details and a proof, we refer to \cite{daskalakis2015learningPBD}): It first computes
the set $\calW_1 = \{ x | H_1(x) > H_2(x) \}$ and the 
probabilities $p_i = H_i(\calW_1), i \in \{1,2\}$.
The routine draws $O(\log(1/\delta)/\eps^2)$ independent samples from $X$ and calculates the fraction $\tau$ of these samples that fall inside $\calW_1$. If $\tau > p_1 - \eps,$ it selects $H_1$ as the winner and if $\tau < p_2 + \eps$, it chooses $H_2$; otherwise, it declares a draw and returns either $H_i$. 

\begin{remark}
We underline that the result in \cite{daskalakis2015learningPBD} does not require sample access to $H_1,H_2$. These distributions are given as input (and their description is short enough to be read by the learner) and the set $\calW_1$ can then be computed efficiently. The work of \cite{de2014learning} deals with scenarios where this is not the case (e.g., the domain is exponentially large in $n$) and in this case sample and evaluation oracle access to $H_1, H_2$ is needed too. We discuss this point later.
\end{remark}

The $\textsc{SelectHypothesis}^X$ routine is the main tool of the following hypothesis testing mechanism (Proposition \ref{proposition:tournament}), for which we refer the reader to \citep{de2014learning,de2018learning}
and, more generally, to e.g., \cite{daskalakis2015learningPBD, daskalakis2013learningSIIRV}:
\begin{proposition}
[Tournament Selection \cite{de2018learning}]
\label{proposition:tournament}
Let $D$ be a distribution over $W \subseteq \ints$ and let $\calH_\eps = \{ H_j \}_{j \in [M]}$ be a collection of $M$ hypothesis distributions over $W$ with the property that there exists $i \in [M]$ such that $\tv(D, H_i) \leq \eps$. There exists an algorithm $\textsc{SelectTournament}^D$ which is given $\eps$, a confidence parameter $\delta$ and is provided with access to $(i)$ a source of i.i.d. draws from $D$ and from $H_i$ for all $i \in [M]$; and $(ii)$ an 'evaluation oracle' $\textsc{eval}_{H_i}$ for each $i \in [M]$, which, on input $w \in W$, deterministically outputs the value $H_i(w)$ and that has the following behavior: It makes $m = O((1/\eps^2) \cdot (\log(M) + \log(1/\delta)))$ draws from $D$ and from each $H_i, i \in [M]$ and $O(m)$ calls to each oracle $\textsc{eval}_{H_i},  i \in [M]$. It runs in time $\poly(m, M)$\footnote{We count each call to an evaluation oracle $\textsc{eval}_{H_i}$ and draw from a $H_i$ distribution as unit time.} and, with probability $1-\delta$, it outputs an index $i^\star \in [M]$ that satisfies $\tv(D, H_{i^\star}) \leq 6\eps$.  
\end{proposition}
The routine $\textsc{SelectTournament}$ performs a tournament by running the procedure of Proposition \ref{proposition:learning-hypothesis-selection} $\textsc{SelectHypothesis}^X(H_i, H_j, \eps, O(\delta/M))$
for every pair $(H_i, H_j), i < j$ of distributions in the collection of Proposition \ref{proposition:tournament} of size $M$
and either outputs the distribution that was never a loser
or it fails.
The bound on the running time is a result of the corresponding time bound of the $\textsc{SelectHypothesis}^X$
routine. The sample complexity is a result of the union bound over the competitions.

\paragraph{Sampling \& Evaluation Oracles.} For this section, we only need sample access to the target $X$ in order to run the version of Proposition \ref{proposition:learning-hypothesis-selection}. 
During the learning phase of the sparse instances, we will construct the sparse cover and we will perform the tournament procedure for the distributions in the cover. 
Crucially, the sparse forms have bounded support and its size does not depend on $n$. 
Hence, for each sparse form, we have access to an efficient evaluation oracle for the purposes of Proposition \ref{proposition:tournament}.
For any two distributions in the cover $H_i, H_j$ with domain $W \subseteq \ints$, the algorithm can efficiently compute the set $W_{ij} = \{ w \in W  | H_i(w) \geq H_j(w) \}$ without additional assumptions. In the dense case, the algorithm will estimate the best fitting discretized Gaussian distribution and we do not need to contruct any cover or run any tournament procedure.

\subsubsection{The Result and the Algorithm}
Our main learning result for SIIURVs follows.
\begin{theorem}
[Learning SIIURVs]
\label{theorem:learning-SIIURVs}
Under Assumption \ref{assumption:SIIURV}, for any $n \in \nats$, accuracy $\eps >0$ and confidence $\delta > 0$, there exists an algorithm $\textsc{LearnerSIIURV}^X$ (see Figure \ref{algorithm:learning-SIIURVs}) with the following properties:
Given $n, \eps, \delta$ and sample access to independent draws from an
unknown SIIURV $X$ of order $n$, the algorithm uses
    \[
    m = O \left (\frac{1}{\eps^2} \log(1/\delta) \right) + O\left( \poly(B,1/\gamma, 1/\eps) \right)
    \]
    samples from $X$ and, in time 
    \[
    \poly \left(m, L^{\poly(B,1/\gamma, 1/\eps)}  \right)\,,
    \]
    outputs a (succint description of a) distribution $\wt{X}$ which satisfies
    $\tv(X, \wt{X}) \leq \eps$, with probability at least $1-\delta$.
\end{theorem}

Our algorithm works as follows.
\begin{figure}[ht]
    \centering
\begin{mdframed}[
    linecolor=black,
    linewidth=1pt,
    roundcorner=3pt,
    backgroundcolor=white,
    userdefinedwidth=\textwidth,
]
\textit{Algorithm for SIIURVs}:  $(\eps,\delta)$-Learning SIIURVS $X = \sum_{i \in [n]} X_i$. 
\begin{enumerate}
    \item[\text{1.}] Run $\textsc{LearnSparse}^X(n, \eps, \delta/3)$ of Claim \ref{claim:improper-learning-sparse-SIIURVs} and get the distribution $H_S$.
    
    \item[\text{2.}] Run  $\textsc{LearnDense}^X(n, \eps, \delta/3)$ of Claim \ref{claim:learning-gaussian-SIIURVs} and get the distribution $H_D$.
    
    \item[\text{3.}] Return the distribution that is the output of $\textsc{SelectHypothesis}^X(H_S, H_D, \eps, \delta/3)$ of Proposition \ref{proposition:learning-hypothesis-selection}.
\end{enumerate}
\end{mdframed}
    \caption{Learning algorithm for SIIURVs.}
    \label{algorithm:learning-SIIURVs}
\end{figure}

\newcommand{\calCpr}{\calD_{\mathrm{E}}}
\newcommand{\calCimp}{\calD_{\mathrm{U}}}

We continue with a short discussion on how the algorithm works:
The learning algorithm (Figure \ref{algorithm:learning-SIIURVs}) of Theorem \ref{theorem:learning-SIIURVs} is separated in two distinct stages. In the first stage, it runs two different learning procedures, corresponding to the sparse and dense case of our main structural covering result for SIIURVs. At the end of this stage, two hypotheses are obtained and, hence, the second phase of the learning algorithm performs hypothesis testing in order to select the correct one.
In the first phase, the processes $\textsc{LearnSparse}^X$ and $\textsc{LearnDense}^X$ are performed. The $\textsc{LearnSparse}^X$ procedure (see  Claim \ref{claim:improper-learning-sparse-SIIURVs})
performs a tournament over the distributions of the cover of the sparse regime $\calCimp^{(s)}(\eps)$ with error $\eps$ and outputs the hypothesis/distribution $H_S$ that is closer to $X$. On the other hand, the the $\textsc{LearnDense}^X$ process (see Claim \ref{claim:learning-gaussian-SIIURVs}) estimates the parameters of a discretized Gaussian, which approximates the input sum $X$ and outputs this distribution. We denote by $H_D$ the output hypothesis of the dense procedure. 
Now, in the second phase, the learning algorithm runs the black-box procedure $\textsc{SelectHypothesis}^X$ (see Proposition \ref{proposition:learning-hypothesis-selection}) that chooses the winner between the two hypotheses $H_S$ and $H_D$, i.e., the one that is closer to $X$ with high probability.

\paragraph{The Proof of Theorem \ref{theorem:learning-SIIURVs}.}
\begin{proof}
The algorithm (Figure \ref{algorithm:learning-SIIURVs}) runs the routine $\textsc{LearnSparse}^X$ of Claim \ref{claim:improper-learning-sparse-SIIURVs} with input $(n, \eps, \delta/3)$ and gets the distribution $H_S$. Then, it runs $\textsc{LearnDense}^X$ of Claim \ref{claim:learning-gaussian-SIIURVs} with input $(n, \eps, \delta/3)$ and aims to learn the best fitting discretized Gaussian to the input sum $X$ and output the hypothesis distribution $H_D$. In order to conclude the SIIURV learning part, via Proposition \ref{proposition:learning-hypothesis-selection}, one can examine which hypothesis between $H_S$ and $H_D$ is closer to the target $X$, by running     $\textsc{SelectHypothesis}^X(H_S, H_D)$ with parameter $\eps, \delta/3$. In conclusion, with probability at least $1-\delta$, 
the algorithm will satisfy the desiderata of
Theorem \ref{theorem:learning-SIIURVs}. We divide the proof in a series of claims.

In Claim \ref{claim:improper-learning-sparse-SIIURVs}, we analyze an algorithm which learns sparse instances and outputs a hypothesis distribution $H_S$.
\begin{claim}
[Learning Sparse Instances]
\label{claim:improper-learning-sparse-SIIURVs}
Under Assumption \ref{assumption:SIIURV},
for any $n, \eps, \delta > 0$, there is an algorithm $\textsc{LearnSparse}^X(n, \eps, \delta)$ that given
\[
m = O \left ( \poly(B,1/\gamma,1/\eps)\cdot \log(L) + \frac{1}{\eps^2}\log(1/\delta) \right) 
\]
samples from the target SIIURV $X$ over $\ints$,
outputs a (succint description of a) hypothesis distribution $H_S$ with the following guarantee: If $X$ is $\eps$-close to a sparse form (see Theorem \ref{theorem:structural-SIIURV}),
then it holds that $\tv(H_S, X) \leq c_1 \eps$, for some universal constant $c_1 \geq 1$, with probability at least $1-\delta$.
Furthermore, the running time of the algorithm is
$\poly\left(m, L^{\poly(B,1/\gamma,1/\eps)} \right)$.
\end{claim}
\begin{proof}
Let $\eps > 0$ and assume that Assumption \ref{assumption:SIIURV} holds.
Since we are in the sparse form case, the algorithm can construct a cover of small size as described in the proof of Theorem \ref{theorem:structural-SIIURV} (see Appendix \ref{appendix:proof-siiurv}), i.e., by quantizing the probability mass on each point of each of a set of intervals that provably include some interval that contains the support of $X$.
The structural result of Theorem \ref{theorem:structural-SIIURV} implies that there exists a cover $\calCimp^{(s)} = \calCimp^{(s)}(\eps)$ of radius $\eps$ (i.e., a collection of probability distributions that contains an $\eps$-close -- in total variation distance -- representative for each distribution in $\calD$) whose size is equal to 
\[
\left |\calD_{\mathrm{U}}^{(s)} \right| \leq
L^{\poly(B,1/\gamma,1/\eps)}\,.
\]
By the hypothesis and from the structure of the cover, $X$ is $\eps$-close to an element of the set $\calCimp^{(s)}$. Note that Proposition \ref{proposition:tournament} is applicable since the learner can read the distributions of the cover as input and so no specific oracle access is required. We can apply the $\textsc{SelectTournament}^X$ algorithm with 
input the distributions' collection which lie in $\calCimp^{(s)}$ with accuracy $\eps$ and confidence $\delta$.
This concludes the proof. The sample complexity of the algorithm is
\[
m = O \left(\frac{1}{\eps^2} \left( \log \left|\calCimp^{(s)}\right| + \log(1/\delta) \right) \right)\,, 
\]
and the running time is $\poly\left(m, \left|\calCimp^{(s)}\right|\right)$.
\end{proof}

As a next step, we analyze the learning phase concerning the dense instances: In Claim \ref{claim:learning-gaussian-SIIURVs}, we deal with the SIIURV learning of the dense case, using the Gaussian approximation.
\begin{claim}
[Learning Dense Instances]
\label{claim:learning-gaussian-SIIURVs}
Under Assumption \ref{assumption:SIIURV},
for any $n, \eps, \delta > 0$, there is an algorithm $\textsc{LearnDense}^X(n, \eps, \delta)$ that given
\[
O(\log(1/\delta)/\eps^2)
\]
samples from the target SIIURV $X$ over $\ints$, runs in time $O(\log(1/\delta)/\eps^2)$ and
outputs a (succint description of a) hypothesis distribution $H_D$ with the following guarantee: 
If $X$ is $\eps$-close to a dense form (see Theorem \ref{theorem:structural-SIIURV}), then it holds that $\tv(H_D, X) \leq O(\eps)$, with probability at least $1-\delta$,
and $H_D$ is a discretized Gaussian distribution.
\end{claim}

\begin{proof}
Let $\eps > 0$ and assume that $X$ is $\eps$-close to a dense form SIIURV.
Let $X = \sum_{i \in [n]} X_i$ be an SIIURV and set $\mu = \E[X]$ and $\sigma^2 = \Var(X)$. There exists an algorithm that uses $O(\log(1/\delta)/\eps^2)$ samples from $X$ and runs in time $O(\log(1/\delta)/\eps^2)$ and, with probability at least $1-\delta$, computes estimates $\wh{\mu}$ and $\wh{\sigma}^2$ so that
\[
|\mu - \wh{\mu}| \leq \eps \sigma
\text{ and }
|\sigma^2 - \wh{\sigma}^2| \leq \eps \sigma^2\cdot O(1)\,.
\]
Our proof follows the steps presented in \cite{daskalakis2015learningPBD}: We will provide a routine achieving the desired estimation with probability at least $2/3$. Afterwards, there exists a standard procedure\footnote{The boosting argument requires running the weak estimators $O(\log(1/\delta))$ times in order to obtain two sequences of estimates $(\wh{\mu}_i)_{i \in [O(\log(1/\delta))]}$ and $( \wh{\sigma}_i^2)_{i \in [O(\log(1/\delta))]}$. Finally, the boosting process will output the median of these sequences of $O(\log(1/\delta))$ weak estimates.} 
that boosts the success probability to $1-\delta$ at the expense of a multiplicative $O(\log(1/\delta))$ overhead in the number of samples.\\

\noindent\texttt{Mean Estimation.} In order to weakly estimate the mean $\mu$, let $\{Z_i\}_{i \in [N]}$ be i.i.d. samples from $X$ and let $\wh{\mu} = \frac{1}{N} \sum_{i \in [N]} Z_i$. Chebyshev's inequality implies that
\[
\Pr \left[ |\wh{\mu} - \mu| \geq t \sqrt{\Var(X)}\right] = 
\Pr \left [ |\wh{\mu} - \mu| \geq t \sigma /\sqrt{N}\right] \leq 1/t^2\,.
\]
Choosing $t = \sqrt{3}$ and $N = O(3/\eps^2)$, we get that $|\wh{\mu} - \mu| \leq \eps \sigma$ with probability at least $2/3.$\\

\noindent\texttt{Variance Estimation.} Similarly, one can compute a weakly estimate for the variance $\sigma^2.$ Let $\{Z_i\}_{i \in [N]}$ be i.i.d. samples from $X$ and, using the Bessel's correction, let $\wh{\sigma}^2 = \frac{1}{N-1} \sum_{i \in [N]} (Z_i - \frac{1}{N} \sum_{i \in [N]} Z_i)^2$. We have that
\[
\E \left [\wh{\sigma}^2\right ] = \sigma^2 \text{ and } \Var \left (\wh{\sigma}^2 \right) = \sigma^4 \left( \frac{2}{N-1} + \frac{\kappa}{N} \right)\,,
\]
where $\kappa := \frac{\E[(X - \mu)^4]}{\sigma^4} -3$ is the excess kurtosis of the distribution of $X$. For the random variable $X = \sum_{i \in [n]} X_i$ with $X_i \sim \calE_{\vec T}(\vec a_i)$, we have that
\[
\kappa = \frac{1}{\sigma^4} \E[(X - \mu)^4] - 3= \frac{\sum_{i\in[n]}(\Var(X_i))^2 \cdot \left( \frac{\E[(X_i-\E[X_i])^4]}{(\Var(X_i))^2} - 3\right)}{\left(\sum_{i\in[n]}\Var(X_i)\right)^2} \leq \frac{1}{n} \cdot B/ \gamma^{2}\,,
\]
by using independence and conditions \eqref{cond:unimodal-SIIURV} and \eqref{cond:bounded-centered-SIIURV}. The second equality follows from the next computations: Define $\kappa_4(X) := \E[(X-\E[X])^4] - 3\Var(X)^2$ be the fourth cumulant of $X$. Since cumulants for sums of independent random variables are additive, we have that the excess kurtosis of $X$ is
\[
\kappa = \kappa(X) = \frac{\kappa_4(X)}{\Var(X)^2} = \frac{\sum_{i \in [n]} \kappa_4(X_i)}{(\sum_{i \in [n]} \Var(X_i))^2} = \frac{\sum_{i \in [n]} (\Var(X_i))^2 \cdot \kappa(X_i)}{(\sum_{i \in [n]} \Var(X_i))^2}\,.
\]
We expect $\kappa$ to vanish with $n$ since for a Gaussian distribution $W$ we have $\E[(W-\E[W])^4] = 3\sigma^4$ (and $X$ resembles a Gaussian as the number of terms increases). Note that the dense case corresponds to large number of terms in $X$ and, in particular $n\ge \Omega(B/\gamma^2)$ (see Appendix \ref{appendix:proof-siiurv}). Therefore, we might assume here that $\frac{1}{n} \cdot B/ \gamma^{2} = O(1)$.

Hence, we have that
\[
\Var(\wh{\sigma}^2) \leq \sigma^4 \left( \frac{2}{N-1} + \frac{\kappa}{N} \right) \leq
\frac{\sigma^4}{N} \cdot O(1)\,.
\]
Chebyshev's inequality implies that
\[
\Pr \left [ |\wh{\sigma}^2 - \sigma^2| \geq t \frac{\sigma^2}{\sqrt{N}} \cdot O(1)\right ] \leq \frac{1}{t^2}\,.
\]
Choosing $t = \sqrt{3}$ and $N = O(3/\eps^2)$, we get that $|\sigma^2 - \wh{\sigma}^2| \leq O(\eps)\cdot \sigma^2$ with probability at least $2/3$.\\

\noindent\texttt{Total Variation Gap.} Finally, we have that, using Lemma \ref{lemma:tv-gaussians} 
\[
\tv(X, \Dgauss(\wh{\mu}, \wh{\sigma}^2))
\leq 
\tv(X, \Dgauss(\mu, \sigma^2))
+
\tv(\Dgauss(\mu, \sigma^2),  \Dgauss(\wh{\mu}, \wh{\sigma}^2))\,.
\]
Since $X$ is in dense form and the cover has radius $\eps$, the first quantity of the right-hand side is at most $\eps$ and the second one, applying Lemma \ref{lemma:tv-gaussians}, gives
\[
\tv(\Dgauss(\mu, \sigma^2),  \Dgauss(\wh{\mu}, \wh{\sigma}^2))
\leq 
\frac{1}{2} 
\left(
\frac{\eps \sigma}{\sigma}
+ 
\frac{O(\eps)\cdot \sigma^2 }{\sigma^2}
\right) = O(\eps)\,,
\]
with high probability (where the randomness is over the estimates $\wh{\mu}$ and $\wh{\sigma}^2$), if $\sigma^2$ is large enough. Hence, the total variation distance between the two Gaussians is of order $O(\eps)$.
\\

\noindent\texttt{Conclusion of Claim \ref{claim:learning-gaussian-SIIURVs}.} So, we get that there exists an algorithm that computes the parameters $(\wh{\mu}, \wh{\sigma}^2)$ of a discretized Gaussian distribution so that $\tv(X, \Dgauss(\wh{\mu}, \wh{\sigma}^2)) \leq O(\eps)$, with high probability, using $\wt{O}(1/\eps^2)$ samples. We set $H_D = \Dgauss(\wh{\mu}, \wh{\sigma}^2)$. As a conclusion, there is an algorithm that given
$O(\log(1/\delta)/\eps^2)$ samples from the target SIIURV $X$ over $\ints$, runs in time $O(\log(1/\delta)/\eps^2)$ and
outputs a hypothesis discretized Gaussian distribution $H_D$ with the following guarantee: If $X$ is $\eps$-close to a dense form SIIURV, then it holds that $\tv(H_D, X) \leq O(\eps)$, with probability at least $1-\delta$.
\end{proof}

Combining the above claims concludes the proof.
\end{proof}

\subsection{Properly Learning SIIERVs}
\label{subsection:learning-SIIERVs}
In the problem of learning $\calE_{\vec T}(\calA)$-SIIRVs, the learner is given the value $n$ (of the number of summands), accuracy and confidence parameters $\eps, \delta \in (0,1)$ and has sample access to independent draws from an unknown $\calE_{\vec T}(\calA)$-SIIRV $X$. The goal of the learning algorithm is to output a hypothesis distribution $\wt{X}$ that is $\eps$-close to $X$ in total variation distance, i.e., $\tv(X, \wt{X}) \leq \eps$, with probability at least $1-\delta$. Recall that a weakly proper learner is an algorithm that outputs a distribution $\wt{X}$ that is itself a $\calE_{\vec T}(\calA')$-SIIRV, i.e., $\wt{X} = \sum_{i \in [m]} \wt{X}_i$ with $\wt{X}_i \sim \calE_{\vec T}(\wt{\vec a_i})$ with $\wt{\vec a_i } \in \calA' \subseteq \reals^k$ for any $i \in [m]$, where $m$ may be different than the input's order $n$ and $\calA'$ is a set containing $\calA$.

\subsubsection{Hypothesis Testing and Oracle Access}

\paragraph{Sampling \& Evaluation Oracles.} Apart from sample access to the target distribution $X$, we will require the following: Both in the sparse and the dense case, as we will see in the proof, we must be able to perform the hypothesis selection routine for the dense cover whose elements are $\calE_{\vec T}(\calA_{\varrho})$-SIIRVs. Hence, for any two distributions in the cover $H_i, H_j$ with domain $\ints$, the algorithm has to efficiently compute the mass assigned to the set $W_{ij} = \{ w \in W  | H_i(w) \geq H_j(w) \}$ by $H_i$ and $H_j$. In fact, even an approximate computation of these two values (in the sense of \cite{de2014learning}) is sufficient\footnote{In our setting, given an exponential family distribution $D = \calE_{\vec T}(\vec a)$, we should be able to compute (even approximately) the value $D(x)$ given a query $x \in \ints$. We will require access to approximate evaluation oracles; this is natural since we may use some approximation method to estimate the partition function $Z$ and then output the value $D(x) = \exp(- \vec a \cdot \vec T(x))/\wt{Z}$, where $\wt{Z}$ is the estimation of the partition function.}. Essentially the $\textsc{SelectHypothesis}^X$ routine requires
estimates to the probabilities $H_i(W_{ij})$ and $H_j(W_{ij})$ (see the routine $\textsc{Estimate}$ of Claim 24 in \cite{de2014learning}, that estimates the probability $p_i = H_i(W_{ij})$). Such estimates can be obtained using sample access to $H_k$ and access to evaluation oracles $\mathrm{EVAL}_{H_k}$ (even approximate evaluation oracles with multiplicative accuracy) for $k=1,2$. In our case, this can be done using an evaluation oracle: We assume that, given the set $\calA$, our algorithms have access to an $\calE_{\vec T}(\calA_{\varrho})$ evaluation oracle, that with a parameter $\vec a \in \calA_{\varrho}$ as an input and query $w \in \ints$, it outputs a 
the probability mass assigned to $x$ (even with some multiplicative error as in \cite{de2014learning}) by the distribution $\calE_{\vec T}(\vec a)$, where $\calA_{\varrho} = \varrho\-\conehull \calA$. Moreover, we assume sample access to the distributions in $\calE_{\vec T}(\calA_{\varrho})$
and 
to a discretized Gaussian, in the sense that, given $(\mu, \sigma^2)$, the learning algorithm can perform independent draws from the distribution $\calZ(\mu, \sigma^2).$  

Specifically, we assume the following. The value of the approximation tolerance $\beta$ of the oracle will be related with the learning accuracy $\eps$ and this relation is provided in \cite{de2014learning} (see Assumption \ref{assumption:oracle-access}).

\begin{assumption}
\label{assumption:oracle-access}
We assume that the learning algorithm can (i) query a sample oracle with input $(\mu, \sigma^2)$ and draw a sample from $\calZ(\mu, \sigma^2)$, (ii) query a sample oracle with input $\vec a \in \varrho\-\conehull \calA$ and draw a sample from $\calE_{\vec T}(\vec a)$ and (iii) query a $\beta$-approximate evaluation oracle $\mathrm{EVAL}_{D}(\beta)$ for $D = \calE_{\vec T}(\vec a)$ with $\vec a$ as in (ii) with input $x \in \ints$ and obtain a value $p_x$ with $D(x)/(1+\beta) \leq p_x \leq (1+\beta) D(x)$ for some $\beta > 0$. Moreover, given learning accuracy $\eps \in (0,1)$, we assume that $(1+\beta)^2 \leq 1 + \eps/8$.
\end{assumption}

We will replace Proposition \ref{proposition:learning-hypothesis-selection} with the following statement. In a similar fashion, we can obtain the analogue of Proposition \ref{proposition:tournament} (see Proposition 6 in \cite{de2014learning}). 

\begin{proposition}
[Lemma 22 in \cite{de2014learning}]
\label{proposition:hypothesis-de}
Assume that $X,H_1, H_2$ are distributions over $W \subseteq \ints$.\footnote{\cite{de2014learning} provide this result in the context of distributions with finite support. However, the sample and time complexity bounds they derive do not depend on the size of the support. What is in fact dependent on the size of the support is the bit complexity of the algorithms (since a sample could require, in principle, an arbitrarily large representation). In our case, we do not account for bit complexity (in fact, we focus on sample complexity). One could either think that the bit complexity is a random variable, which will, in practice take only a small number of possible values, due to the concentration properties of the distributions we examine or impose a ``hard bound'' on the number of bits the algorithm reads for each sample.}
Let $\eps,\delta \in (0,1)$.
There exists an algorithm
$\textsc{SelectHypothesis}^X(H_1, H_2, \eps, \delta)$, which is given sample access to 
$X$ and (i) to independent samples from $H_i$ and (ii) to a $\beta$-approximate evaluation oracle $\mathrm{EVAL}_{H_i}(\beta)$ for $i \in \{1,2\}$,
an accuracy parameter $\eps$ and a confidence parameter $\delta > 0$ and has the following behavior: It draws
\[
m = 
O(\log(1/\delta)/\eps^2)
\]
samples from each of $X, H_1$ and $H_2$, it performs $O(m)$ calls to the oracles $\mathrm{EVAL}_{H_i}(\beta)$ for $i \in \{1,2\}$, it performs $O(m)$ arithmetic operations and if some $H_i$ has $\tv(X, H_i) \leq \eps$, then, with probability $1-\delta$, it outputs an index $i^\star \in \{1,2\}$ that satisfies $\tv(X, H_{i^\star}) \leq 6\eps$.
\end{proposition}

\paragraph{Cover Construction.} Additionally, our algorithm has to construct the cover for $\calE_{\vec T}(\calA)$-sums. For both the sparse and the dense case, given the parameters $\eps,\varrho,\descriptivityparam,B$, we can consider the set $\calA'$ of the parameters' space where
$\calA' = \varrho\-\conehull\calA \cap \{\vec a: \|\vec a\|\le \rcrit\}$ for some sufficiently large $\rcrit \le (\varrho+\frac{1}{\descriptivityparam})\cdot\ln(1/\eps)+\frac{1}{2\descriptivityparam}\cdot\ln(B)+O(\varrho+\frac{1}{\descriptivityparam})$. Using Proposition \ref{proposition:euclidean-cover}, we can obtain a discretization for this set in time $T_{\mathrm{c}}$.
Let us define the total construction time of the cover of Theorem \ref{theorem:covering-siiervs} as
\begin{equation}
\label{eq:time}
T_{\mathrm{c}}^{\mathrm{total}} = T_{\mathrm{c}}(\calA, n, \eps/\ncrit, \vec T, \varrho,\descriptivityparam,B) +  T_{\mathrm{c}}(\calA, n, \eps/m, \vec T, \varrho,\descriptivityparam,B)\,. 
\end{equation}
The first term corresponds to the sparse cover construction $(\ncrit = \poly(B,L,1/\gamma)/\eps^2)$ and the second for the dense one $(m \leq n' \cdot \sqrt{B}/\gamma)$.

\begin{remark}
[On the runtime]
\label{remark:runtime}
    Note that when $\calA$ is not a very complicated set, then, for any $r>0$, the set $\calAvarrho\cap \ball_r[0]$ can be $O(\eps)$-covered in Euclidean distance in time polynomial to the size of the cover. Therefore, in such cases, $T_c$ can be omitted from the execution time of our learning algorithm, since it is dominated by the remaining terms.
\end{remark}

\subsubsection{The Result and the Algorithm}
The main learning result is stated in Theorem \ref{theorem:learning-SIIERVs}.

\begin{theorem}
[Weakly Proper Learner for SIIERVs]
\label{theorem:learning-SIIERVs}
Let $k \in \nats$ and consider the exponential family $\calE_{\vec T}(\calA)$ with $\calA \subseteq \reals^k$.
Let $n \in \nats$, accuracy $\eps \in (0,1)$, confidence $\delta \in (0,1)$ and let $X$ be an
unknown $\calE_{\vec T}(\calA)$-sum over $\ints$ of order $n$. 
Assume that Assumption \ref{assumption:proper} holds with parameters $\varrho, B,\gamma, \Lambda, \descriptivityparam$ and that Assumption \ref{assumption:oracle-access} holds.
There exists an algorithm  $\textsc{ProperLearner}^X$ (see Figure \ref{algorithm:learning-SIIERVs}) with the following properties: Given $n, \eps, \delta$, the algorithm uses
    \[
    m = {O} \left (\frac{1}{\eps^2} \log(1/\delta) \right) + 
    k \cdot \wt{O}(1/\eps^2) \cdot \poly \left(B, L, \frac{1}{\gamma} \right) \cdot \log \left(\frac{\varrho \sqrt{\Lambda}}{\descriptivityparam}
\right)
    \]
    samples from $X$ and, in time 
    \[
    \poly \left(m, \left( \frac{\varrho \cdot \sqrt{\Lambda}}{\descriptivityparam} \right)^{k\cdot \poly\left(B, \frac{1}{\eps}, \frac{1}{\gamma} \right)},  \left( n^2\cdot \poly\left(B, \frac{1}{\gamma} \right) \cdot O\left( \frac{\varrho\cdot\sqrt{\Lambda}}{\descriptivityparam\cdot \eps} \right) \right)^k, T_{\mathrm{c}}^{\mathrm{total}} 
    \right)\,,
    \]
    where $T_{\mathrm{c}}^{\mathrm{total}}$ is given in \eqref{eq:time},
    outputs a (succint description of a) distribution $\wt{X}$ which satisfies
    $\tv(X, \wt{X}) \leq \eps$, with probability at least $1-\delta$
    and, moreover, $\wt{X}$ is an $\calE_{\vec T}(\calA')$-sum of order at most 
    $(\sqrt{B}/\gamma) \cdot n$ and $\calA' = \varrho\-\conehull \calA$.
\end{theorem}

The algorithm follows. As in the SIIURV case, there are two regimes resembling to the input sum $X$ having small (i.e., sparse) or large (i.e., dense) variance and a final hypothesis testing routine. We now shortly depict the two learning sub-routines, corresponding to the small and large variance cases.  
When $X$ is close to a sparse form, the learning algorithm runs a tournament between all possible distributions of the sparse case and chooses the hypothesis that won each pairwise competition, i.e., the tournament's winner .
The dense proper case is more challenging: Our ultimate goal is to learn the dense form hypothesis that is close to $X$ with a sample complexity that does not depend on $n$. Crucially, we have to make use of the structure of the cover. In the dense regime, the input sum $X$ is close to a discretized Gaussian random variable and its parameters can be estimated using $O(1/\eps^2)$ samples. Having this approximation for $X$, we run the tournament hypothesis testing procedure between \emph{the estimated Gaussian} and the distributions of the dense form. Hence, we draw no more samples from $X$, but instead we generate draws from the Gaussian. By a union bound on the two events (i.e., the Gaussian is close to $X$ and that the winner of the tournament is close to the Gaussian), we get a dense form that is close to $X$. In the following, we may omit the ''weakly proper'' phrasing and simply use the term ''proper''.
\begin{figure}[ht]
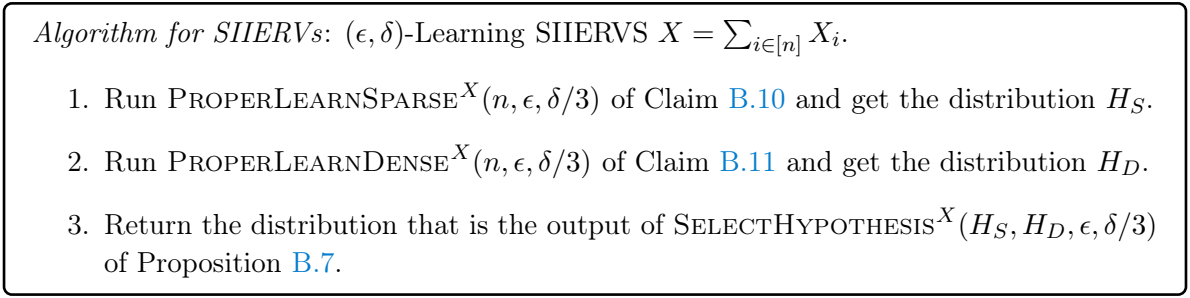

    \centering
\begin{mdframed}[
    linecolor=black,
    linewidth=1pt,
    roundcorner=3pt,
    backgroundcolor=white,
    userdefinedwidth=\textwidth,
]
\textit{Algorithm for SIIERVs}:  $(\eps,\delta)$-Learning SIIERVS $X = \sum_{i \in [n]} X_i$. 
\begin{enumerate}
    \item[\text{1.}] Run $\textsc{ProperLearnSparse}^X(n, \eps, \delta/3)$ of Claim \ref{claim:proper-learning-sparse-SIIERVs} and get the distribution $H_S$.
    
    \item[\text{2.}] Run  $\textsc{ProperLearnDense}^X(n, \eps, \delta/3)$ of Claim \ref{claim:proper-learning-dense-SIIERVs} and get the distribution $H_D$.
    
    \item[\text{3.}] Return the distribution that is the output of $\textsc{SelectHypothesis}^X(H_S, H_D, \eps, \delta/3)$ of Proposition \ref{proposition:hypothesis-de}.
\end{enumerate}
\end{mdframed}
    \caption{Proper Learning algorithm for SIIERVs.}
    \label{algorithm:learning-SIIERVs}
\end{figure}

After describing these fundamental procedures, we are now ready to provide a complete proof of our main learning result (see Theorem \ref{theorem:learning-SIIERVs}).
\paragraph{The Proof of Theorem \ref{theorem:learning-SIIERVs}.}
The analysis of Theorem \ref{theorem:learning-SIIERVs} works as follows:
Under Assumption \ref{assumption:proper}, we analyze our proper learner.     As we have already mentioned, the algorithm (Figure \ref{algorithm:learning-SIIERVs}) 
runs as follows: First, it calls the 
$\textsc{ProperLearnSparse}^X$ of Claim \ref{claim:proper-learning-sparse-SIIERVs} with input $(n, \eps, \delta/3)$ and gets the distribution $H_S$.
Then, it runs $\textsc{ProperLearnDense}^X$ of Claim \ref{claim:proper-learning-dense-SIIERVs} with input $(n, \eps, \delta/3)$ and gets a distribution $H_D$, which lies inside the desired distribution class. In order to conclude the proper learning part, via Proposition \ref{proposition:hypothesis-de}, it runs the procedure
$\textsc{SelectHypothesis}^X(H_S, H_D)$ with parameter $\eps, \delta/3$.
In conclusion, with probability at least $1-\delta$, 
the algorithm $\textsc{ProperLearner}^X$ will satisfy the desiderata of
Theorem \ref{theorem:learning-SIIERVs}. We divide the proof in a series of claims. 
Before presenting the proof, we remind the reader that
the following hold for the $(B^{1/2}/\gamma, \varrho\-\conehull)$-proper $\eps$-cover $\calD_{\mathrm{E}} = \calD_{\mathrm{E}}(\eps) = \calD_{\mathrm{E}}^{(s)}\cup \calD_{\mathrm{E}}^{(d)}$ with:
\begin{equation}
\label{equation:cover-SIIERV}
|\calD_{\mathrm{E}}^{(s)}|  + 
|\calD_{\mathrm{E}}^{(d)}| =  \left( \frac{\varrho \cdot \sqrt{\Lambda}}{\descriptivityparam} \right)^{k\cdot \wt{O}\left(\frac{1}{\eps^2}\right)\cdot \poly \left(B, L, \frac{1}{\gamma} \right)} + \left( n^2\cdot \poly\left(B, \frac{1}{\gamma} \right) \cdot O\left( \frac{\varrho\cdot\sqrt{\Lambda}}{\descriptivityparam\cdot \eps} \right) \right)^k\,.
\end{equation}

\begin{claim}
[Proper Learning of Sparse Instances]
\label{claim:proper-learning-sparse-SIIERVs}
Assume that Assumption \ref{assumption:proper} and Assumption \ref{assumption:oracle-access} hold.
For any $n, \eps, \delta > 0$, there is an algorithm $\textsc{ProperLearnSparse}^X(n, \eps, \delta)$ that given
\[
m = O \left (\frac{1}{\eps^2} \log(1/\delta) \right) + 
O \left( k \cdot \wt{O}\left(\frac{1}{\eps^2}\right)\cdot \poly \left(B, L, \frac{1}{\gamma} \right) \cdot \log \left(\frac{\varrho \cdot \sqrt{\Lambda}}{\descriptivityparam}\right)
\right)
\]
samples from the target $\calE_{\vec T}(\calA)$-sum $X$ over $\ints$ of order $n$,
outputs a (succint description of a) hypothesis distribution $H_S$ with the following guarantee: If $X$ is $\eps$-close to some element of
$\calCpr^{(s)}(\eps)$ (of Equation \eqref{equation:cover-SIIERV}), 
then it holds that $\tv(H_S, X) \leq c_1 \cdot \eps$, for some universal constant $c_1 \geq 1$, with probability at least $1-\delta$.
Moreover, if $\calCpr^{(s)}(\eps)$ is the sparse subset of the weakly proper cover $\calCpr(\eps)$, then
$H_S$ lies in $\calCpr^{(s)}(\eps)$ and the running time of the algorithm is
\[\poly\left(m, 2^{k \cdot \poly(B, 1/\eps, 1/\gamma) \cdot \log(\varrho \sqrt{\Lambda}/\descriptivityparam) }, T_{\mathrm{c}}^{\mathrm{sparse}} \right)\,.
\] 
In particular,
\begin{enumerate}
    \item $H_S$ will be an $\calE_{\vec T}(\varrho\-\conehull \calA)$-sum of order $\frac{1}{\eps^2} \cdot \poly(B,L,1/\gamma)$ and,
    \item $T_{\mathrm{c}}^{\mathrm{sparse}} = T_{\mathrm{c}}(\calA, n, \eps^3/\poly(B,L,1/\gamma), \vec T, \varrho,\descriptivityparam,B)$.
\end{enumerate}
\end{claim}

\begin{proof}
Let $\eps \in (0,1)$. Under Assumption \ref{assumption:proper}, according to our structural result, there exists a sparse form cover $\calCpr^{(s)} = \calCpr^{(s)}(\eps)$ of radius $\eps$ whose size is equal to 
\[
  M = \left  |\calD_{\mathrm{E}}^{(s)} \right| \leq  \left( \frac{\varrho \cdot \sqrt{\Lambda}}{\descriptivityparam} \right)^{k\cdot \wt{O}(1/\eps^2)\cdot \poly\left(B, L, \frac{1}{\gamma} \right)}\,,
\]
and each $\calE_{\vec T}(\calA)$-SIIRV of order $\ncrit$ (the sparse SIIERVs) can be $\eps$-approximated by some distribution $\calE_{\vec T}(\varrho\-\conehull\calA)$-SIIRV of order $\ncrit$. The algorithm has to construct the cover in time $T_{\mathrm{c}}^{\mathrm{sparse}}$ (with accuracy $\eps/\ncrit$ since we are in the sparse regime; this is indicated by the proof of the sparse case).
Let us assume that
$X$ is $\eps$-close to a sparse form element
in the cover $\calCpr(\eps)$.
Using Assumption \ref{assumption:oracle-access}, we can apply the $\textsc{SelectTournament}^X$ algorithm of \cite{de2014learning} (see Proposition 6 which is a variant of Proposition \ref{proposition:tournament}) with 
input 
the distributions' collection  $\calCpr^{(s)}$
with accuracy $\eps$ and confidence $\delta$.
We observe that there exists a distribution from the collection that is $\eps$-close in total variation distance and, so, the requirements are satisfied. Moreover, we assume that we have sample oracle access and evaluation oracle access to any distribution in $\calD_{\mathrm{E}}^{(s)}$. We can apply the variant of Proposition \ref{proposition:tournament}: The algorithm makes $O\left(\frac{1}{\eps^2} \left( \log(M) +  \log(1/\delta)\right)\right)$ draws from $X$ and from each $Y$ that is 
in $\calCpr^{(s)}$, runs in time
polynomial in the number of samples and in the size of the collection and, with probability at least $1-\delta$, outputs an index $i^\star \in [M]$ so that the sum $Y^\star$ with the corresponding parameters satisfies $\tv(X, Y^\star) \leq 6\eps$. We set $c_1 = 6$.
Moreover, in order to obtain a detailed expression for the sample complexity, we have that
\[
m = O\left(\frac{1}{\eps^2} \left( \log (M) +  \log(1/\delta) \right) \right)\,.
\]

The running time is $\poly\left(m, M, T_{\mathrm{c}}^{\mathrm{sparse}} \right).$
The result follows by replacing $M$.
\end{proof}
As a next step, we analyze the learning phase concerning the dense instances: 

\begin{claim}
[Proper Learning of Dense Instances]
\label{claim:proper-learning-dense-SIIERVs}
Under Assumption \ref{assumption:proper} and Assumption \ref{assumption:oracle-access},
for any $n, \eps, \delta > 0$, there is an algorithm $\textsc{ProperLearnDense}^X(n, \eps, \delta)$ that given
\[
m = O\left (\frac{1}{\eps^2}\log(1/\delta) \right)
\]
samples from the target $\calE_{\vec T}(\calA)$-sum $X$ over $\ints$ of order $n$,
outputs a (succint description of a) hypothesis distribution $H_D$ with the following guarantee: If $X$ is $\eps$-close to a dense form $\calE_{\vec T}(\calA)$-sum
in the $\calCpr(\eps)$
of Equation \eqref{equation:cover-SIIERV}
, then it holds that $\tv(H_D, X) \leq c_1 \cdot \eps$,
for some absolute constant $c_1 \geq 1$,
with probability at least $1-\delta$.
Moreover, if $\calCpr^{(d)}(\eps)$ is the dense subset in the proper cover $\calCpr(\eps)$, then  
$H_D$ lies in $\calCpr^{(d)}(\eps)$ and the running time of the algorithm is
\[
\poly \left(m, \left( n^2\cdot \poly\left(B, \frac{1}{\gamma} \right) \cdot O\left( \frac{\varrho\cdot\sqrt{\Lambda}}{\descriptivityparam\cdot \eps} \right) \right)^k, T_{\mathrm{c}}^{\mathrm{dense}} \right)\,.
\]
In particular,
\begin{enumerate}
    \item $H_D$ will be an $\calE_{\vec T}(\varrho\-\conehull \calA)$-sum of order $(\sqrt{B}/\gamma)\cdot n$ and,
    \item $T_{\mathrm{c}}^{\mathrm{dense}} = T_{\mathrm{c}}(\calA, n, \eps \cdot \gamma/(n \cdot \sqrt{B}), \vec T, \varrho,\descriptivityparam,B).$
\end{enumerate}

\end{claim}
\begin{proof}
Let us assume that $X$ is $\eps$-close to a dense form 
$\calE_{\vec T}(\calA)$-sum
in the $\calCpr(\eps)$.
We apply Claim \ref{claim:learning-gaussian-SIIURVs} to the target sum with some accuracy $\eps'$ and confidence $\delta/2$. Hence, with high probability, after drawing $O(\log(1/\delta)/\eps^2)$ samples from $X$, we get parameters $\wh{\mu}$ and $\wh{\sigma}^2$ so that the random variable $Z \sim \Dgauss(\wh{\mu}, \wh{\sigma}^2)$ satisfies
\[
\tv(X, Z) \leq \eps \,.
\]
The accuracy $\eps'$ is chosen so that the resulting total variation gap is $\eps$.
Our assumption about $X$ implies that 
there exists a distribution $Y$ is a sum of i.i.d. random variables with common parameter vector $\vec b$ and distribution $\calE_{\vec T}(\vec b)$, then $X$ is close to $Y$. So, there exists a cover $\calCpr^{(d)} = \calCpr^{(d)}(\eps)$ of radius $\eps$ for dense instances of size 
\[
\left|\calCpr^{(d)}\right| =  \left( n^2\cdot \poly\left(B, \frac{1}{\gamma} \right) \cdot O\left( \frac{\varrho\cdot\sqrt{\Lambda}}{\descriptivityparam\cdot \eps} \right) \right)^{k}\,.
\]
The algorithm constructs the cover in time $T_{\mathrm{c}}^{\mathrm{dense}}$ (with accuracy $\eps/(n \sqrt{B}/\gamma)$ as indicated by the proof of the dense case).
By the structure of the cover, there exists a distribution $Y$ so that
\begin{enumerate}
    \item $Y$ lies in $\calCpr^{(d)}$,
    \item and $\tv(X,Y) \leq \eps$.
\end{enumerate}
Hence, by the triangle inequality, we have that $\tv(Y, Z) \leq 2\eps$.
We then apply the algorithm $\textsc{SelectTournament}^X$ (we use the modification of Proposition \ref{proposition:tournament} (see \cite{de2014learning} with hypothesis selection algorithm as in Proposition \ref{proposition:hypothesis-de}) with the following input:
\begin{enumerate}
    \item Let the target $D$ be the distribution of $Z$ (i.e., the discretized Gaussian),
    \item consider the collection of distributions corresponding to the set $\calCpr^{(d)}$,
    \item and accuracy $2\eps$ and confidence $\delta/2$.  
\end{enumerate}
Note that there exists a distribution in the provided collection that is $2\eps$-close to the target distribution and, using Assumption \ref{assumption:oracle-access}, we have the required sample and evaluation oracle access.
The $\textsc{SelectTournament}^X$ procedure makes $O\left(\frac{1}{\eps^2} \left( \log\left|\calCpr^{(d)}\right| + \log(1/\delta)\right)\right)$ draws from the target $Z$ and from each distribution in the collection. This is possible and it only costs \emph{in runtime}. Recall that
we have assumed samples access to a discretized Gaussian oracle and sample and evaluation oracle access to the elements of the cover. Moreover, it runs in time polynomial in the number of samples and in the size of the collection (i.e., the runtime depends on $n$) and, with probability at least $1-\delta/2$, outputs an index $i^\star \in \left[\left|\calCpr^{(d)}\right|\right]$ so that the sum $Y^\star$ with the corresponding parameters satisfies 
\[
\tv(Z, Y^{\star}) \leq 12 \eps\,.
\]
Hence, it holds that $\tv(X, Y^\star) \leq 13\eps$. Let $c_1 = 13$. Applying union bound, we have that, with probability at least $1-\delta,$ the algorithm $\textsc{ProperLearnDense}^X$ will use
\[
m = O(\log(1/\delta)/\eps^2)
\]
samples (from the target $X$) and, in time $\poly\left(m, 1/\eps, \log(1/\delta), \left|\calCpr^{(d)}\right|\right)$, it will output a distribution $H_D$ so that
\begin{enumerate}
    \item $H_D$ lies inside $\calCpr^{(d)}$,
    \item and $\tv(X,H_D) \leq 13\eps$.
\end{enumerate}
The result follows.
\end{proof}

By combining  Claim \ref{claim:proper-learning-sparse-SIIERVs} and Claim \ref{claim:proper-learning-dense-SIIERVs} with the guarantees of Proposition \ref{proposition:hypothesis-de}, Theorem \ref{theorem:learning-SIIERVs} follows.

\section{The Proof of Theorem \ref{theorem:structural-SIIURV} (Structural Result for SIIURVs)}
\label{appendix:proof-siiurv}
\begin{proof}
[Proof of Theorem \ref{theorem:structural-SIIURV}]
Let us consider the SIIURV $X = \sum_{i \in [n']} X_i$ for some $n' \leq n$ where the distribution of each $X_i$ satisfies Assumption \ref{assumption:SIIURV}.
There exists a critical threshold value $\ncrit$, to be decided, that indicates whether $X$ belongs to the sparse or to the dense form. Let us first consider the case where $n' \geq \ncrit$.\\
\noindent\texttt{Dense Case.} In this case, we will approximate $X$ with a suitable discretized Gaussian random variable.
Let $\mu = \E[X]=\sum_{i\in[n']} \mu_i,$ where $\mu_i = \E[X_i]$ and $\sigma^2 = \Var(X)= \sum_{i\in[n']} \sigma_i^2$, where $\sigma_i^2 = \Var(X_i)$ and consider some random variable $\dgaussvar_{X}$ with $\dgaussvar_{X} \sim \Dgauss(\mu,\sigma^2)$. 
Moreover, we set
$\beta = \sum_{i\in[n']} \beta_i$ where $\beta_i = \sum_{i\in[n']}\E \left[|X_i-\E[X_i]|^3\right]$ and consider $\delta \in [0,1]$ to be such that
    \[
        \delta = \max_{i\in[n']} \tv(X-X_i, X-X_i+1)\,.
    \]
If we apply the Gaussian approximation lemma (see Lemma \ref{lemma:discr-gaussian-approx}), we get that
\[
    \tv(X,\dgaussvar_{X}) \le  O(1/\sigma)+O(\delta)+O(\beta/\sigma^3)+O(\delta\beta/\sigma^2)\,.
\]
Our goal is to control the right-hand side of this inequality. In fact, it is reasonable to 
upper bound the ratio between the sum of third centered moments to the variance, to lower bound the variance of $X$ and to upper bound $\delta$. 
In what follows, we insist on these three desiderata.
\begin{claim}
[Variance Lower Bound] It holds that $\Var(X) \geq n' \cdot \gamma/4$.
\end{claim}
\begin{proof}
Let us focus on a particular $X_i$ in the sum that satisfies Assumption \ref{assumption:SIIURV}. Let $M$ be a mode of the unimodal random variable $X_i$. We have that
$\Var(X_i) \geq \frac{1}{4} \sum_{x \in \ints} \Pr[x] |x - M|^2 \geq \frac{1}{4} \min_{x \neq M} |x - M|^2 \sum_{x \neq M} \Pr[x] = \Theta(\gamma)$, since we can sum over $x \neq M$ and this sum has mass at least $\gamma$, also $\min_{x \neq M}|x-M|^2 = 1$. Since the random variables $X_i$ are independent, the SIIRV $X$ has variance at least $\sigma^2 = \Omega(n' \cdot \gamma)$.
\end{proof}
\begin{claim}
[Third Centered Moment - Variance Ratio]
It holds that $\beta/\sigma^2 = O \left( \frac{B}{\gamma} \right).$
\end{claim}
\begin{proof}
We have that
\[
\frac{\beta}{\sigma^2} = \frac{\sum_{i \in [n']} \beta_i}{\sum_{i \in [n']} \sigma_i^2}
\]
Note that each term in the above ratio is non-negative and so we can apply Lemma \ref{lemma:ratio-ineq} in order to obtain
\[
\frac{\beta}{\sigma^2} \leq \max_{i \in [n']} \frac{\beta_i}{\sigma_i^2}
\]
Using the proof of the previous claim, we have that $\sigma_i^2 = \Omega(\gamma)$. Moreover, the fourth centered moment is upper bounded by $B$ and so
\[
\frac{\beta}{\sigma^2} = O \left( \frac{B}{\gamma} \right)\,.
\]
\end{proof}

\begin{claim}
[TV Shift] It holds that $\delta =  O \left( \frac{1}{\sqrt{1 + (n'-1) \cdot (1-\gamma)}}\right)$.
\end{claim}
\begin{proof}
For a single term $X_i$, it holds that
\[
\tv(X_i, X_i+1) = \frac{1}{2} \cdot \sum_{x \in \ints} | \Pr[X_i = x] - \Pr[X_i = x-1] | 
\]
Let $M$ be a mode of $X_i$. Since $X_i$ is unimodal, we get a telescopic sum and $\tv(X_i, X_i + 1) = \Pr[X_i = M].$ Hence, we get that the TV shift is at most $1-\gamma$. We now apply Lemma \ref{lemma:dshift-sum} and get
\[
\tv(X - X_i, X -X_i + 1) = \tv\left(\sum_{j \neq i} X_j, 1+\sum_{j \neq i} X_j\right) \leq \frac{\sqrt{2/\pi}}{\sqrt{\frac{1}{4} + \sum_{j \neq i} \tv(X_j, X_j + 1) } }\,.
\]
This implies that
\[
\tv(X - X_i, X -X_i + 1)  = O \left( \frac{1}{\sqrt{1 + (n'-1) \cdot (1-\gamma)}}\right)\,.
\]
Taking the supremum of $i \in [n']$, we get that
\[
\delta = O \left( \frac{1}{\sqrt{1 + (n'-1) \cdot (1-\gamma)}}\right)\,.
\]
\end{proof}

\begin{claim}
For $n' \geq \Omega \left(\frac{B^2}{\gamma^3 \eps^2} \right)$, we get that $\tv(X, \Dgauss_X) \leq \eps$.
\end{claim}
\begin{proof}
We require that $\frac{1}{\sigma} \cdot \beta/\sigma^2 \leq \epsilon$, which implies that $\frac{B}{\gamma \sqrt{n' \cdot \gamma}} \leq \eps$ and so $n' = \Omega(B^2/(\gamma^3 \eps^2)).$
Also, we require that $\delta B/\gamma \leq \eps$, which implies that $\sqrt{1 + (n'-1)(1-\gamma)} \geq \frac{B}{\gamma \eps}$. This is satisfied by the above choice of $n'$. Hence, we can choose $\ncrit = \Omega \left( \frac{B^2}{\gamma^3 \eps^2} \right)$.
\end{proof}

\noindent\texttt{Sparse Case.} Let us now focus on the case $n' \leq \frac{B^2}{\gamma^3 \eps^2}.$ For the term $X_i$ with mode $M$ ($M$ could be any mode of $X_i$, since $X_i$ might have many consequent modes and still be considered unimodal), we have that
\[
\E |X_i - M|^4 = O(B)\,,
\]
since $|\E[X_i]-M|^2\le 3\Var(X_i)$ whenever $X_i$ is unimodal (see \cite{johnson1951themomentproblem}).

It holds that $|x-M|^4 \Pr[X_i = x] \leq O(B)$ for any $x \in \ints$. Let us consider the points $x \in \ints$ so that $\Pr[X_i = x] \geq \frac{\eps}{|x-M|^{3.5}}$. It holds that
\[
O(B) \geq |x-M|^4 \Pr[X_i = x] \geq \eps \cdot \sqrt{|x-M|}
\]
and so these points lie in
\[
|x - M| \leq B^2/\eps^2
\]
We have that
\[
\sum_{x : |x - M| > B^2 / \eps^2} \Pr[X_i = x] \leq \sum_{x : |x - M| > B^2 / \eps^2} \frac{\eps}{|x-M|^{3.5}} = O(\eps)\,, 
\]
since $\sum_{x : |x - M| > B^2 / \eps^2} \frac{1}{|x-M|^{3.5}} \leq \sum_{i \geq 1} 1/i^2 = \pi^2/6$.
This implies that there exists a distribution supported on the bounded interval $[M - B^2/\eps^2, M+B^2/\eps^2]$ which is $(1-\eps)$ close in total variation to $X_i$. In order to get the desired result, we have to make $\eps = \wt{\eps}/n'$ and so the SIIRV $X$ will be $\wt{\eps}$ close in statistical distance to a discrete random variable $Y$ whose support is included within an interval of size at most $(n')^3 \cdot B^2/\wt{\eps}^2 = \poly(B/\gamma \wt{\eps})$ (due to convolution). Moreover, for each $X_i$, the mode takes some out of $L$ at most values (due to condition \eqref{cond:bounded-modes-SIIURV}) and therefore there are $L^{n'} = L^{\poly(B, 1/\gamma, 1/\wt{\eps})}$ possible choices for the interval that contains the support of $Y$ (since fixing the modes fixes the intervals corresponding to each term $X_i$).

For every such interval $\calI$, we know that it has size at most $s = \poly(B,1/\gamma, 1/\wt{\eps})$. Each point in the interval can be assigned by $Y$ a value within $[0,1]$. Therefore, if we quantize the possible values for each point in the interval $\calI$ into $s/\wt{\eps}$ equidistant levels, then we get $s^{s/\wt{\eps}} = 2^{\poly(B,1/\gamma,1/\wt{\eps})}$ possible distributions $Y'$, corresponding to $\calI$. We know that for some $\calI$ there exists some distribution $Y'$ that is $O(\wt{\eps})$ close to $Y$ (and hence to $X$) in total variation distance. The total number of possibilities is $L^{\poly(B,1/\gamma,1/\wt{\eps})}$.
\end{proof}

\section{Bounding the Parameter Space (Theorem \ref{theorem:projection})}\label{appendix:bound-param-space}

\subsection{The Proof of Theorem \ref{theorem:projection} (Bounding the Parameter Space)}
\label{proof:projection-step}
We restate the theorem we are going to prove for readers' convenience.
\begin{theorem*}
\emph{
Under assumptions \eqref{assumption:geometry}, \eqref{assumption:unimodal}, \eqref{assumption:bounded-modes} and \eqref{assumption:bounded-moments}, there exists some value $\descriptivityparam = \descriptivityparam(\calA, \vec T)>0$ depending on the geometric properties of $\calA$ and $\vec T$, 
such that for any $\eps\in(0,1)$ and any $\vec a\in \calA$, there exists some $\vec b\in \varrho\-\conehull \calA$ with $\|\vec b\| \le (\varrho+\frac{1}{\descriptivityparam}) \cdot \ln(1/\eps) + \frac{1}{2\descriptivityparam}\cdot \ln(B) + O(\varrho + \frac{1}{\descriptivityparam})$ such that
\[
    \tv(\calE_{\vec T}(\vec a),\calE_{\vec T}(\vec b)) \le \eps\,.
\]
}
\end{theorem*}
In order to show this result, we make use of Lemma \ref{lemma:descriptivity}.

\begin{proof}
Let $\vec a \in \varrho\-\conehull\calA$ with $\| \vec a \|_2 \ge \rcrit$ with $\rcrit\ge \varrho$ to be decided. Our goal is to provide a parameter vector $\vec b$ so that $\| \vec b \|_2 = \rcrit$ and $\tv(\calE_{\vec T}(\vec a), \calE_{\vec T}(\vec b)) = O(\eps).$ 

\noindent Let $W\sim\calE_{\vec T}(\vec a)$ and $W'\sim\calE_{\vec T}(\vec b)$ (at first, $\vec a$ and $\vec b$ are unspecified). We have that
\[
\tv(W,W') = \frac{1}{2} \sum_{x \in \ints} 
\left | 
\frac{\exp(- \vec a \cdot \vec T(x))}{\sum_{y \in \ints} \exp(- \vec a \cdot \vec T(y))}
-
\frac{\exp(- \vec b \cdot \vec T(x))}{\sum_{y \in \ints} \exp(- \vec b \cdot \vec T(y))}
\right |\,.
\]

Consider some mode $M_{\vec a}$ of $\calE_{\vec T}(\vec a)$ and some mode $M_{\vec b}$ of $\calE_{\vec T}(\vec b)$. 
Note that $\ZZ(\vec a),\ZZ(\vec b)\ge 1$. Then, we have that
\[
\tv(W,W') =
\frac{1}{2 \ZZ(\vec a) \ZZ(\vec b)} 
\sum_{x \in \ints} 
\left |
\sum_{y \in \ints} 
e^{- \vec a \cdot \TT(x) - \vec b \cdot \TT(y)} 
- 
e^{- \vec b \cdot \TT(x) - \vec a \cdot \TT(y)}
\right |\,.
\]
By moving the absolute value inside the sum over $y \in \ints$, the total variation distance is
\[
\tv(W,W') \le
\frac{1}{2 \ZZ(\vec a) \ZZ(\vec b)} 
\sum_{(x,y) \in \ints^2} 
\left |
e^{- \vec a \cdot \TT(x) - \vec b \cdot \TT(y)} 
- 
e^{- \vec b \cdot \TT(x) - \vec a \cdot \TT(y)} 
\right |\,.
\]

We will apply Lemma \ref{lemma:vanishing-deviation} with $\eta = 1/2$, $s=0$ and $k=1$ and get $\ell\le O(B^{1/2})$.
This motivates us to partition $\ints$ into two sets
\[
    Z_1 = \{x\in\ints: |x-M_{\vec a}|> \ell\} \text{ and }Z_2 = \ints\setminus Z_1\,.    
\]
Based on $Z_1,Z_2$, we can decompose $\ints^2$ into four sets: $N_1 = Z_1\times Z_1, N_2 = Z_1\times Z_2, N_3 = Z_2\times Z_1, N_4 = Z_2\times Z_2$.

Set $\Deltae := \frac{e^{- \vec a \cdot \TT(x) - \vec b \cdot \TT(y)} 
- 
e^{- \vec b \cdot \TT(x) - \vec a \cdot \TT(y)} }{ \ZZ(\vec a) \ZZ(\vec b)} $ and $S := \sum_{(x,y)\in\ints^2} |\Deltae|$. We have that
\[
    S = S_1 + S_2 + S_3 + S_4\,,
\]
where $S_i = \sum_{(x,y)\in N_i}|\Deltae|$ and observe that an upper bound on $S$ would control the total variation distance.

Let us choose $\vec b$. In what follows, we consider $\vec b$ to be the parameter vector given by Lemma \ref{lemma:descriptivity}, for $r=\rcrit$ and $\vec a\in\varrho\-\conehull\calA$ with $\|\vec a\| \ge \rcrit$. We also consider $M_{\vec a} = M_{\vec b}$. We next upper bound each term $S_i$ separately.

\noindent\underline{\texttt{Term $S_1$:}} For the term $S_1$, we use the fact that if $Q_\ell = \vec 1\{|W-M_{\vec a}|\le \ell\}$, we get
    \[
        \Pr_{\vec a}[|W-M_{\vec a}|>\ell] = \E[|W-M_{\vec a}|^0\cdot (1-Q_\ell)] \le e^{-\|\vec a\|/\varrho} \cdot O(1)\,,
    \]
    and similarly for $\Pr_{\vec b}[|W-M_{\vec b}|>\ell]$, since $\vec a$ and $\vec b$ belong to $\varrho\-\conehull\calA$ and $\ell$ is selected accordingly, as Lemma \ref{lemma:vanishing-deviation} suggests.
    
    Moreover, we have that $|\Deltae| \le \frac{e^{- \vec a \cdot \TT(x) - \vec b \cdot \TT(y)}
+ 
e^{- \vec b \cdot \TT(x) - \vec a \cdot \TT(y)} }{ \ZZ(\vec a) \ZZ(\vec b)} $ and therefore 
\begin{align*}
    S_1 \le &\ 2\cdot\Pr_{\vec a}[W\in Z_1]\cdot \Pr_{\vec b}[W\in Z_1] \\
    = &\ 2\cdot \Pr_{\vec a}[|W-M_{\vec a}|>\ell]\cdot \Pr_{\vec b}[|W-M_{\vec b}|>\ell] \\
    \le &\ e^{-2 \rcrit/\varrho}\cdot O(1)\,.
\end{align*}

\noindent\underline{\texttt{Terms $S_2, S_3$:}} For $S_2$ and $S_3$, we have for similar reasons that
\begin{align*}
    S_2, S_3 \le &\ \Pr_{\vec a}[W\in Z_1]\cdot \Pr_{\vec b}[W\in Z_2] + \Pr_{\vec b}[W\in Z_1]\cdot \Pr_{\vec a}[W\in Z_2] \\
    \le &\ \Pr_{\vec a}[|W-M_{\vec a}|>\ell] + \Pr_{\vec b}[|W-M_{\vec b}|>\ell] \\
    \le &\ e^{- \rcrit/\varrho}\cdot O(1)\,.
\end{align*}

\noindent\underline{\texttt{Term $S_4$:}} For the term $S_4$, we split $N_4$ to $N_4^{(1)}, N_4^{(2)}, N_4^{(3)}, N_4^{(4)}$ and form the four sums $S_4^{(1)},$ $S_4^{(2)},$ $S_4^{(3)},$ $S_4^{(4)}$ (which sum to $S_4$), similarly to how we split $\ints^2$ into $N_1,N_2,N_3,N_4$. In this case, we consider $Z'_1 = \{x\in Z_2 : \Pr_{\vec a}[W=x] \le e^{-\descriptivityparam\rcrit}\Pr_{\vec a}[W=M]\}$ and $Z'_2 = Z_2\setminus Z'_1$. 

We know that $|Z_2|\le 2\ell$ and therefore $\Pr_{\vec a}[W\in Z_1'] \le 2\ell\cdot e^{-\descriptivityparam\rcrit}$ and, due to the selection of $\vec b$ (according to Lemma \ref{lemma:descriptivity}), we also have that $\Pr_{\vec b}[W\in Z_1'] \le 2\ell\cdot e^{-\descriptivityparam\rcrit}$. Hence, with a similar reasoning as the one used for $S_1,S_2,S_3$ and since $\ell = O(B^{1/2})$ we have
\begin{align*}
    S_4^{(1)} \le &\ e^{-2\descriptivityparam\rcrit} \cdot O(B)\,, \\
    S_4^{(2)}, S_4^{(3)} \le &\ e^{-\descriptivityparam\rcrit} \cdot O(B^{1/2})\,.
\end{align*}
\noindent It remains to bound $S_4^{(4)}$. We have
\begin{align*}
    S_4^{(4)} = &\ \sum_{(x,y)\in Z_2'\times Z_2'} |\Deltae| \\
    = &\ \sum_{(x,y)\in N_4^{(4)}}
        \frac{
        \left| 
            \frac{\Pr_{\vec a}[W=x]}{\Pr_{\vec a}[W=M_{\vec a}]}\cdot \frac{\Pr_{\vec b}[W=y]}{\Pr_{\vec b}[W=M_{\vec b}]}
            -
            \frac{\Pr_{\vec b}[W=x]}{\Pr_{\vec b}[W=M_{\vec b}]}\cdot \frac{\Pr_{\vec a}[W=y]}{\Pr_{\vec a}[W=M_{\vec a}]}
        \right|
        }{
            e^{\vec a\cdot \vec T(M_{\vec a})} \cdot \ZZ(\vec a) \cdot e^{\vec b\cdot \vec T(M_{\vec b})} \cdot \ZZ(\vec b) 
        }
    = 0\,,
\end{align*}
due to the selection of $\vec b$ according to Lemma \ref{lemma:descriptivity}.
Therefore, in total, we pick 
\[
    \rcrit = \varrho \cdot \ln(1/\eps) + \frac{1}{2\descriptivityparam}\cdot \ln(B) + \frac{1}{\descriptivityparam} \cdot \ln(1/\eps) + O(\varrho + 1/\descriptivityparam)\,,
\]
and get that $\tv(W,W') \le \eps$.
\end{proof}

\subsection{The Proof of Lemma \ref{lemma:descriptivity} (Structural Distance \& Bounding Norms)}
\label{proof:sparsification-st}
In order to show Lemma \ref{lemma:descriptivity}, we will rely on the geometry induced by the exponential family distributions. Let us restate this result.
\begin{lemma*}
[Structural Distance \& Bounding Norms]
Under assumptions \eqref{assumption:geometry}, \eqref{assumption:unimodal} and \eqref{assumption:bounded-modes}, there exists some constant $\descriptivityparam>0$ such that for any $r\ge \varrho$ and any $\vec a \in \calA$ with $\|\vec a\|\ge r$, there exists some $\vec b\in \calA_{\varrho}$ and $\|\vec b\| = r$ so that
    $\dnew(\calE_{\vec T}(\vec a), \calE_{\vec T}(\vec b)) \leq e^{- \descriptivityparam \cdot r}$, i.e., for any $x \in \ints$, at least one of the following should hold:
    \begin{itemize}
        \item Either $x$ satisfies $\Pr_{\vec a}(x) \leq e^{- \descriptivityparam \cdot r} \cdot \Pr_{\vec a}(M_{\vec a})$ and $\Pr_{\vec b}(x) \leq e^{- \descriptivityparam \cdot r} \cdot \Pr_{\vec b}(M_{\vec b})$, 
        \item or $x$ satisfies $\Pr_{\vec a}(x)/\Pr_{\vec a}(M_{\vec a}) = \Pr_{\vec b}(x)/\Pr_{\vec b}(M_{\vec b})$.
    \end{itemize}
\end{lemma*}

\begin{proof}
[Proof of Lemma \ref{lemma:descriptivity}]
We decompose the proof into a number of steps.
\paragraph{Alternative form of Lemma \ref{lemma:descriptivity}.}
We can formulate a geometric framework through the observation that $\Pr_{\vec a}[W=x]\propto \exp(-\vec a\cdot\vec T(x))$ for any $\vec a\in\calAvarrho$. In particular, we have that
\begin{equation}\label{equation:geometry-basis}
    \Pr_{\vec a}[W=x] \ge \Pr_{\vec a}[W=y] \text{ is equivalent with }\vec a\cdot (\vec T(y) - \vec T(x)) \ge 0\,.
\end{equation}
Using relation \eqref{equation:geometry-basis}, we arrive to the following equivalent formulation for Lemma \ref{lemma:descriptivity}. In particular, the structural distance states that there exists some $\descriptivityparam>0$ such that for any $r\ge \varrho$ ($\varrho$ is defined in Assumption \ref{assumption:proper}) and any $\vec a\in\calA$ with $\|\vec a\|\ge r$ there exists some $\vec b\in\calAvarrho$ (recall that for $\varrho>0$, $\calAvarrho = \varrho\-\conehull\calA$, i.e., the superset of $\calA$ that also contains every vector in the conical hull of $\calA$ that has norm at least $\varrho$) with $\|\vec b\| = r$ such that $\ints = \calX_1\cup\calX_2$ where $\calX_1$ and $\calX_2$ are defined as follows
\begin{enumerate}
    \item\label{cases:geometric-descriptivity-insignificant} $\calX_1\subseteq\ints$ so that for any $x\in\calX_1$ we have 
    \[
        \vec a\cdot(\vec T(x)-\vec T(M_{\vec a})) \ge \descriptivityparam r\text{ and }\vec b\cdot(\vec T(x)-\vec T(M_{\vec b})) \ge \descriptivityparam r\,,
    \]
    \item\label{cases:geometric-descriptivity-equivalence} $\calX_2\subseteq \ints$ so that for any $x\in\calX_2$ we have 
    \[
        \vec a\cdot(\vec T(x)-\vec T(M_{\vec a})) = \vec b\cdot(\vec T(x)-\vec T(M_{\vec b}))\,,
    \]
\end{enumerate} 
where $M_{\vec a}$ (resp. $M_{\vec b}$) is any mode of $\calE_{\vec T}(\vec a)$ (resp. $\calE_{\vec T}(\vec b)$).

Our goal is to select the parameter $\theta>0$ appropriately so that for any given $\vec a\in\calA$, we can find $\vec b\in\calAvarrho$ with $\|\vec b\| = r$ such that any $x\in\ints$ either belongs in $\calX_1$ or $\calX_2$.

\paragraph{Step 1.} First, note that any mode (global maximum point of the probability mass function) of the distribution $\calE_{\vec T}(\vec a)$ cannot be in $\calX_1$ (since $\descriptivityparam, r>0$ and $\vec a\cdot\vec T(y) = \vec a\cdot \vec T(y')$ whenever $y,y'$ are modes). Therefore we get that $\calE_{\vec T}(\vec a)$ and $\calE_{\vec T}(\vec b)$ must have the same set of modes. We define the regions $\calR_M$ of the parameter vectors that correspond to distributions with $M$ as a mode. In particular, such regions are defined by the property that for any $\vec u\in\calR_M$ it holds that $\Pr_{\vec u}[W=M] \ge \Pr_{\vec u}[W=x]$, for any $x\in\ints$ (if $\calE_{\vec T}(\vec u)$ is well defined), or using relation \eqref{equation:geometry-basis}, more generally as follows
\begin{equation}\label{equation:mode-sets}
    \calR_M = \{\vec u\in\reals^k : \vec u\cdot (\vec T(x)-\vec T(M)) \ge 0, \text{ for any } x\in\ints\}\,.
\end{equation}
Note that the sets $\calR_M$ are convex cones that could be polyhedral cones in the case that a finite number of points $x\in\ints$ correspond to a set of restrictions that implies the remaining ones. We also define, for any $\calM\subseteq\ints$, intersections of such sets as follows:
\begin{equation}\label{equation:mode-sets-intersections}
    \calR_\calM = \bigcap_{M\in\calM}\calR_M\,. 
\end{equation}
\noindent For the demand that $\modes_{\vec a} = \modes_{\vec b}$ to be satisfied we must (at least) pick $\vec b$ so that
\begin{equation}\label{equation:restriction1}
    \vec b \in \calR_{\modes_{\vec a}}\,.
\end{equation}
In order to develop some intuition about the regions of the form $\calR_{\modes}$, one might consider $\modes=\{M,M'\}\subseteq\ints$. In this case
\begin{align*}
    \calR_\modes &\ = \calR_M\cap \calR_{M'} \\
    &\ = \{\vec u: \vec u\cdot(\vec T(x) - \vec T(M))\ge 0 \text{ and } \vec u\cdot(\vec T(x)-\vec T(M')) \ge 0, \text{ for any }x\in\ints\} \\
    &\ = \calR_M \cap \{\vec u: \vec u\cdot (\vec T(M)-\vec T(M')) = 0\}\,.
\end{align*}
Therefore, if $\vec T(M)\neq \vec T(M')$, then the dimension of $\calR_\modes$ is at most $k-1$ (and this can be generalized for larger sets $\modes$ by using the notion of affine independence). For any $M\in\ints$, the set $\calR_M$ is a countable intersection of halfspaces of the form $\calH = \{\vec u\in\reals^k: \vec u\cdot \vec h \ge 0\}$. If $M\in\modes\subseteq\ints$, then $\calR_\modes$ is a subset of the boundary of $\calR_M$, since for any $M'\in\modes$, the vector $\vec T(M')-\vec T(M)$ corresponds to some of the halfspaces that define $\calR_M$.

\paragraph{Step 2.} Our goal is to pick $\vec b$ so that any $x\in\ints$ lies in either $\calX_1$ or $\calX_2$. En route, we will use Assumption \ref{assumption:proper}. In this step we will get rid of $x\in\ints$ for which we get for free that $x\in\calX_1$, due to the fact that $\vec a,\vec b\in\calAvarrho$ anyway. We will use the following assumption.

Let $M_{\vec a}$ be any mode of $\calE_{\vec T}(\vec a)$. Then, we may define the sequence of vectors $(\vec v_x)_{x\in \ints}$ by $\vec v_x = \vec T(x) - \vec T(M_{\vec a})$ and reformulate $\calX_1 = \{x\in\ints: \vec a\cdot\vec v_x \ge \descriptivityparam r, \vec b\cdot \vec v_x \ge \descriptivityparam r\}$ as well as $\calX_2 = \{x\in\ints : \vec a\cdot \vec v_x = \vec b\cdot \vec v_x\}$\,. We are allowed to use $\vec v_x$ for both $\vec a$ and $\vec b$ in the definitions of $\calX_1$ and $\calX_2$, since, according to \textbf{Step 1}, vector $\vec b$ has to be selected within $\calR_{\modes_{\vec a}}$ anyway.

\def\calY{\mathcal{Y}}

We will first classify (to $\calX_1$) the points $x\in\ints$ for which the hyperplane defined by $\vec v_x$ does not correspond to any boundary of $\calR_{\modes_{\vec a}}$. That is to say, $\vec v_x\cdot\vec u>0$ for any 
$\vec u\in\calR_{\modes_{\vec a}}\cap\calAvarrho$ 
with $\vec u\neq 0$. In particular, we define for any $\modes\subseteq\modes_{\calA}$, the following set of points
\begin{equation*}\label{equation:geometry-non-boundary-points}
    \calY_{\modes} = \{x\in\ints: \vec u\cdot (\vec T(x)-\vec T(M)) > 0, \text{ for any }\vec u\in\calR_{\modes}\cap\calAvarrho 
    \text{ with } \vec u\neq 0 \text{ and }M\in\modes\}\,.
\end{equation*}
Our goal here will be to show that there exists some constant $\descriptivityparam_1>0$ such that for any $\vec u\in\calAvarrho$ and any $y\in\calY_{\modes_{\vec u}}$ we have that $\vec u\cdot (\vec T(y)-\vec T(M_{\vec u})) \ge \descriptivityparam_1 \|\vec u\|$.

To this end, observe, first, that due to assumption \eqref{assumption:bounded-modes}, the number of different possible $\modes\subseteq\modes_\calA$ must be finite. Therefore, if we show that for every fixed $\modes\subseteq\modes_{\calA}$ there exists some constant that satisfies the desired property for any $\vec u\in\calR_\modes\cap \calAvarrho$, then by taking the minimum over the selection of $\modes$, we can find the target $\descriptivityparam_1$ (swap of logical quantifiers).

\def\calJ{\mathcal{J}}

For a fixed $\modes\subseteq\modes_\calA$, we consider any vector $\vec u\in\calR_\modes\cap\calAvarrho$. Note that the only guarantee we have is that $\modes\subseteq\modes_{\vec u}\subseteq\modes_{\calA}$. Let $x_1,x_2\in\modes_{\vec u}$ be the smallest and largest elements of $\modes_{\vec u}$, respectively (i.e., $x_1\le x\le x_2$  for any $x\in\modes_{\vec u}$). Note that $\calY_{\modes} \cap \modes_{\vec u} = \emptyset$ by construction since
$\vec u\in\calR_\modes\cap\calAvarrho$ 
, $\vec u\neq 0$ and therefore $x_1,x_2\not\in\calY_\modes$. Consider $y_1,y_2\in\calY_{\modes}$ with $y_1\le x_1$ and largest possible and $y_2\ge x_2$ and smallest possible. Then, due to unimodality (assumption \eqref{assumption:unimodal}), we have that $\vec u\cdot \vec T(y) \ge \min\{\vec u\cdot \vec T(y_1) , \vec u\cdot \vec T(y_2)\}$ for any $y\in\calY_\modes$. Since $\modes_{\vec u}\subseteq\modes_{\calA}$, the possible values for $(x_1,x_2)$ are finite and therefore the possible values for $(y_1,y_2)$ are also finite (since given $\modes$, there is a 1-1 correspondence between $(x_1,x_2)$ and $(y_1,y_2)$). We may, therefore split $\calR_\modes\cap\calAvarrho$ into a finite number of equivalence classes with respect to the minimum and maximum point $(x_1,x_2)$ of the set of modes corresponding to the vector $\vec u$. It is sufficient to find for any equivalence class a (possibly different) constant that satisfies the desired property for any $\vec u$ in the class. Then we could minimize over the equivalence classes to find $\descriptivityparam_1>0$ as desired.

\def\equivclass{\mathfrak{C}}
Consider now the equivalence class $\equivclass$ corresponding to some fixed pair $(x_1,x_2)$ (which gives $(y_1,y_2)$). Then, for any $\vec u\in\equivclass$, we have $\vec u\cdot \vec T(y) \ge \min\{\vec u\cdot \vec T(y_1) , \vec u\cdot \vec T(y_2)\}$ or equivalently,
that $\vec u_\varrho\cdot (\vec T(y)-\vec T(M_{\vec u})) \ge \min\{\vec u_\varrho\cdot (\vec T(y_1)-\vec T(M_{\vec u})) , \vec u_\varrho\cdot (\vec T(y_2)-\vec T(M_{\vec u}))\}$ for any $y\in\calY_\modes$, where 
$\vec u_\varrho = \varrho\cdot \vec u/\|\vec u\|$. 
Also, $\equivclass\subseteq\calR_\modes\cap\calAvarrho$ 
and $\calR_\modes\cap \conehull\calA$ 
is a cone and therefore $\calR_\modes\cap\calAvarrho$ contains all vectors $\vec u_\varrho$ where $\vec u\in\equivclass$. Moreover, $\calR' :=\calR_\modes\cap\calAvarrho\cap\{\vec u': \|\vec u'\| = \varrho\}$ is closed (since $\calA$ is closed by assumption \eqref{assumption:geometry}). Therefore, for any $\vec u \in\equivclass$ and any $y\in\calY_{\modes}$:
\[
{\vec u_{\varrho}}\cdot (\vec T(y)-\vec T(M_{\vec u})) \ge \min\left\{ \inf_{\vec u'\in\calR'}\vec u'\cdot (\vec T(y_1)-\vec T(M_{\vec u})), \inf_{\vec u'\in\calR'}\vec u'\cdot (\vec T(y_2)-\vec T(M_{\vec u}))\right\}.
\]
We know that $\vec u'\cdot (\vec T(y_1)-\vec T(M_{\vec u})), \vec u'\cdot (\vec T(y_2)-\vec T(M_{\vec u})) >0$ for any $\vec u'\in \calR_\modes\cap\calAvarrho$, since $y_1,y_2\in\calY_\modes$. Since, additionally $\calR'$ is closed, the infima in the above inequality are attained for some vectors $\vec u_1',\vec u_2'\in \calR'$ and correspond to positive values $\descriptivityparam_{11}',\descriptivityparam_{12}'>0$.

We have proven that there exists some $\descriptivityparam_1>0$ so that for any $\vec u\in\calAvarrho$ and any $y\in\calY_{\modes_{\vec u}}$ we have that $\vec u\cdot (\vec T(y)-\vec T(M_{\vec u})) \ge \descriptivityparam_1 \|\vec u\|$. As a consequence, returning to our vectors $\vec a$ (given vector) and the desired $\vec b$, since $\vec a,\vec b\in\calAvarrho$, we have that if we pick $\descriptivityparam \le \descriptivityparam_1$, then $\calY_{\modes_{\vec a}} \subseteq \calX_1$.

\paragraph{Step 3.} It remains to account for the points $x\in\ints\setminus\calY_{\modes_{\vec a}}$ (i.e., find conditions for the selection of $\vec b$ so that any such $x$ is classified either in $\calX_1$ or $\calX_2$). The first crucial observation is that the set $\ints\setminus\calY_{\modes_{\vec a}}$ must be finite. In particular, $\ints\setminus\calY_{\modes_{\vec a}}$ consists of points $x$ such that the boundary of the halfspace defined by $\vec v_x$ intersects the set $\calR_{\modes_{\vec a}}\cap \calAvarrho$, due to the definition of $\calY_{\modes_{\vec a}}$. Consider some vector $\vec u\in\calR_{\modes_{\vec a}}\cap \calAvarrho$ with $\vec u\cdot \vec v_x = 0$. The vector $\vec u$ corresponds to some distribution in $\calE_{\vec T}(\calAvarrho)$ and $x$ is a mode of $\vec u$ since $\vec u\cdot (\vec T(x)-\vec T(M_{\vec a})) = 0$ and $\vec u\in\calR_{\modes_{\vec a}}$. Hence $\ints\setminus\calY_{\modes_{\vec a}}\subseteq \modes_{\calAvarrho}$. Due to assumption \eqref{assumption:bounded-modes}, $|\modes_{\calAvarrho}|$ is finite and so does $|\ints\setminus\calY_{\modes_{\vec a}}|$.

The next observation we will use is that $\calR_{\modes_{\vec a}}\cap \conehull\calA$ is a polyhedral cone, due to assumption \eqref{assumption:geometry}, and for any $x\in\ints\setminus\calY_{\modes_{\vec a}}$, $\vec v_x\cdot \vec u\ge 0$ for any $\vec u\in \calR_{\modes_{\vec a}}\cap \calAvarrho$, while $\ints\setminus\calY_{\modes_{\vec a}}$ is finite. In particular, we consider the matrix $H\in\reals^{k\times t}$, where $t\in\nats$ and $H$ contains as columns all the vectors of the form $\vec v_x$ for $x\in \ints\setminus\calY_{\modes_{\vec a}}$. However, $H$ could have some additional columns so that $\calR_{\modes_{\vec a}}\cap \calAvarrho = \{\vec u\in\reals^k: H^T\vec u\ge 0\}$. We may apply Theorem \ref{theorem:geometry-up} (restated below) accordingly to get a bound for $\theta$ and a way to pick $\vec b$ that imply the desired result (with appropriate rescaling). Note that the bound we get for $\descriptivityparam$ can be considered independent from $\vec a$, since there is only a finite number of possible selections of $\modes_{\vec a}$ and we may minimize over them to get a global bound. 
\end{proof}

Let us restate Theorem \ref{theorem:geometry-up} for reader's convinience. Its proof can be found in the main part of the paper.
\begin{theorem*}
    Consider any polyhedral cone $\calC\subseteq\reals^k$, $k\in\nats$, where $\calC = \{\vec u: H^T\vec u\ge \vec 0\}$ for some matrix $H\in\reals^{k\times t}$, $t\in\nats$ is a description of $\calC$ as an intersection of halfspaces. Then there exists some $\theta>0$ such that for any $\vec u\in \calC$ with $\|\vec u\|\ge 1$, there exists $\vec u'\in\calC$ with $\|\vec u'\| = 1$ so that for any column $\vec h$ of $H$ at least one of the following is true:
    \begin{enumerate}
        \item Either $\vec h\cdot \vec u \ge \theta$ and $\vec h\cdot \vec u' \ge \theta$,
        \item or $\vec h\cdot \vec u = \vec h\cdot \vec u'$.
    \end{enumerate}
\end{theorem*}

We now also briefly restate the intuition behind Theorem \ref{theorem:geometry-up}. In particular, since for any $x\in\calX_1$ we must have that $\vec a\cdot \vec v_x$ is large enough, there should be some threshold for $\vec a\cdot \vec v_x$ below which we know that $x$ has to be classified in $\calX_2$. One idea would be to decide the set in which $x$ should be classified by considering $x\in\calX_2$ exactly when $\vec a\cdot \vec v_x$ is below some threshold (of the form $\descriptivityparam\cdot r$). Then, if we want to classify $x$ to $\calX_1$, we could pick any $\vec b$ so that $\|\vec b\|=r$ and $\cos(\vec b,\vec v_x)$ is large enough. To this end, we may only perturb the direction of $\vec b$, since its norm is restricted a priori. Consequently, (recall that $\vec b\cdot \vec v_x = \|\vec b\|\cdot\cos(\vec b,\vec v_x)\cdot \|\vec v_x\|$) $\vec b\cdot \vec v_x$ is also at least equal to the threshold. If $x$ should be classified to $\calX_2$, then we should pick $\vec b = \vec a+\vec u$, where $\vec u\cdot \vec v_x = 0$. The main complication here is that we are not interested in classifying only a single point $x$, but, rather, any point $x\in\ints\setminus\calY_{\modes_{\vec a}}$. For different points in $\ints\setminus\calY_{\modes_{\vec a}}$ we would then have different restrictions for $\vec b$, which could be mutually exclusive. Theorem \ref{theorem:geometry-up} states that, due to the structure of polyhedral cones, there is a way to satisfy all such restrictions simultaneously.

\section{The Proof of Theorem \ref{theorem:single-term} (Sparsifying the Parameter Space)}
\label{proof:discretization}
We restate the result for convenience.
\begin{theorem*}
Under assumptions \eqref{assumption:geometry}, \eqref{assumption:unimodal}, \eqref{assumption:bounded-modes}, \eqref{assumption:bounded-moments} and \eqref{assumption:bounded-max-eval}, there exists some value $\descriptivityparam = \descriptivityparam(\calA, \vec T)>0$ depending on the geometric properties of $\calA$ and $\vec T$, such that for any $\eps\in(0,1)$, there exists a set $\calB\subseteq\varrho\-\conehull\calA$ with $|\calB|\le \left(\widetilde{O}\left(\frac{\sqrt{\Lambda}\cdot \varrho}{\eps} + \frac{\sqrt{\Lambda}}{\eps\cdot \descriptivityparam}  \right) + O\left( \frac{\sqrt{\Lambda}}{\eps\cdot \descriptivityparam}\cdot \log(B) \right) \right)^k$ such that, for any $\vec a\in\calA$, it holds that
    \[
        \tv(\calE_{\vec T}(\vec a),\calE_{\vec T}(\vec b))\le \eps\,, \text{ for some }\vec b\in\calB\,.
    \]
\end{theorem*}
\begin{proof}
    According to Theorem \ref{theorem:projection}, if we consider $\calA' = \varrho\-\conehull\calA \cap \{\vec a: \|\vec a\|\le \rcrit\}$ for some sufficiently large $\rcrit \le (\varrho+\frac{1}{\descriptivityparam})\cdot\ln(1/\eps)+\frac{1}{2\descriptivityparam}\cdot\ln(B)+O(\varrho+\frac{1}{\descriptivityparam})$, then the exponential family $\calE_{\vec T}(\calA')$ $\eps$-covers the family $\calE_{\vec T}(\calA)$. However, $\calE_{\vec T}(\calA')$ might contain infinitely many elements.
    
    In order to sparsify $\calE_{\vec T}(\calA')$, we make use of Lemma \ref{lemma:exp-fami-tv-kl} (applied to $\chull \calA'$), combined with assumption \eqref{assumption:bounded-max-eval}. In particular, we get that for any $\vec a,\vec b\in\calA'(\subseteq\chull\calA')$ it holds
    \[
        \tv(\calE_{\vec T}(\vec a), \calE_{\vec T}(\vec b)) \le \|\vec a-\vec b\|\cdot\sqrt{\Lambda/2}\,,
    \]
    by making use of Pinsker's inequality. Therefore, the problem of sparsely covering $\calE_{\vec T}(\calA')$ in total variation distance is reduced to sparsely covering $\calA'$ in Euclidean distance.
    
    The cover in Euclidean distance is given by Proposition \ref{proposition:euclidean-cover} and we get that $\calE_{\vec T}(\calA')$ is $\eps$-covered by $\calE_{\vec T}(\calB)$ for some $\calB\subseteq\calA'$ where $|\calB|\le (1+\rcrit\cdot\sqrt{2\Lambda}/\eps)^k$.
\end{proof}

\begin{proposition}\label{proposition:euclidean-cover}
    For any $\eps>0$, any $k\in\nats$, any $r>0$ and any subset $\calB$ of $\reals^k$ with $\sup_{\vec b\in \calB}\|\vec b\| \le r$, there exists an $\eps$-cover of $\calB$ with respect to the Euclidean distance with size at most $(1+2r/\eps)^k$.
\end{proposition}

\begin{proof}
    We use a simple greedy algorithm: We create the cover incrementally by adding in each step an arbitrary point $\vec b$ of the remaining set (initially, the remaining set is $\calB$) and remove from the remaining set the ball $\ball_{\eps}[\vec b]$.
    
    Let $(\vec b_i)_{i\in [N]}$ be the points of the cover. Note that it might be possible that $N=\infty$. However, as we will show, this is not the case.
    
    First, note that $\|\vec b_i - \vec b_j\|>\eps$, whenever $i\neq j$, since (assuming wlog $j>i$) $ \vec b_j\not\in \ball_{\eps}[\vec b_i]$. Therefore, $\ball_{\eps/2}[\vec b_i]\cap \ball_{\eps/2}[\vec b_j] = \emptyset$ whenever $i\neq j$. Note that $N$ must be finite.

    Let $\vol(\cdot)$ denote the volume measure that inputs a set and outputs its volume. Since $\vol(\cdot)$ is a measure and $(\ball_{\eps/2}[\vec b_i])_i$ are disjoint, we have
    \[
        \vol\left(\bigcup_{i\in[N]} \ball_{\eps/2}[\vec b_i]\right) = \sum_{i\in[N]} \vol \left(\ball_{\eps/2}[\vec b_i]\right) = N\cdot \left(\frac{\eps}{2}\right)^k \cdot \vol\left(\ball_{1}[\vec 0_k]\right)\,.
    \]
    Also, $\cup_{i\in[N]}\ball_{\eps/2}[\vec b_i]$ has to be a subset of $\ball_{r+\eps/2}[\vec 0_k]$, since $\calB$ is a subset of $\ball_{r}[\vec 0_k]$. Therefore
    \[
        \vol\left( \bigcup_{i\in[N]} \ball_{\eps/2}[\vec b_i] \right) \le \left(r + \frac{\eps}{2} \right)^k \cdot \vol(\ball_1[\vec 0_k])\,.
    \]
    We get that $N \le (1+2r/\eps)^k$.
\end{proof}

\section{Technical Lemmata for the Proof of Theorem \ref{theorem:covering-siiervs}}

This lemma shows that for any $\vec a$ in the $\varrho\-\conehull \calA$, the partition function is bounded under the unimodality and the bounded central fourth moment conditions.
\begin{lemma}
[Bounded Partition Function]
\label{lemma:partition-function-upper-bound}
Consider parameter space $\calA$ and sufficient statistics vector $\vec T$.
Under assumptions \eqref{assumption:unimodal} and \eqref{assumption:bounded-moments}, we have that
\[
    \Z_{\vec T}(\vec a):=\sum_{x\in\ints}\exp(-\vec a\cdot\vec T(x)) \le \exp(-\vec a\cdot \vec T(M_{\vec a}))\cdot O(B^{1/4})\,,
\]
for any $\vec a\in\varrho\-\conehull\calA$.
\end{lemma}

\begin{proof}
From assumption \eqref{assumption:bounded-moments} and the fact that $(\E[|W-\E[W]|^2])^2\le \E[|W-\E[W]|^4]$, we get the following inequality
\[
    \Var_{\vec a}(W)\le O(\sqrt{B})\,.
\]
We also know that $\E_{\vec a}[|W-M_{\vec a}|^2]\le 4\Var_{\vec a}(W)$, due to unimodality of the random variable $W$ (which implies that $|\E_{\vec a}[W]-M_{\vec a}|\le \sqrt{3\Var_{\vec a}(W)}$ as shown by \cite{johnson1951themomentproblem}).

Therefore, $\E[|W-M_{\vec a}|^2]\le O(\sqrt{B})$. Consider the random variable $U:=|W-M_{\vec a}|$. We have that $\Var_{\vec a}(U) \ge 0$ and hence $\E_{\vec a}[U^2]\ge (\E_{\vec a}[U])^2$. Therefore
\[
    \E_{\vec a}[|W-M_{\vec a}|]\le O(B^{1/4})\,.
\]

We have
\begin{align*}
    \E_{\vec a}[|W-M_{\vec a}|] = & 
        \frac{
            \sum_{x\in\ints}|x-M_{\vec a}|\cdot\exp(-\vec a\cdot\vec T(x))
            }{
            \sum_{x\in\ints}\exp(-\vec a\cdot\vec T(x))} \\
    = & 
        \frac{
            \sum_{x=0}^\infty x (e^{-\vec a\cdot \vec T(x+M_{\vec a})} + e^{-\vec a\cdot \vec T(M_{\vec a}-x)}) 
            }{
            \sum_{x\in\ints} e^{-\vec a\cdot\vec T(x)} } \\
    = & 
        \frac{ 
            \sum_{y=1}^{\infty} \sum_{x=y}^\infty (e^{-\vec a\cdot \vec T(x+M_{\vec a})} + e^{-\vec a\cdot \vec T(M_{\vec a}-x)}) 
            }{ 
            e^{-\vec a\cdot\vec T(M_{\vec a})} + \sum_{x=1}^{\infty} (e^{-\vec a\cdot \vec T(x+M_{\vec a})} + e^{-\vec a\cdot \vec T(M_{\vec a}-x)}) } \\
    = & \sum_{y=1}^{\infty} \left( 1 -
    \frac{ 
            e^{-\vec a\cdot\vec T(M_{\vec a})} + \sum_{x=1}^{y-1} (e^{-\vec a\cdot \vec T(x+M_{\vec a})} + e^{-\vec a\cdot \vec T(M_{\vec a}-x)}) 
        }{ 
            e^{-\vec a\cdot\vec T(M_{\vec a})} + \sum_{x=1}^{\infty} (e^{-\vec a\cdot \vec T(x+M_{\vec a})} + e^{-\vec a\cdot \vec T(M_{\vec a}-x)}) 
        } 
    \right) \\
    = & \sum_{y=1}^{\infty} 
    \left(
        1 - 
        \frac{
                1 + \sum_{x=1}^{y-1} (e^{-\vec a\cdot (\vec T(x+M_{\vec a}) - \vec T(M_{\vec a}) ) } + e^{-\vec a\cdot (\vec T(M_{\vec a}-x) - \vec T(M_{\vec a}))}  )
            }{
                e^{\vec a\cdot\vec T(M_{\vec a})} \cdot \Z_{\vec T}(\vec a)
            }
    \right) \\
    \ge & \sum_{y=1}^{t} 
    \left(
        1 - 
        \frac{
                1 + \sum_{x=1}^{y-1} (e^{-\vec a\cdot (\vec T(x+M_{\vec a}) - \vec T(M_{\vec a}) ) } + e^{-\vec a\cdot (\vec T(M_{\vec a}-x) - \vec T(M_{\vec a}))}  )
            }{
                e^{\vec a\cdot\vec T(M_{\vec a})} \cdot \Z_{\vec T}(\vec a)
            }
    \right) \\
    \ge & \sum_{y=1}^{t} 
    \left(
        1 -
        \frac{
                1 + 2(y-1)
            }{
                e^{\vec a\cdot\vec T(M_{\vec a})} \cdot \Z_{\vec T}(\vec a)
            }
    \right)
    =
    t - \frac{t^2}{e^{\vec a\cdot \vec T(M_{\vec a})} \cdot \Z_{\vec T}(\vec a)}
    \,,
\end{align*}
where $t\in\nats$ is arbitrary and the last inequality follows from the fact that $M_{\vec a}$ is a mode (which implies that $\vec a\cdot (\vec T(x) - \vec T(M_{\vec a}) ) \ge 0$ for any $x\in\ints$).

We can pick $t = \frac{1}{2}\cdot\exp(\vec a\cdot \vec T(M_{\vec a}))\cdot \Z_{\vec T}(\vec a) \pm O(1)$ in order to get the following bound 
\[
    \exp(\vec a\cdot \vec T(M_{\vec a}))\cdot \Z_{\vec T}(\vec a) \le 4\E_{\vec a}[|W-M_{\vec a}|] \pm O(1)\,,
\]
which concludes the proof since, as we have shown, $\E_{\vec a}[|W-M_{\vec a}|]\le O( B^{1/4})$.
\end{proof}

The next key lemma (which has also been stated and proven in the main part of the paper) shows that, under unimodality and bounded fourth central moment, the mass of points that are sufficiently far from the modes of the distribution decays exponentially. Moreover, the centered moments of order at most 2 can be roughly controlled by points that lie only in a bounded interval around the mode.   

\begin{lemma}\label{lemma:vanishing-deviation-app}
    Under assumptions \eqref{assumption:unimodal} and  \eqref{assumption:bounded-moments}, for any $\kappa > 0$, any $\eta>0$ and any $s\in\{0,1,2\}$ there exists some $\ell=e^{\kappa/(3-\eta-s)}\cdot O(B^{\frac{5}{4\cdot(3-\eta-s)}})$ such that for any $\vec a\in\varrho\-\conehull\calA$ and any mode $M_{\vec a}$ of the corresponding distribution we have
    \begin{enumerate}
        \item $\Pr_{\vec a}[W=x] \le \frac{e^{-\kappa\cdot \max \left\{1,\frac{\|\vec a\|}{\varrho} \right\} }}{|x-M_{\vec a}|^{1 + \eta + s}} \cdot\Pr_{\vec a}[W=M_{\vec a}]$, for any $x\in\ints$ with $|x-M_{\vec a}|\ge \ell$.
        \item If $Q_\ell = \vec 1\{|W-M_{\vec a}|\le \ell\}$ then
    \[
        \E_{\vec a}[|W-M_{\vec a}|^s] \le \E_{\vec a}[|W-M_{\vec a}|^s\cdot Q_\ell\bigr] +  e^{-\kappa\cdot \max \left\{1,\frac{\|\vec a\|}{\varrho} \right\} } \cdot O(1/\eta) \,.
    \]
    \end{enumerate}
    In particular, for $s=0$, we use the convention $\E[W^0] = \Pr[W\neq 0]$, for any random variable $W$.
\end{lemma}

We next show that the expectation and the variance for parameters inside $\varrho\-\conehull \calA$ are continuous functions with respect to the parameter vectors. 
\begin{lemma}\label{lemma:continuity}
    Under assumptions \eqref{assumption:unimodal} and \eqref{assumption:bounded-moments}, the expectation $\E_{\vec a}[W]$ and variance $\Var_{\vec a}(W)$ are continuous functions of $\vec a$ on $\varrho\-\conehull\calA$.
\end{lemma}
\begin{proof}
    We will prove that sums of the form $\sum_{x\in\ints} \exp(-\vec a\cdot \vec T(x))$ and $\sum_{x\in\ints} x^s\cdot\exp(-\vec a\cdot \vec T)$, where $s=1,2$, are continuous functions of $\vec a$ on $\calA$. Then, $\E_{\vec a}[W]$ and $\Var_{\vec a}(W)$ have to be continuous.
    
    We proceed with the proof for $S:=\sum_{x\in\ints} x^2\cdot\exp(-\vec a\cdot \vec T(x)$, since the other cases can be proven similarly. 
    
    Fix some $\vec a\in\calA$, some $\eps>0$ and consider any $\delta\vec a\in\reals^k$ so that $\vec a':=\vec a+\delta\vec a\in\calA$ and $\|\delta\vec a\|\le \delta$, where $\delta>0$ to be decided (possibly dependent on $\eps$ and $\vec a$). We may apply Lemma \ref{lemma:vanishing-deviation-app} to $\calE_{\vec T}(\calA)$, with $s=2$, $\eta=1/2$ and $\kappa\in(0,\infty)$ to be decided. Therefore, we get some $\ell=\ell(\kappa)$ such that for any $\vec b\in \calA'$ and any $x\in\ints$ with $|x-M_{\vec b}|>\ell$ we have
    \[
        \exp(-\vec b\cdot \vec T(x)) \le e^{-\kappa\|\vec b\|/\varrho} \cdot \frac{1}{|x-M_{\vec b}|^{3.5}} \cdot \exp(-\vec b\cdot \vec T(M_{\vec b}))\,.
    \]
    Hence we have that
    \begin{align*}
        \sum_{x:|x-M_{\vec b}|>\ell} x^2 \cdot \exp(-\vec b\cdot \vec T(x)) \le &\ e^{-\kappa\|\vec b\|/\varrho} \cdot \frac{x^2}{|x-M_{\vec b}|^{3.5}} \cdot \exp(-\vec b\cdot \vec T(M_{\vec b})) \\
        = &\ e^{-\kappa\|\vec b\|/\varrho} \cdot \exp(-\vec b\cdot\vec T(M_{\vec b})) \cdot \sum_{y:|y|>\ell} \frac{|y+M_{\vec b}|^2}{|y^4|}\,.
    \end{align*}
We have that $\|\vec a\|,\|\vec a'\|>0$, since otherwise $\calE_{\vec T}(\calA)$ would not be well defined. Let $f(\vec b) = \exp(-\vec b\cdot\vec T(M_{\vec b})) \cdot \sum_{y:|y|>\ell} \frac{|y+M_{\vec b}|^2}{|y^4|}$. Therefore, we can pick $\kappa$ as a function of the quantities $1/\eps$, $\max\{ \|\vec a\|^{-1}, \|\vec a'\|^{-1} \}$, $\varrho$ and $\max\{f(\vec a), f(\vec a')\}$, which are all finite, for any $\vec a$ and $\vec a'$ as hypothesised, so that $\sum_{x:|x-M_{\vec b}|>\ell} x^2 \cdot \exp(-\vec b\cdot \vec T(x)) \le \eps/4$. Hence, if we consider $N = \{x\in\ints: |x-M_{\vec a}|\le \ell \text{ or } |x-M_{\vec a'}|\le \ell \}$, we have that
\[
    \sum_{x\in\ints} x^2 \cdot \exp(-\vec a\cdot\vec T(x)) = \sum_{x\in N} x^2 \cdot \exp(-\vec a\cdot\vec T(x)) \pm \eps/4\,,
\]
and similarly for $\sum_{x\in\ints} x^2 \cdot \exp(-\vec a'\cdot\vec T(x))$. Therefore 
\[
    \left|\sum_{x\in\ints} x^2 \cdot e^{-\vec a\cdot\vec T(x)} - \sum_{x\in\ints} x^2 \cdot e^{-\vec a'\cdot\vec T(x)} \right| \le \sum_{x\in N} x^2 \cdot e^{-\vec a\cdot \vec T(x)} \cdot \left|1-e^{- \delta \vec a \cdot \vec T(x)} \right| + \eps/2\,.
\]
We have that $N$ is finite and we can pick $\delta$ so that since $\|\delta\vec a\|\le \delta$, the distance is at most $\eps$ (by upper bounding the sum of the right hand side with $|N|$ times the maximum term).
\end{proof}

\newpage
\section{Applications and Examples}

\subsection{Examples of Distributions that we capture}
\label{appendix:examples}

Our assumptions for proper learning and covering of SIIERVs (see Assumption \ref{assumption:proper}), capture a wide variety of families of discrete distributions, including discretized versions of many fundamental distributions, like Gaussian, Laplacian, etc. Although we focus on the case where the family $\calE_{\vec T}(\calA)$ includes distributions supported on $\ints$, our results (and our assumptions) naturally extend to the cases where the support is some subset of $\ints$, like $\natszero$. In some cases, for example for distributions with finite support, our assumptions can be relaxed. In the following table, we represent examples of distributions with infinite support that our results capture.\\

\begin{table}[ht]
\caption{A collection of pairs $(\vec T, \calA)$ on which our results on learning and covering apply.\newline
}
\label{table:examples}
\centering
\begin{tabular}{@{}cclc@{}}
\toprule
{\color[HTML]{000000} Sufficient Statistic $\vec T$} &
{\color[HTML]{000000} Support} &
{\color[HTML]{000000} Extended Parameter Space $\calAvarrho$} &
{\color[HTML]{000000} Distribution} \\ \midrule

{\color[HTML]{000000} $T(x) = \ln(x)$} &
{\color[HTML]{000000} $x \in \nats$} &
{\color[HTML]{000000} $[5+\eta, \infty), \eta>0$ } &
{\color[HTML]{000000} Zeta } \\

{\color[HTML]{000000} $T(x) = x$} &
{\color[HTML]{000000} $x\in \natszero$} &
{\color[HTML]{000000} $[\eta,\infty), \eta>0$} &
{\color[HTML]{000000} Geometric} \\

{\color[HTML]{000000} $T(x) = |x|$} &
{\color[HTML]{000000} $x\in \ints$} &
{\color[HTML]{000000} $[\eta,\infty), \eta>0$} &
{\color[HTML]{000000} Discrete Laplacian} \\

{\color[HTML]{000000} $\vec T(x) = (x, x^2)$} &
{\color[HTML]{000000} $x\in \ints$} &
{\color[HTML]{000000} $\{\vec a: a_2\ge |a_1|/L\} \setminus \ball_\eta(\vec 0), L>0$} &
{\color[HTML]{000000} Discrete Gaussian} \\

{\color[HTML]{000000} $\vec T(x) = (|x|, x, x^2)$} &
{\color[HTML]{000000} $x\in \ints$} &
{\color[HTML]{000000} $\{\vec a: a_3\ge |a_2|/L, a_1 \ge 0\} \setminus \ball_\eta(\vec 0)$} &
{\color[HTML]{000000} Gaussian-Laplacian} \\

{\color[HTML]{000000}} &
{\color[HTML]{000000}} &
{\color[HTML]{000000}} &
{\color[HTML]{000000} Interpolation}

   \\ \bottomrule
\end{tabular}
\end{table}

\subsection{Parametric Application: Proper Covers for PNBDs}\label{appendix:pnbds}
In this section, we provide a parametric application that is captured by our techniques. We study the class of Poisson Negative Binomial random variables, i.e., sums of independent but not necessarily identically distributed Geometric random variables. We provide the following structural result.
\begin{theorem}
[Proper Cover of Poisson Negative Binomials]
\label{theorem:main-cover-pnbds}
Let $\plow \in(0,1)$.
 For any $\eps>0$, the family of Poisson Negative Binomial distributions (i.e., sums of Geometric random variables with success probability at least $\plow$) of order $n$ admits an $\eps$-proper cover of size $O(n^2/\poly(\plow)) + 2^{\poly(1/\eps, 1/\plow)}$. Moreover, for any PNBD $X$, there exists $Y$ so that $\tv(X,Y) \leq \eps$ and (i) either $Y$ is a PNBD of order $O(\poly(1/\eps, 1/\plow))$ among $2^{\poly(1/\eps, 1/\plow)}$ candidates (sparse form) or (ii) $Y$ is a Negative Binomial random variable of order $O(n) \cdot\poly(1/\plow)$ (dense form).
\end{theorem}

The essentially important part of the proof is that we do not need to assume a variance lower bound (as we did in assumption \eqref{assumption:variance-lower-bound}), since this is assured using the so-called Massage step of \cite{daskalakis2015sparse}. The main tool of this trick is the Poisson approximation technique. Hence, in the proof of the above theorem, we solely focus on this massage procedure and we omit the details on how to handle the sparse and the dense case since they follow by adapting the techniques of our main results.

\subsection{The proof of Theorem \ref{theorem:main-cover-pnbds}}
\begin{proof}
Let $\eps > 0$.
Consider $X=\sum_{i\in[n]}X_i$, where $X_1,...,X_n$ are independent and for and $i\in[n]$, $X_i\sim \Geo(p_i)$ with $p_i\in[\plow,1]$. Our proof involves three main parts. First, we perform a massage step to discard the terms with low variance from the sum. Then, we split two cases according to the number of terms that have survived. If the number of surviving terms is smaller than some (appropriately selected) $\ncrit$, then it is sufficient to approximate each term with accuracy $O(\eps/\ncrit)$. If the number of surviving terms is higher than $\ncrit$, then we prove that $X$ is close to some discretized Gaussian and from that we find a Negative Binomial random variable that matches the first two moments of the sum and so is close to the Gaussian. The proximity follows by the triangle inequality of the TV distance.
For the following, consider $\kappa >1$ where $1/\kappa =O(\eps)$.

\paragraph{Massage Step.} 
Consider the index set $I = \{ i \in [n] : p_i > 1-1/\kappa \}$. For any $i \notin I$, we let $X_i' \sim \Geo(p_i)$ and, using Lemma \ref{lemma:tv-subadd}, we get that
\[
    \tv\left(\sum_{i \in [n]}X_i,\sum_{i \in [n]} X_i'\right)\le \tv\left(\sum_{i\in I} X_i,\sum_{i\in I}X_i'\right)\,.
\]

For any $i \in I$, we either set $X_i' \sim \Geo(p_i')$ with $p_i' = 1-1/\kappa $. or we set $X_i'=0$ almost surely. Since $X_1, \ldots, X_n$ are independent geometric random variables, we can apply
the following technical lemma:
\begin{lemma}
[Corollary 2.5 of \cite{barbour1987asymptotic}]
Consider $n$ independent random variables $X_1, ..., X_n$ that are geometrically distributed with success probabilities $p_1,...,p_n$ respectively. Let $\lambda = \sum_{i \in [n]} \frac{1-p_i}{p_i}$. Then, it holds that
\[
\tv \left(\sum_{i \in [n]} X_i, \Poi(\lambda) \right) \leq \lambda^{-1}(1-e^{-\lambda}) \cdot \sum_{i \in [n]} \left ( \frac{1-p_i}{p_i} \right)^2\,.
\]
\end{lemma}
Note that $\lambda^{-1}(1-e^{-\lambda}) \leq \min\{1, \lambda^{-1})$. We make use of the above Poisson approximation lemma on the set of indices $I$ and get that the random variable $\sum_{i \in I} X_i$ can be approximated by a Poisson random variable with distribution $\Poi \left(\sum_{i \in I} \E[X_i] \right)$. Specifically, we get that
\[
    \tv\left(\sum_{i\in I}X_i,\Poi \left (\sum_{i\in I}\E[X_i] \right) \right) \le
    \frac{\sum_{i \in I}\E[X_i]^2}{\sum_{i \in I}\E[X_i]} \leq
    \max_{i \in I} \E[X_i]\,,
\]
where we applied Lemma \ref{lemma:ratio-ineq} to the sequences of non-negative real numbers $(\E[X_i]^2)_{i \in I}$ and $(\E[X_i])_{i \in I}$. Hence, we have that
\begin{equation}
\label{eq:tv-input-poisson-highdim}
    \tv\left(\sum_{i\in I}X_i,\Poi \left(\sum_{i\in I}\E[X_i]\right)\right) \le
 \max_{i \in I} \left\{ \frac{1-p_i}{p_i}
 \right\} = \frac{1}{\kappa -1} \,.
\end{equation}

We get the same upper bound for the total variation distance between $X_I':=\sum_{i\in I}X_i'$ and $\Poi(\E[X_I'])$, similarly. We continue with the following claim.
\begin{claim}
[Correct Rounding]
\label{claim:correct-rounding-highdim}
We can partition the set $I \subseteq [n]$ into two sets $I_{\star},I_0$ and set $X_i' \sim \Geo(p_i')$ with $p_i' = 1-1/\kappa $, for any $i \in I_{\star}$ and $X_i'=0$ almost surely for any $i \in I_0$ so that
\[
\left| \sum_{i \in I}\E[X_i] - \sum_{i \in I} \E[X_i']\right| \leq \frac{1}{\kappa -1}\,.
\]

\end{claim}
\begin{proof}
If $i \in I_{\star}$, we have that $\E[X_i'] \leq 1/(\kappa -1)$, whereas $\E[X_i'] = 0$ if $i \in I_0.$ In the extreme case where $I_{\star} = I$, note that the expectation of the Geometric is non-increasing and so we have that $\E[X_i'] \geq \E[X_i]$ for any $i \in I$ and, so, we have that
\[
\sum_{i\in I} \E_{X_i' \sim \Geo(p_i')}[X_i'] \geq \sum_{i \in I}\E[X_i]\,.
\]
Hence, we can pick $I_{\star}$ to be any minimal subset of $I$ so that
\[
\sum_{i\in I_\star}\E_{X_i' \sim \Geo(p_i')}[X_i']
\geq \sum_{i \in I} \E_{X_i \sim \Geo(p_i)}[X_i]\,.
\]
This choice of $I_{\star}$ yields
\[
\left | \sum_{i \in I_\star} \E[X_i'] - \sum_{i \in I}\E[X_i] \right| \leq 1/(\kappa -1)\,,
\]
and this provides Claim \ref{claim:correct-rounding-highdim}.
\end{proof}
By Lemma \ref{lemma:tv-subadd}, we conclude that
\begin{align*}
\tv\left(\sum_{i\in I} X_i,\sum_{i\in I}X_i'\right)
\leq \frac{3}{\kappa -1}\,,
\end{align*}
using Poisson approximation for $(X_i)_{i \in I}$ and $(X_i')_{i \in I}$ and combining the upper bound for the total variation distance of two Poisson distributions (see Lemma \ref{lemma:tv-poissons}) with Claim \ref{claim:correct-rounding-highdim}. 

Without loss of generality, we consider $n' = |I_\star|$, rearrange the terms and discard the trivial ones so that $X'=\sum_{i\in[n']}X_i'$, with $\tv(X,X')\le 3/(\kappa -1)$ and $X_i'\sim{\Geo(p_i')}$, with $p_i'\in[\plow,1-1/\kappa ]$.

The next steps are similar to the general case. Using Gaussian approximation, we compute $\ncrit$. 
In particular, we can get that $\ncrit = \poly(\kappa /\plow) = \poly(1/(\eps \cdot \plow))$.
If $n' \leq \ncrit$, the PNBD is close to a sparse form that is a sum of Geometric random variables consisting of at most $\ncrit$ terms. In this case, it is sufficient to approximate each term $X_i$ separately using a random variable $Y_i \sim \Geo(q_i)$. Due to sub-additivity of the statistical distance, it suffices to control the TV distance between $X_i$ and $Y_i$ by $\eps/\ncrit$. Then, it will hold that $\tv(\sum_{i \in [n']}X_i, \sum_{i \in [n']} Y_i) \leq \eps$ for $n' \leq \ncrit$. We have to discretize the interval $[\plow, 1-1/\kappa ]$ with appropriate accuracy in order to get the result. The discretization depends on the TV distance between two Geometric random variables that can be easily computed.

Otherwise, we first approximate it using a discretized Gaussian random variable and then match the expectation and the variance in order to find a Negative Binomial that is close to the input PNBD. This gives the bounds presented in the statement but we omit the details.
\end{proof}

\subsection{Verification of Assumptions}
\label{appendix:verification}

Although Assumption \ref{assumption:proper} might not be efficiently verifiable for every selection of the sufficient statistics, assuming a simple given description of the sufficient statistics vector, analytic methods can potentially reduce the assumptions to restrictions on the space of parameters. In particular, we have the following.


    
    
    
    


\begin{itemize}
    \item For conditions \ref{assumption:unimodal} and \ref{assumption:bounded-modes} (unimodality and localization of modes) we have already identified an algebraic condition in terms of an appropriate set of linear inequalities (see Appendix \ref{proof:sparsification-st}).
    \item Condition \ref{assumption:bounded-moments} (bounded central moments) is linked to a lower bound on the minimum norm of the parameter space. For instance, if the sufficient statistics is a scalar logarithmic function (corresponding to Zeta distribution), the fourth moment is bounded when the parameter takes values bounded away above $5$, according to the convergence of the zeta function (see Appendix \ref{appendix:examples}).
    \item For condition \ref{assumption:bounded-max-eval} (spectral bound on the covariance matrix), it is sufficient to show an upper bound on the expected value of the squared norm of the sufficient statistics vector, i.e., $\E[\|\vec T(W)\|^2]$. Such an upper bound may correspond to the exclusion of some parameter values when different coordinates of the sufficient statistics have different behavior in the limit $x\to \infty$. For example, if one coordinate is polynomial, while another one is logarithmic, we would like to ensure that when the parameter corresponding to the polynomial statistic is zero, the other parameter will be bounded away above a value that depends on the degree of the polynomial statistic, i.e., if $\vec T(x) = (x^r, \log x)$, then $(0,a_2)\in\calA$ implies $a_2\ge f(r)$, where $f$ is some appropriate (increasing) function.
    \item Finally, for condition \ref{assumption:variance-lower-bound} (variance lower bound), consider the simple example of the Geometric distributions, where the sufficient statistics is a scalar linear function over $\nats$ (i.e., $\vec T(x)=x$). Then, the variance lower bound is equivalent to an upper bound on the parameter space (i.e., $ a\in\calA$ implies $a\le a_{\max}$). We note, however, that the variance lower bound does not always imply that the parameter space is bounded. In particular, when a distribution has two or more subsequent modes, then, as the norm of the parameter increases to the limit, the variance remains bounded away above zero. Therefore, the variance lower bound may correspond to a different upper bound on the norm for each direction of the parameter space (since the parameter vector's direction defines the set of modes; see Appendix \ref{proof:sparsification-st}).
\end{itemize}

\subsection{Variance Lower Bounds and Poisson Approximation}
One can show that condition \ref{cond:unimodal-SIIURV} of Assumption \ref{assumption:SIIURV} gives a variance lower bound. This implies that each summand of the input will either have very small variance or there will exist a global lower bound. This condition is not present in previous works and is due to a very delicate phenomenon that may not occur in our very general setting; the phenomenon is connected with the technique of Poisson approximation \cite{novak2019poisson,rollin2005approximation,barbour1987asymptotic,barbour1992poisson} and the so-called ``magic factors''. 

Let us set up the problem.
Consider a collection of events $\{A_i\}_{i \in [n]}$ with $\Pr[A_i] = p_i$ and let $X$ be the random number of events that occur. So, $X$ has expectation $\lambda = \sum_{i \in [n]} p_i$ and, if $n$ is large and the $p_i$'s are small and the events are not quite dependent, then the distribution of $X$ should be close to the Poisson distribution with mean (and variance) $\lambda$. Hence, the Poisson approximation appears naturally in scenarios where one deals with a large number of rare events. 

This technique was used by \cite{daskalakis2015sparse} in order to ``massage'' the indicator random variables that have sufficiently small success probabilities and replace them with a small TV distance overhead. Specifically, they used a result about Poisson approximation for the Poisson Binomial random variable \cite{novak2019poisson}. In our case, we would require a similar result for our sum of unimodal and integer-valued random variables to hold.
The main questions arising is (i) which distribution can approximate such a sum when the expectations of the individual terms are quite small (``rare events'')? (ii) is it possible to obtain TV distance bounds or can we only obtain weaker bounds (e.g., in Kolmogorov distance)?

\paragraph{Poisson Approximation.} 
\cite{le1960approximation} first provided a Poisson approximation for a PBD using the maximal coupling technique: For $W = \sum_{i \in [n]} X_i$ with $X_i \sim \Be(p_i)$ and  $\E X_i = p_i$, he showed that $\tv(W, \Poi(\lambda)) \leq \sum p_i^2$ for $\lambda = \E W$ (see \cite{novak2019poisson}). However, this bound is not optimal. 
In \cite{chen1975poisson}, Stein's method was introduced in the Poisson context in order to obtain statistical distance bounds between a PBD and a Poisson distribution. Previously, different techniques for similar bounds were developed by \cite{kerstan1964verallgemeinerung}.
The Stein-Chein method was refined in \cite{barbour1984rate} to obtain an improvement
which states that $\tv(W, \Poi(\lambda)) \leq (1-e^{-\lambda})\lambda^{-1} \sum p_i^2$ and this improvement is tight up to constant factors. The term $\lambda^{-1}$ makes this inequality of the right order and is called the ``magic factor''. The main question is whether this magic factor (which comes from bounds of the solution of the Stein equation)
can appear when we replace the PBD with other distributions. Using the Stein-Chen method, \cite{barbour1987asymptotic} derived approximation results for sums of independent nonnegative integer random variables in the neighbourhood of the Poisson distribution in total variation distance (and, more generally for any distance $d_H = \sup_{h \in H}|\E h(W) - \E h(Q)|$ induced by some class of testing functions $H$). This result gives a bound (generally called the Barbour-Eagleson estimate), similar to the one used for Poisson Binomial Distributions, for the important class of Poisson Negative Binomial Distributions (see \Cref{appendix:pnbds}).
For more general distributions, the result employs ``distributional expansions of the Poisson measure'' using the Charlier polynomials (this class of polynomials is the analog of the Hermite polynomials for the Gaussian measure) and we refer to \cite{barbour1987asymptotic} for further details. 
Several other results for Poisson approximation of non-negative random variables can be found at \cite{novak2019poisson}. For the case of integer-valued random variables, a number of authors has also dealt with shifted Poisson approximation \cite{novak2019poisson}.

\end{document}